\documentclass[twoside,11pt]{article}

\usepackage{natbib}
\usepackage[colorlinks=false,allbordercolors={1 1 1}]{hyperref}
\usepackage{url}

\usepackage{epsfig}
\usepackage{graphicx}
\usepackage{amsmath}
\usepackage{amssymb}
\usepackage{amsthm}
\usepackage{algorithm}
\usepackage{algorithmic}
\usepackage{enumerate}

\newtheorem{theorem}{Theorem}
\newtheorem{remark}{Remark}

\newtheorem{lemma}{Lemma}
\newtheorem{assumption}{Assumption}

\DeclareMathOperator*{\argmax}{arg\,max}
\newcommand{\dual}[1]{{#1'}}
\newcommand{\logdet}{\log\det}
\newcommand{\R}{\mathbb{R}}
\newcommand{\iverson}[1]{1[{#1}]}
\newcommand{\Frob}{\mathfrak{F}}
\newcommand{\zero}{\mathbf{0}}
\newcommand{\one}{\mathbf{1}}
\newcommand{\I}{\mathbf{I}}
\renewcommand{\t}[1]{{#1}^{\rm T}}
\newcommand{\inv}[1]{{#1}^{-1}}
\newcommand{\diag}{\mathbf{diag}}
\newcommand{\greekbf}[1]{\text{\boldmath $#1$}}
\newcommand{\BSigma}{\mathbf{\Sigma}}
\newcommand{\BSigmah}{\widehat{\BSigma}}
\newcommand{\sigmah}{\widehat{\sigma}}
\newcommand{\vsigmah}{\widehat{\greekbf{\sigma}}}
\newcommand{\BSigmat}{\overline{\BSigma}}
\newcommand{\sigmat}{\overline{\sigma}}
\newcommand{\BOmega}{\mathbf{\Omega}}
\newcommand{\BOmegah}{\widehat{\BOmega}}
\newcommand{\BOmegat}{\overline{\BOmega}}
\newcommand{\BOmegac}{\widetilde{\BOmega}}
\newcommand{\omegah}{\widehat{\omega}}
\newcommand{\vomega}{\greekbf{\omega}}
\newcommand{\vomegah}{\widehat{\vomega}}
\newcommand{\vomegac}{\widetilde{\vomega}}
\newcommand{\omegat}{\overline{\omega}}
\newcommand{\W}{\mathbf{W}}
\newcommand{\y}{\mathbf{y}}
\renewcommand{\S}{\mathbf{S}}
\renewcommand{\u}{\mathbf{u}}
\renewcommand{\H}{\mathbf{H}}
\newcommand{\h}{\mathbf{h}}
\newcommand{\A}{\mathbf{A}}
\newcommand{\B}{\mathbf{B}}
\renewcommand{\a}{\mathbf{a}}
\newcommand{\x}{\mathbf{x}}
\newcommand{\q}{\mathbf{q}}
\renewcommand{\r}{\mathbf{r}}
\newcommand{\g}{\mathbf{g}}
\renewcommand{\c}{\mathbf{c}}
\newcommand{\si}[1]{^{(#1)}}
\renewcommand{\O}{\mathcal{O}}
\newcommand{\z}{\mathbf{z}}
\renewcommand{\b}{\mathbf{b}}
\newcommand{\Sup}{\mathcal{S}}
\newcommand{\Supc}{\Sup^c}
\newcommand{\Supt}{\overline{\Sup}}
\newcommand{\E}{\mathbb{E}}
\renewcommand{\P}{\mathbb{P}}
\renewcommand{\exp}[1]{e^{#1}}
\newcommand{\lexp}[1]{{\rm exp}\left(#1\right)}
\newcommand{\sgn}{{\rm sgn}}
\newcommand{\BGammat}{{\overline{\mathbf{\Gamma}}}}
\newcommand{\BPsit}{{\overline{\mathbf{\Psi}}}}
\newcommand{\psit}{{\overline{\psi}}}
\newcommand{\BPhit}{{\overline{\mathbf{\Phi}}}}
\newcommand{\phit}{{\overline{\phi}}}
\newcommand{\Z}{\mathbf{Z}}
\newcommand{\Zh}{\widehat{\Z}}
\newcommand{\zh}{\widehat{z}}
\newcommand{\zc}{\widetilde{z}}
\newcommand{\vzh}{\widehat{\z}}
\newcommand{\vzc}{\widetilde{\z}}
\newcommand{\Zc}{\widetilde{\Z}}
\newcommand{\m}{g}
\newcommand{\eps}{\varepsilon}
\newcommand{\K}{\mathcal{K}}
\newcommand{\vt}{\mathbf{t}}
\renewcommand{\vec}{\mathbf{vec}}
\newcommand{\dotprod}[2]{\langle #1,#2 \rangle}
\newcommand{\kronprod}[2]{#1 \otimes #2}
\newcommand{\BDelta}{\mathbf{\Delta}}
\newcommand{\C}{\mathbf{C}}
\newcommand{\D}{\mathbf{D}}
\renewcommand{\d}{\mathbf{d}}
\newcommand{\myprod}[2]{#1 \odot #2}
\newcommand{\hns}{\hspace{-0.025in}}
\newcommand{\J}{\mathbf{J}}
\newcommand{\F}{\mathbf{F}}
\newcommand{\Ball}{\mathbb{B}}
\newcommand{\U}{\mathcal{U}}
\newcommand{\DD}{\mathcal{D}}
\newcommand{\T}{\mathcal{T}}
\newcommand{\V}{\mathcal{V}}
\newcommand{\BesselK}{\mathbb{K}}
\newcommand{\MeijerG}{\mathbb{G}}
\newcommand{\captionx}[1]{\caption{\footnotesize{#1}}} 
\renewcommand{\cite}[1]{\somethingunexistant}

\begin{document}

\title{\textbf{On the Statistical Efficiency of $\ell_{1,p}$ Multi-Task Learning of Gaussian Graphical Models}}

\author{
Jean Honorio, \texttt{jhonorio@purdue.edu} \\
Computer Science Dept., Purdue University, West Lafayette, IN 47907, USA \\
Tommi Jaakkola, \texttt{tommi@csail.mit.edu} \\
CSAIL, Massachusetts Institute of Technology, Cambridge, MA 02139, USA \\
Dimitris Samaras, \texttt{samaras@cs.stonybrook.edu} \\
Computer Science Dept., Stony Brook University, Stony Brook, NY 11794, USA}

\date{\vspace{-0.2in}}

\maketitle

\begin{abstract} %
In this paper, we present $\ell_{1,p}$ multi-task structure learning for Gaussian graphical models.
We analyze the sufficient number of samples for the correct recovery of the support union and edge signs.
We also analyze the necessary number of samples for any conceivable method by providing information-theoretic lower bounds.
We compare the statistical efficiency of multi-task learning versus that of single-task learning.
For experiments, we use a block coordinate descent method that is provably convergent and generates a sequence of positive definite solutions.
We provide experimental validation on synthetic data as well as on two publicly available real-world data sets, including functional magnetic resonance imaging and gene expression data.
\end{abstract}

\section{Introduction}

Structure learning aims to discover the topology of a probabilistic graphical model that accurately represents a given data set.
Accuracy of representation is measured by the likelihood that the model explains the observed data.
One challenge faced by structure learning is that the number of possible structures is super-exponential in the number of variables.
A computationally tractable method, $\ell_1$-norm regularization, has been successfully used for learning sparse structures, cf., \citep{Meinshausen06,Banerjee06,Yuan07}.
Sparse structures are easier to be interpreted, given the small number of edges.
Furthermore, sparsity guarantees the correct recovery of the edges \citep{Ravikumar11}.

In several application domains, structure learning is very useful for analyzing data sets for which probabilistic dependencies are not known apriori.
For instance, these techniques allow for modeling interactions between brain regions as well as the interactions between genes.
Suppose that we want to learn the structure of interactions for one group under analysis (e.g., a specific collection site or a specific cancer type).
We can expect that the interaction patterns of two groups are not exactly the same.
On the other hand, when learning the structure for one group, we would like to use evidence from other groups as side information in our learning process.
This becomes more important in settings with limited amount of data and high variability, such as in functional magnetic resonance imaging (fMRI) as well as in gene expression data.
Multi-task learning allows for a more efficient use of training data which is available for multiple related tasks.

In this paper, we consider the $\ell_{1,p}$ multi-task structure learning problem, which generalizes the learning of sparse Gaussian graphical models to the multi-task setting by replacing the $\ell_1$-norm regularization with an $\ell_{1,p}$-norm, also known as the \emph{simultaneous} prior \citep{Turlach05,Tropp06b} for $p=\infty$, or the \emph{group-sparse} prior \citep{Yuan06,Meier08} for $p=2$.
Here, we assume that $p>1$.
This includes two specific instances previously evaluated experimentally: \citep{Honorio10b} when $p=\infty$, and \citep{Varoquaux10} when $p=2$.
The regularizer in \citep{Mohan14} considers a general $\ell_{1,p}$ regularizer with an additional $\ell_1$-norm regularization.
In this manuscript, we perform a sample complexity analysis and provide further experimental validation.

Our main contributions and their implications are the following.
First, we analyze the sufficient number of samples for correctly recovering the support union and edge signs for $N$ variables and $K$ tasks.
For sub-Gaussian variables, the sufficient number of samples is $\O(\log K+\log N)$.
For random variables with finite high-order moment, the sufficient number of samples is $\O(K^{1/\m}N^{1/\m})$ for some integer $\m \geq 1$.
Second, we provide information-theoretic lower bounds and show that the necessary number of samples is $\O(\log K+\log N)$ for any conceivable algorithm.
The latter has the same rate as the sufficient number of samples for sub-Gaussian variables.
Thus, the polynomial-time method of $\ell_{1,p}$ regularization achieves optimal rates.
Third, we compare the statistical efficiency of multi-task and single-task learning.
For Gaussian variables, we show the sharp \emph{phase transition} between the recovery success and failure of both multi-task and single-task learning.
The sufficient number of samples is ${\O(\log K + \log N)}$ for both single-task and multi-task learning.
The necessary number of samples is ${\O(\log K + \log N)}$, $\O(K^{-1}(\log K + \log N))$ and $\O(K^{-2}(\log K + \log N))$ for the single-task, $\ell_{1,2}$ multi-task and $\ell_{1,\infty}$ multi-task problems respectively.
Thus, as more tasks are available, multi-task regularization requires less samples in order to avoid failure.

The paper is organized as follows.
Section \ref{sec:background} introduces Gaussian graphical models as well as techniques for learning such structures from data.
Section \ref{sec:prelims} presents the $\ell_{1,p}$ multi-task structure learning problem.
Section \ref{sec:consistency} analyzes the consistency of the method for correctly recovering the support union and edge signs, for both sub-Gaussian random variables and variables with finite high-order moment.
We also include information-theoretic lower bounds, and compare the statistical efficiency of multi-task and single-task learning.
Section \ref{sec:bcd} presents a block coordinate descent method that produces sparse and positive definite estimates.
We briefly discuss the computational complexity and convergence of the algorithm.
Section \ref{sec:results} presents experimental results on synthetic data as well as in two real-world data sets, including fMRI and gene expression data.
The $\ell_{1,p}$ multi-task method compares favorably with others.
Section \ref{sec:conclusions} summarizes the main contributions and results.

\section{Background} \label{sec:background}

\begin{table}
\begin{center}
\begin{footnotesize}
\begin{tabular}{p{1in}@{\hspace{0.07in}}p{4.35in}}
  \hline
  \textbf{Notation} & \textbf{Description} \\
  \hline
  $\|\c\|_1$ & $\ell_1$-norm of $\c\in \R^N$, i.e., $\sum_n{|c_n|}$ \\
  $\|\c\|_\infty$ & $\ell_\infty$-norm of $\c\in \R^N$, i.e., $\max_n{|c_n|}$ \\
  $\|\c\|_p$ & $\ell_p$-norm of $\c\in \R^N$, i.e., $(\sum_n{|c_n|^p})^{1/p}$ for $p \in (1,\infty)$ \\
  $\diag(\c) \in \R^{N \times N}$ & matrix with elements of $\c\in \R^N$ on its diagonal \\
  $\A\succeq \zero$ & $\A\in \R^{N \times N}$ is symmetric and positive semidefinite \\
  $\A\succ \zero$ & $\A\in \R^{N \times N}$ is symmetric and positive definite \\
  $\|\A\|_1$ & $\ell_1$-norm of $\A\in \R^{M \times N}$, i.e., $\sum_{mn}{|a_{mn}|}$ \\
  $\|\A\|_\infty$ & $\ell_\infty$-norm of $\A\in \R^{M \times N}$, i.e., $\max_{mn}{|a_{mn}|}$ \\
  $\|\A\|_2$ & spectral norm of $\A\in \R^{N \times N}$, i.e., the maximum eigenvalue of $\A\succ \zero$ \\
  $\|\A\|_\Frob$ & Frobenius norm of $\A\in \R^{M \times N}$, i.e., $\sqrt{\sum_{mn}{a_{mn}^2}}$ \\
  $\diag(\A) \in \R^{N \times N}$ & matrix with diagonal elements of $\A\in \R^{N \times N}$ only \\
  $\dotprod{\A}{\B}$ & scalar product of $\A,\B\in \R^{M \times N}$, i.e., $\sum_{mn}{a_{mn}b_{mn}}$ \\
  $\kronprod{\A}{\B} \in \R^{M^2 \times N^2}$ & Kronecker product of $\A,\B\in \R^{M \times N}$, i.e., $[\kronprod{\A}{\B}]_{(m_1m_2)(n_1n_2)} = a_{m_1n_1}b_{m_2n_2}$ \\
  \hline
\end{tabular}
\end{footnotesize}
\end{center}
\vspace{-0.15in}
\captionx{Notation used in this paper.}
\label{tab:notation}
\end{table}

A \emph{Gaussian graphical model} is a Markov random field in which the random variables are jointly Gaussian distributed.
This model corresponds to the multivariate Gaussian distribution for $N$ variables with covariance matrix $\BSigma\in \R^{N\times N}$.
Conditional independence in a Gaussian graphical model is simply reflected in the zero entries of the precision matrix $\BOmega=\inv{\BSigma}$ \citep{Lauritzen96}.
Let $\BOmega \equiv \{\omega_{n_1n_2}\}$, two variables $n_1$ and $n_2$ are conditionally independent given all the remaining variables, if and only if $\omega_{n_1n_2}=0$.

The concept of robust estimation by performing covariance selection was first introduced in \citep{Dempster72} where some elements of the precision matrix $\BOmega$ are set to zero.
Finding the most sparse precision matrix which fits a data set, is a NP-hard problem \citep{DAspremont08}.
In order to overcome this problem, several $\ell_1$-regularization methods have been proposed for learning sparse Gaussian graphical models from data, cf., \citep{Meinshausen06,Banerjee06,Yuan07}.

In this paper, we use the notation in Table \ref{tab:notation}.
Given a dense sample covariance matrix $\BSigmah\succeq \zero$, the problem of finding a sparse precision matrix $\BOmega$ by regularized maximum likelihood estimation (MLE) is defined as the following non-smooth convex problem:
\begin{align} \label{SparseGGM}
\max_{\BOmega\succ \zero}\left(\ell_{\BSigmah}(\BOmega)-\rho \|\BOmega\|_1\right) \; ,
\end{align}
\noindent for regularization parameter $\rho >0$.
The term $\|\BOmega\|_1$ encourages sparsity of the precision matrix or conditional independence among variables, while the term $\ell_{\BSigmah}(\BOmega)$ is the Gaussian log-likelihood function, and it is defined as:\footnote{The average Gaussian log-likelihood is in fact $\frac{1}{2}\ell_{\BSigmah}(\BOmega) - \frac{N}{2}\log{(2\pi)}$.
We consider the likelihood function in Eq.~\eqref{eq:loglikelihood} for clarity of exposition.}
\begin{align} \label{eq:loglikelihood}
\ell_{\BSigmah}(\BOmega) \equiv \logdet\BOmega-\dotprod{\BSigmah}{\BOmega} \; .
\end{align}
Several algorithms have been proposed for solving Eq.~\eqref{SparseGGM}:
sparse regression \citep{Meinshausen06,Zhou11},
linearly-constrained optimization \citep{Yuan07,Yuan10,Cai11},
block coordinate descent \citep{Banerjee06,Banerjee08,Friedman07,Hsieh13,Treister14,Yun11},
Cholesky decomposition \citep{Rothman08},
a projected gradient method \citep{Duchi08},
Nesterov's smooth optimization \citep{Lu09},
alternating linearization \citep{Scheinberg10},
quadratic approximation \citep{Hsieh11,Hsieh14},
Newton-like methods \citep{Schmidt09b,Olsen12,Dinh13},
greedy optimization \citep{Scheinberg10b,Johnson12},
divide and conquer \citep{Hsieh12},
iterative thresholding \citep{Guillot12} and
distributed optimization \citep{Kambadur13,Wang13}.

Besides sparsity, several regularizers have been proposed for Gaussian graphical models for \emph{single-task} learning:
for enforcing diagonal structure \citep{Levina08},
block structure for known block-variable assignments \citep{Duchi08,Schmidt09b,Yang13c} and unknown block-variable assignments \citep{Marlin09,Marlin09b,Sun14},
spatial coherence \citep{Honorio09},
sparse changes in controlled experiments \citep{Danaher14,Kolar09,Mohan12,Mohan14,Zhang10},
power law regularization in scale free networks \citep{LiuQiang11} and
variable selection \citep{Honorio12}.
In Section \ref{sec:prelims}, we discuss regularizers that are more relevant to our setting.

Multi-task learning has been applied to very diverse problems, such as:
linear regression \citep{Liu09,LiuJun09},
classification \citep{Jebara04},
compressive sensing \citep{Qi08},
reinforcement learning \citep{Wilson07} and
structure learning of Bayesian networks \citep{Niculescu07,Oyen12}.
Structure learning through $\ell_1$-regularization has been also proposed for different types of graphical models:
Markov random fields \citep{Lee06,Wainwright06},
Bayesian networks \citep{Schmidt07} and
conditional random fields \citep{Schmidt08}.

\section{Preliminaries} \label{sec:prelims}

In this section, we present the $\ell_{1,p}$ multi-task structure learning problem.
The $\ell_{1,p}$-norm regularizer is motivated from the multi-task learning literature.
Given $K$ arbitrary tasks, our goal is to learn one structure for each task that best explains the observed data, while promoting a common sparsity pattern of edges for all tasks.

For each task $k$, we learn a precision matrix $\BOmega\si{k}\in \R^{N\times N}$ for $N$ variables.
Let ${\BOmega \equiv \{\omega_{n_1n_2}\si{k}\}}$.
Recall that the $\ell_1$-norm regularizer is a proxy for the number of edges in the graph.
Similarly, the $\ell_{1,p}$ multi-task regularizer is a proxy for the minimum number of edges that cover the graphs from all the tasks.
In order to do this, the $\ell_{1,p}$ regularizer penalizes corresponding edges across tasks (i.e., $\omega_{n_1n_2}\si{1},\dots,\omega_{n_1n_2}\si{K}$).
Let $\BSigmah\si{k}\succeq \zero$ be the dense sample covariance matrix for task $k$, and $T\si{k}>0$ be proportional to the number of samples in task $k$.
The \emph{$\ell_{1,p}$ multi-task structure learning problem} is defined as the following non-smooth convex problem:
\begin{align} \label{eq:multitaskggm}
\BOmegah = \argmax_{(\forall k){\rm\ }\BOmega\si{k}\succ \zero}\left(\sum_k{T\si{k}\ell_{\BSigmah\si{k}}(\BOmega\si{k})} - \rho\|\BOmega\|_{1,p} \right) \; ,
\end{align}
\noindent for regularization parameter $\rho >0$ and $p>1$.
The term $\ell_{\BSigmah\si{k}}(\BOmega\si{k})$ is the Gaussian log-likelihood function defined in Eq.~\eqref{eq:loglikelihood}, while the term $\|\BOmega\|_{1,p}$ is the $\ell_{1,p}$ regularizer, and it is defined as:
\begin{align} \label{MultiTaskMeasure}
\|\BOmega\|_{1,p} \equiv \sum_{n_1 \neq n_2}{\|(\omega_{n_1n_2}\si{1},\dots,\omega_{n_1n_2}\si{K})\|_p} \; .
\end{align}
Note that the above regularizer allows for different signs of corresponding edges across tasks (i.e., $\omega_{n_1n_2}\si{1},\dots,\omega_{n_1n_2}\si{K}$).

Here, we assume that $p>1$.
Note that for $p=1$, the multi-task problem in Eq.~\eqref{eq:multitaskggm} reduces to $K$ \emph{single-task} problems as in Eq.~\eqref{SparseGGM}.
For $p<1$, Eq.~\eqref{eq:multitaskggm} is not convex.
The factor $T\si{k}$ naturally comes from the definition of the average Gaussian log-likelihood of the observed data from the $K$ tasks.

Some generalizations have been proposed subsequent to our preliminary version \citep{Honorio10b}.
Having the same intuitive motivation as for the $\ell_{1,p}$ multi-task regularizer, \citet{Guo11} proposed the following non-convex regularizer:
\begin{align*}
\|\BOmega\|_{\rm{Guo}} \equiv \sum_{n_1 \neq n_2}{\textstyle{ \sqrt{\|(\omega_{n_1n_2}\si{1},\dots,\omega_{n_1n_2}\si{K})\|_1} }} \; ,
\end{align*}
A regularizer which encourages sign coherence across tasks was proposed in \citep{Chiquet11}.
In contrast, as we mentioned before, we allow for different signs across tasks.
The regularizer proposed by \citet{Chiquet11} is:
\begin{align*}
\|\BOmega\|_{\rm{Chiquet}} \equiv \sum_{n_1 \neq n_2}{\left(\textstyle{ \sqrt{\sum_k{\max(0,\omega_{n_1n_2}\si{k})^2}} + \sqrt{\sum_k{\max(0,-\omega_{n_1n_2}\si{k})^2}} }\right)} \; .
\end{align*}
A \emph{dirty model} was proposed in \citep{Hara11,Hara13}.
In this model, the precision matrix for each task is equal to the sum of two elements.
The $\ell_1$-norm regularizer is applied to one summand, and the $\ell_{1,p}$-norm regularizer is applied to the other summand.

Bayesian methods were proposed in \citep{Peterson15,Zhu15}.
The disadvantage of \citep{Peterson15} is that the computational complexity of the normalizing constant is exponential in the number of tasks $K$ and thus, it does not scale.
(For instance, one of our real-world data sets has $41$ tasks.)
Regarding \citep{Zhu15}, although experimental results are encouraging, the non-convexity of the problem make the theoretical understanding of the consistency of the method unclear.

Next, we discuss the \emph{multi-attribute} setting of \citet{Kolar13,Kolar14}, in which each node contains a data vector.
It might seem that the multi-attribute setting is very similar to the \emph{multi-task} setting.
After all, in the multi-task setting there is one node for each task, thus, we could think of a node as containing a data vector from all tasks.
The key differences between both settings are the following.
The multi-attribute method outputs a single structure, in contrast to the multi-task method that outputs one structure for each task.
Furthermore, one data sample in the multi-attribute setting contains all the attributes for all the nodes.
In the multi-task setting each task is independent, and might contain different number of data samples.

\section{Consistency Analysis} \label{sec:consistency}

In this section, we study the consistency of the $\ell_{1,p}$ multi-task structure learning problem.
The $\ell_{1,p}$ multi-task problem, as any other MLE approach, provides a strategy for the estimation of the \emph{true} precision matrices when training data is available.

For each task $k$, let $\x\si{k} \in \R^N$ denote a zero-mean random variable.
We denote the \emph{true covariance} by $\BSigmat\si{k} = \E[\x\si{k}\t{\x\si{k}}] \succ \zero$, and thus the \emph{true precision matrix} is given by:
\begin{align} \label{eq:truesolution}
(\forall k){\rm\ }\BOmegat\si{k} = \inv{\BSigmat\si{k}} \; .
\end{align}
When training data is available, we approximate the true covariances by using sample covariances $\BSigmah\si{k} \succeq \zero$.
We then solve the $\ell_{1,p}$ multi-task structure learning problem in Eq.~\eqref{eq:multitaskggm} and obtain the \emph{empirical minimizer} $\BOmegah$.

In what follows, we analyze the conditions under which the empirical minimizer $\BOmegah$ perfectly recovers the sparsity pattern and signs of the true model $\BOmegat$.
We use a definition of sparsity pattern similar to the one used in the linear regression literature \citep{Negahban11,Obozinski11}.
Let $\BOmega \equiv \{\omega_{n_1n_2}\si{k}\}$.
The \emph{support union} (i.e., the union of the edge sets across tasks) is given by:
\begin{align} \label{eq:supportunion}
\Sup_{\BOmega} \equiv \left\{(n_1,n_2) \mid n_1 \neq n_2 \text{\ \ and\ \ } (\exists k){\rm\ }\omega_{n_1n_2}\si{k} \neq 0 \right\} \; .
\end{align}
Our goal is to show that the support union and the edge signs of the empirical minimizer and the true model are equal.

Although other techniques are also possible for consistency analysis, we follow a proof similar to the one of \citet{Ravikumar11}.
The main reason for our choice is to allow for comparison of the results.
Our proof is based in the \emph{primal-dual witness} method which has been also used, for instance, in the analysis of linear regression \citep{Negahban11,Obozinski11,Wainwright09b} and structure learning of Ising models \citep{Wainwright06}.

Our analysis keeps explicit track of some problem-specific quantities, that can scale in a non-trivial manner with respect to the number of variables $N$ and tasks $K$.
The first of these quantities is the node-degree of true model, which is defined as the maximum degree (among all nodes) of the support union of $\BOmegat \equiv \{\omegat_{n_1n_2}\si{k}\}$:
\begin{align} \label{eq:degree}
d_{\BOmegat} \equiv 1 + \max_{n_1}{|\{ n_2 \mid n_1 \neq n_2 \text{\ \ and\ \ } (\exists k){\rm\ }\omegat_{n_1n_2}\si{k} \neq 0 \}|} \; .
\end{align}
Some quantities involve the Hessian of the $\logdet$ function evaluated at the true precision matrices:
\begin{align} \label{eq:hessian}
(\forall k){\rm\ }\BGammat\si{k} \equiv \left. \nabla^2 \logdet{\BOmega\si{k}}  \right|_{\BOmega\si{k}=\BOmegat\si{k}}  {\rm\ }= \kronprod{\inv{\BOmegat\si{k}}}{\inv{\BOmegat\si{k}}} \; ,
\end{align}
\noindent where $\BGammat\si{k} \in \R^{N^2 \times N^2}$.
Let $\BSigmat \equiv \{\sigmat_{n_1n_2}\si{k}\}$.
Finally, we define the following two quantities:
\begin{align}
C_{\BSigmat} & \equiv \max_{n_1k}{\sum_{n_2}{|\sigmat_{n_1n_2}\si{k}|}} \; , \label{eq:sigmaconst} \\
C_{\BGammat} & \equiv \max_{ik}{\sum_j{|\phit_{ij}\si{k}|}} \text{\ \ for\ \ } \BPhit\si{k} \equiv \inv{\BGammat_{\Sup\Sup}\si{k}} \in \R^{|\Sup| \times |\Sup|} {\rm\ ,\ } \Sup \equiv \Sup_{\BOmegat} \; . \label{eq:gammaconst}
\end{align}
We assume that the Hessian satisfies the following type of \emph{mutual incoherence} or \emph{irrepresentable condition}:
\begin{assumption} \label{asm:incoherence}
Let $\BPsit\si{k} \equiv \BGammat_{\Supc\Sup}\si{k}\inv{\BGammat_{\Sup\Sup}\si{k}} \in \R^{|\Supc| \times |\Sup|}$ where $\Sup \equiv \Sup_{\BOmegat}$ and $\Supc$ is the complement of $\Sup$.
Let the $\ell_\dual{p}$-norm be the dual of the $\ell_p$-norm, i.e., $\frac{1}{p} + \frac{1}{\dual{p}} = 1$.
Assume that there exists some $\alpha_\dual{p} \in (0,1]$ such that:
\begin{align*}
\max_i \|\textstyle{ (\sum_j{|\psit_{ij}\si{1}|},\dots,\sum_j{|\psit_{ij}\si{K}|}) }\|_\dual{p} \leq 1-\alpha_\dual{p} \; .
\end{align*}
\end{assumption}
The above assumption is a minor generalization of the mutual incoherence condition in \citep{Ravikumar11}.
Assumption \ref{asm:incoherence} reduces to the condition in \citep{Ravikumar11} for $p=1$ and $\dual{p}=\infty$.
Intuitively speaking, Assumption \ref{asm:incoherence} limits the influence that the non-edge terms, indexed by $\Supc$, can have on the edge-based terms, indexed by $\Sup$.
Thus, there should not be any node pair that is not included in the graph (i.e., in $\Supc$) and that it is highly correlated with node pairs within the true edge-set $\Sup$.
Mutual incoherence conditions have also been used in the analysis of linear regression \citep{Negahban11,Obozinski11,Wainwright09b} and structure learning of Ising models \citep{Wainwright06}.

As part of our proof, it is necessary to show that for each task, the sample covariance is close to the true covariance.
As it is usual, the related concentration inequalities are decreasing with respect to the number of samples per task.
Thus, the worst case behavior is dominated by the task with the minimum number of training samples.
For clarity of exposition, we assume that every task contains the same number of samples.

Next, we study the consistency of the $\ell_{1,p}$ multi-task structure learning problem under two scenarios: sub-Gaussian random variables, and variates with finite high-order moment.

\subsection{Sub-Gaussian Random Variables}

Here, we provide an exponential-tail bound for sub-Gaussian variates.
The class of sub-Gaussian variates includes for instance Gaussian variables, any bounded random variable (e.g., Bernoulli, multinomial, uniform), any random variable with strictly log-concave density, and any finite mixture of sub-Gaussian variables.
Next, we show that in order to correctly recover the support union and edge signs, it is sufficient to have a number of samples $M$ that is logarithmic respect to the number of variables $N$ and tasks $K$.
\begin{theorem} \label{thm:consistencysubgaussian}
Let the $\ell_\dual{p}$-norm be the dual of the $\ell_p$-norm, i.e., $\frac{1}{p} + \frac{1}{\dual{p}} = 1$.
Let Assumption \ref{asm:incoherence} hold for the $\ell_\dual{p}$-norm and $\alpha_\dual{p} \in (0,1]$.
Let $\BSigmat \equiv \{\sigmat_{n_1n_2}\si{k}\}$.
Assume that for each $n$ and $k$, the random variable $x_n\si{k} / \sqrt{\sigmat_{nn}\si{k}}$ is zero-mean and sub-Gaussian with parameter $C_1$.
Assume that we are given $M$ i.i.d. samples for each of the $K$ tasks.
Set the regularization parameter in Eq.~\eqref{eq:multitaskggm} to $\rho = C_2 \eps K^{1/\dual{p}}/\alpha_\dual{p}$ for some $C_2 \geq 8$, where:
\begin{align*}
\eps = \sqrt{\frac{\log K + \tau \log N}{M}} \left( 8\sqrt{2}(1+4C_1^2) \max_{nk}{\sigmat_{nn}\si{k}} \right) \; ,
\end{align*}
\noindent for some $\tau > 2$.
If the sample size $M$ satisfies the bound:
\begin{align*}
M \geq d_{\BOmegat}^2 {\rm\ } (\log K + \tau \log N) \hns\left( 48\sqrt{2}(1+4C_1^2) \max_{nk}{\sigmat_{nn}\si{k}} \max{(C_{\BSigmat} C_{\BGammat}, C_{\BSigmat}^3 C_{\BGammat}^2)} \right)^2 \hns\hns\left(1\hns+\hns\frac{C_2K^{1/\dual{p}}}{\alpha_\dual{p}}\right)^4 \; ,
\end{align*}
\noindent then with probability at least $1-4/N^{\tau-2}$, we have:
\begin{align*}
\textstyle{ (\forall k){\rm\ }\|\BOmegat\si{k} - \BOmegah\si{k}\|_\infty \leq 2 \eps C_{\BGammat} \left(1+\frac{C_2K^{1/\dual{p}}}{\alpha_\dual{p}}\right) } \; .
\end{align*}
Furthermore, the support union and the edge signs of the empirical minimizer $\BOmegah \equiv \{\omegah_{n_1n_2}\si{k}\}$ are equal to those of the true model $\BOmegat \equiv \{\omegat_{n_1n_2}\si{k}\}$.
That is:
\begin{align*}
\Sup_{\BOmegah} = \Sup_{\BOmegat}\text{\ \ \ and\ \ \ }(\forall k,(n_1,n_2) \in \Sup_{\BOmegat}){\rm\ } & \textstyle{ \left( \omegat_{n_1n_2}\si{k} = 0 \text{\ \ and\ \ } |\omegah_{n_1n_2}\si{k}| \leq 2 \eps C_{\BGammat} \left(1+\frac{C_2K^{1/\dual{p}}}{\alpha_\dual{p}}\right) \right) } \\
 & \text{or\ \ } (\omegat_{n_1n_2}\si{k} > 0 \text{\ \ and\ \ } \omegah_{n_1n_2}\si{k} > 0) \\
 & \text{or\ \ } (\omegat_{n_1n_2}\si{k} < 0 \text{\ \ and\ \ } \omegah_{n_1n_2}\si{k} < 0) \; ,
\end{align*}
\noindent provided that:
\begin{align*}
\textstyle{ (\forall k,(n_1,n_2) \in \Sup_{\BOmegat}){\rm\ }\omegat_{n_1n_2}\si{k} = 0 \text{\ \ \ or\ \ \ } |\omegat_{n_1n_2}\si{k}| > 4 \eps C_{\BGammat} \left(1+\frac{C_2K^{1/\dual{p}}}{\alpha_\dual{p}}\right) } \; .
\end{align*}
\end{theorem}
\noindent (See Appendix \ref{sec:detailedproofs} for detailed proofs.)

\subsection{Random Variables with Finite High-Order Moment}

Here, we provide a polynomial-tail bound for a more general class of random variables than the ones considered in the previous sub-section.
We consider random variables with finite high-order moment.
Next, we show that in order to correctly recover the support union and edge signs, it is sufficient to have a number of samples $M$ that is polynomial with respect to the number of variables $N$ and tasks $K$.
\begin{theorem} \label{thm:consistencygeneral}
Let the $\ell_\dual{p}$-norm be the dual of the $\ell_p$-norm, i.e., $\frac{1}{p} + \frac{1}{\dual{p}} = 1$.
Let Assumption \ref{asm:incoherence} hold for the $\ell_\dual{p}$-norm and $\alpha_\dual{p} \in (0,1]$.
Let $\BSigmat \equiv \{\sigmat_{n_1n_2}\si{k}\}$.
Assume that for each $n$ and $k$, the random variable $x_n\si{k} / \sqrt{\sigmat_{nn}\si{k}}$ is zero-mean and has $4\m$-th moments upper bounded by $C_1$.
Assume that we are given $M$ i.i.d. samples for each of the $K$ tasks.
Set the regularization parameter in Eq.~\eqref{eq:multitaskggm} to $\rho = C_2 \eps K^{1/\dual{p}}/\alpha_\dual{p}$ for some $C_2 \geq 8$, where:
\begin{align*}
\eps = \sqrt{\frac{K^{1/\m} N^{\tau/\m}}{M}} \left( 2\m(\m(C_1+1))^{\frac{1}{2\m}} \max_{nk}{\sigmat_{nn}\si{k}} \right) \; ,
\end{align*}
\noindent for some $\tau > 2$.
If the sample size $M$ satisfies the bound:
\begin{align*}
M \geq d_{\BOmegat}^2 {\rm\ } K^{1/\m} N^{\tau/\m} \left( 12\m(\m(C_1+1))^{\frac{1}{2\m}} \max_{nk}{\sigmat_{nn}\si{k}} \max{(C_{\BSigmat} C_{\BGammat}, C_{\BSigmat}^3 C_{\BGammat}^2)} \right)^2 \left(1+\frac{C_2K^{1/\dual{p}}}{\alpha_\dual{p}}\right)^4 \; ,
\end{align*}
\noindent then with probability at least $1-1/N^{\tau-2}$, we have:
\begin{align*}
\textstyle{ (\forall k){\rm\ }\|\BOmegat\si{k} - \BOmegah\si{k}\|_\infty \leq 2 \eps C_{\BGammat} \left(1+\frac{C_2K^{1/\dual{p}}}{\alpha_\dual{p}}\right) } \; .
\end{align*}
Furthermore, the support union and the edge signs of the empirical minimizer $\BOmegah \equiv \{\omegah_{n_1n_2}\si{k}\}$ are equal to those of the true model $\BOmegat \equiv \{\omegat_{n_1n_2}\si{k}\}$.
That is:
\begin{align*}
\Sup_{\BOmegah} = \Sup_{\BOmegat}\text{\ \ \ and\ \ \ }(\forall k,(n_1,n_2) \in \Sup_{\BOmegat}){\rm\ } & \textstyle{ \left( \omegat_{n_1n_2}\si{k} = 0 \text{\ \ and\ \ } |\omegah_{n_1n_2}\si{k}| \leq 2 \eps C_{\BGammat} \left(1+\frac{C_2K^{1/\dual{p}}}{\alpha_\dual{p}}\right) \right) } \\
 & \text{or\ \ } (\omegat_{n_1n_2}\si{k} > 0 \text{\ \ and\ \ } \omegah_{n_1n_2}\si{k} > 0) \\
 & \text{or\ \ } (\omegat_{n_1n_2}\si{k} < 0 \text{\ \ and\ \ } \omegah_{n_1n_2}\si{k} < 0) \; ,
\end{align*}
\noindent provided that:
\begin{align*}
\textstyle{ (\forall k,(n_1,n_2) \in \Sup_{\BOmegat}){\rm\ }\omegat_{n_1n_2}\si{k} = 0 \text{\ \ \ or\ \ \ } |\omegat_{n_1n_2}\si{k}| > 4 \eps C_{\BGammat} \left(1+\frac{C_2K^{1/\dual{p}}}{\alpha_\dual{p}}\right) } \; .
\end{align*}
\end{theorem}

\subsection{Comparison to the Single-Task Problem}

Next, we compare the statistical efficiency of the \emph{multi-task} problem versus that of the \emph{single-task} problem.
First, we show the sharp \emph{phase transition} between the recovery success and failure of both methods.
Second, we show a regime in which multi-task regularization requires less amount of data than its single-task counterpart.
The analysis of sharp phase transition is difficult for general random variables and arbitrary graph topologies.
For instance, specific Gaussian ensembles have been analyzed in the linear regression literature \citep{Negahban11,Obozinski11,Wainwright09b}.
Here, we assume Gaussian random variables and that all tasks have exactly the same set of edges.
\begin{theorem} \label{thm:tight}
Assume that $\x\si{k}$ follows a zero-mean multivariate Gaussian distribution for each $k$.
Assume that we are given $M$ i.i.d. samples for each of the $K$ tasks.
Set the regularization parameter in Eq.~\eqref{eq:multitaskggm} to $\rho = \eps K^{1/\dual{p}}$ for some $\eps \in (0,1/40)$.
Fix $\delta \in (0,1)$.
There is a true model $\BOmegat$ with support $\Supt$, such that if the sample size $M$ satisfies the bound:
\begin{align} \label{eq:tightsufficient}
M \geq \frac{8}{\eps^2} \left(\log K + 2\log N + \log{\frac{2}{\delta}}\right) \; ,
\end{align}
then the support union is correctly recovered (i.e., $\Sup_{\BOmegah} = \Supt$) with probability at least $1-\delta$.
Furthermore, if the sample size $M$ satisfies the bound:
\begin{align} \label{eq:tightnecessary}
M \leq \frac{1}{\eps^2 K^{2/\dual{p}}} \left(1+\log{\frac{1}{1-\delta^{2/(NK)}}} + \sqrt{1+2\log{\frac{1}{1-\delta^{2/(NK)}}}}\right) \; ,
\end{align}
then the support union is not correctly recovered (i.e., $\Sup_{\BOmegah} \neq \Supt$) with probability at least $1-\delta$.
\end{theorem}
Note that for both multi-task and single-task problems, the sufficient number of samples for the correct support union recovery in Eq.\eqref{eq:tightsufficient} is $\O(\log K + \log N)$.
For the single-task problem ($p=1$, $\dual{p}=\infty$), the necessary number of samples for the correct support union recovery in Eq.\eqref{eq:tightnecessary} is $\O(\log K + \log N)$.
For the $\ell_{1,2}$ multi-task problem ($p=2$, $\dual{p}=2$), the bound in Eq.\eqref{eq:tightnecessary} is $\O(K^{-1}(\log K + \log N))$.
For the $\ell_{1,\infty}$ problem ($p=\infty$, $\dual{p}=1$), the bound in Eq.\eqref{eq:tightnecessary} is $\O(K^{-2}(\log K + \log N))$.
Thus, as more tasks are available, multi-task regularization requires less samples in order to avoid failure.

\subsection{Information-Theoretic Lower Bound}

In the previous analyses, we provided the sufficient number of samples for correctly recovering the true model.
Here, we analyze the necessary number of samples for any conceivable algorithm.
Since the following result is method-independent, it does not make use of any specific regularization technique.
Next, we state our result based on information-theoretic arguments.
We assume Gaussian random variables and models generated uniformly at random from a finite ensemble, as in \citep{Wang10}.
For simplicity, we assume that all tasks have exactly the same set of edges.
\begin{theorem} \label{thm:inftheorylower}
Fix a node-degree $d$.
Let $u$ be the index of a model with precision matrices $\BOmegat\si{u,k}$ for each task $k$, each with support $\Supt\si{u}$ of node-degree $d$.
Assume that $u$ is chosen uniformly at random from a finite ensemble of models $\U$.
Furthermore, assume that the following holds simultaneously for all models $u \in \U$:
\begin{align*}
\min_{k,(n_1,n_2) \in \Supt\si{u}}{\frac{|\omegat_{n_1n_2}\si{u,k}|}{\sqrt{\omegat_{n_1n_1}\si{u,k} \omegat_{n_2n_2}\si{u,k}}}} \geq \lambda \; .
\end{align*}
For a fixed true model $u$, assume that $\x\si{k}$ follows a zero-mean multivariate Gaussian distribution for each $k$.
Assume that we are given a data set $\DD$ of $M$ i.i.d. samples for each of the $K$ tasks.
Assume that an algorithm acts on the data set $\DD$ and returns a precision matrix $\BOmega\si{k}(\DD) \succ \zero$ for each task $k$, with support union $\Sup(\DD)$.
Fix $\lambda \in (0,1/2]$.
There is an ensemble of models $\U$, such that if the sample size $M$ satisfies the bound:
\begin{align} \label{eq:inftheorylower1}
M < \min{\left( \frac{\log{\binom{KN-d}{2}}-1}{4\lambda^2} \; , \; \frac{\log{\binom{KN}{d}}-1}{\frac{1}{2}\left(\log{(1+\frac{d\lambda}{1-\lambda})}-\frac{d\lambda}{1+(d-1)\lambda}\right)} \right)} \; ,
\end{align}
\noindent then the support union is not correctly recovered with probability at least $1/2$.
That is:
\begin{align*}
\P_{u,\DD}\left[ \Sup(\DD) \neq \Supt\si{u} \right] \geq 1/2 \; .
\end{align*}
Fix $\eps \in (0,1/4]$.
There is an ensemble of models $\U$, such that if the sample size $M$ satisfies the bound:
\begin{align} \label{eq:inftheorylower2}
M < \frac{\log K+\log N+\log{(d/4)}-2}{4\eps^2} \; ,
\end{align}
\noindent then the element-wise $\ell_\infty$-norm is large with probability at least $1/2$.
That is:
\begin{align*}
\P_{u,\DD}\left[ (\exists k){\rm\ }\|\BOmega\si{k}(\DD) - \BOmegat\si{u,k}\|_\infty > \eps \right] \geq 1/2 \; .
\end{align*}
\end{theorem}
Note that the bounds in Eq.\eqref{eq:inftheorylower1} and Eq.\eqref{eq:inftheorylower2} are $\O(\log K+\log N)$.
Thus, the necessary number of samples for correct support union recovery in Theorem \ref{thm:inftheorylower} matches the rate of the sufficient number of samples in Theorem \ref{thm:consistencysubgaussian} for sub-Gaussian variates.
Therefore, the polynomial-time method of $\ell_{1,p}$ regularization achieves optimal rates, up to constant factors.

\section{Block Coordinate Descent Method} \label{sec:bcd}

In this section, we present a block coordinate descent method for the $\ell_{1,p}$ multi-task structure learning problem.
There are a plethora of optimization methods to solve general non-smooth convex optimization problems, cf., \citep{Duchi09,Duchi10,Nemirovski09,Xiao10,Yu08}.
In this paper, we apply a block coordinate descent method on the primal problem \citep{Honorio09,Honorio10b,Honorio12}.

Our block coordinate descent algorithm is as follows.
In an outer loop, we cyclically maximize with respect to one row/column of all precision matrices $\BOmega\si{k}$ at a time.
In an inner loop, we cyclically optimize with respect to one entry of such row/column.
In our derivations, the row/column will be the vector $\y\si{k}$ corresponding to the off-diagonal elements of $\BOmega\si{k}$ and the scalar $z\si{k}$ corresponding to the diagonal element of $\BOmega\si{k}$.
The entry of the off-diagonal vector $\y\si{k}$ will be the scalar $x_k$.

More formally, for solving the $\ell_{1,p}$ multi-task structure learning problem in Eq.~\eqref{eq:multitaskggm}, we maximize with respect to one row/column of all precision matrices $\BOmega\si{k}$ at a time.
Without loss of generality, we use the last row/column in our presentation, since permutation of rows and columns is always possible.
Let:
\begin{align} \label{DerivVars}
\BOmega\si{k} = \left[ \begin{array}{cc}
\W\si{k} & \y\si{k} \\
\t{\y\si{k}} & z\si{k}
\end{array} \right]
\; , \;
\BSigmah\si{k} = \left[ \begin{array}{cc}
\S\si{k} & \u\si{k} \\
\t{\u\si{k}} & v\si{k}
\end{array} \right] \; ,
\end{align}
\noindent where $\W\si{k},\S\si{k}\in \R^{(N-1) \times (N-1)}$, $\u\si{k}\in \R^{N-1}$ are constants, and $\y\si{k} \in \R^{N-1}, z\si{k}$ are the variables to be optimized.

In terms of the variables $\y\si{k},z\si{k}$ and the constant matrix $\W\si{k}$, the multi-task structure learning problem in Eq.~\eqref{eq:multitaskggm} can be reformulated as:
\begin{align} \label{DerivMultiTaskGGM}
\max_{(\forall k){\rm\ }\BOmega\si{k}\succ \zero}\left(\begin{array}{l}
  \sum_k{T\si{k}\left(\log(z\si{k}-\t{\y\si{k}}\inv{\W\si{k}}\y\si{k}) -2\t{\u\si{k}}\y\si{k}-v\si{k}z\si{k}\right)} \\
  -2\rho\sum_n\|(y_n\si{1},\dots,y_n\si{K})\|_p
\end{array}\right) \; .
\end{align}
Given the above, the optimal setting of $z\si{1},\dots,z\si{K}$ can be performed in closed form, as we show in the following theorem.
Furthermore, the generated solution $\BOmega\si{k}$ is positive definite as far as the update of $z\si{k}$ is always executed after updating $\y\si{k}$.
\begin{theorem} \label{thm:diagonalstep}
The ``diagonal update step'' of the block coordinate descent method for the $\ell_{1,p}$ multi-task structure learning problem in Eq.~\eqref{eq:multitaskggm} is equivalent to setting:
\begin{align} \label{eq:diagonalstep}
(\forall k){\rm\ }{z\si{k}}^* = \frac{1}{v\si{k}}+\t{\y\si{k}}\inv{\W\si{k}}\y\si{k} \; .
\end{align}
Moreover, the block coordinate descent method generates a sequence of positive definite solutions.
\end{theorem}

In order to optimize the Eq.~\eqref{DerivMultiTaskGGM} with respect to $\y\si{1},\dots,\y\si{K}$, we optimize with respect to one entry at a time.
Without loss of generality, we use the last entry in our presentation, i.e., $\x = \t{(y_{N-1}\si{1},\dots,y_{N-1}\si{K})}$.
As we show in the following theorem, our block coordinate descent algorithm solves a sequence of $\ell_p$ regularized quadratic minimization subproblems parametrized by $\x \in \R^K$.
\begin{theorem} \label{thm:lpquad}
The ``off-diagonal update step'' of the block coordinate descent method for the $\ell_{1,p}$ multi-task structure learning problem in Eq.~\eqref{eq:multitaskggm} is equivalent to solving a sequence of strictly convex $\ell_p$ regularized separable quadratic subproblems:
\begin{align} \label{eq:lpquad}
\min_{\x \in \mathbb{R}^K}\left(\frac{1}{2}\t{\x}\diag(\q)\x-\t{\c}\x+\rho\|\x\|_p\right) \; ,
\end{align}
\noindent where $\x = \t{(y_{N-1}\si{1},\dots,y_{N-1}\si{K})}$.
The factors $\q>\zero$ and $\c$ are defined in terms of the constants $T\si{k}$, $v\si{k}$, $\inv{\W\si{k}}$ and $\u\si{k}$, as well as the constant part of $\y\si{k}$.
\end{theorem}
The regularized separable quadratic problem in Eq.~\eqref{eq:lpquad} has two special instances which can be solved either in closed form or by one-dimensional optimization.
As shown in the following theorem,\footnote{The proof connects the original problem to the continuous quadratic knapsack problem.} Eq.~\eqref{eq:lpquad} can be solved in closed form when $p=\infty$.
\begin{theorem} \label{thm:quadknapsack}
Let $\pi$ be a permutation of the indices ${1,2,\dots,K}$ that sorts the breakpoints in decreasing order, i.e.,
$\frac{|c_{\pi_1}|}{q_{\pi_1}} \geq
\frac{|c_{\pi_2}|}{q_{\pi_2}} \geq \dots \geq
\frac{|c_{\pi_K}|}{q_{\pi_K}} \geq
\frac{|c_{\pi_{K+1}}|}{q_{\pi_{K+1}}}\equiv 0$.
Let ${g_k(\nu)=\max(0,|c_k|-\nu q_k)}$ and $k^*$ be the range in which:
\begin{align} \label{eq:quadknapsackrange}
\t{\one}\g\left(\frac{|c_{\pi_{k^*}}|}{q_{\pi_{k^*}}}\right) \leq \rho \leq \t{\one}\g\left(\frac{|c_{\pi_{k^*+1}}|}{q_{\pi_{k^*+1}}}\right) \; .
\end{align}
For $\q>\zero$, $\rho >0$, $p=\infty$, the $\ell_p$ regularized separable quadratic problem in Eq.~\eqref{eq:lpquad} has the optimal solution:
\begin{align} \label{LinfQuadSolution}
\|\c\|_1\leq \rho & \Rightarrow \x^*=\zero \nonumber \\
\|\c\|_1> \rho \text{\ and\ } k>k^* & \Rightarrow \textstyle{ x_{\pi_k}^* = \frac{c_{\pi_k}}{q_{\pi_k}} } \nonumber \\
\|\c\|_1> \rho \text{\ and\ } k\leq k^* & \Rightarrow \textstyle{ x_{\pi_k}^* = \sgn(c_{\pi_k})\frac{\sum_{k=1}^{k^*}{|c_{\pi_k}|}-\rho}{\sum_{k=1}^{k^*}{q_{\pi_k}}} } \; .
\end{align}
\end{theorem}
As shown in the following theorem,\footnote{The proof connects the original problem to the separable quadratic trust-region problem.} Eq.~\eqref{eq:lpquad} can be solved by the one-dimensional Newton-Raphson method when $p=2$.
\begin{theorem} \label{thm:quadtrustreg}
Let $\lambda^*$ be the optimal solution to the one-dimensional problem:
\begin{align} \label{eq:quadtrustregdual}
\min_{\lambda\geq 0}{\left(\sum_n\frac{c_n^2}{q_n+\lambda q_n^2}+\rho^2\lambda\right)} \; .
\end{align}
For $\q>\zero$, $\rho >0$, $p=2$, the $\ell_p$ regularized separable quadratic problem in Eq.~\eqref{eq:lpquad} has the optimal solution:
\begin{align} \label{L2QuadSolution}
\|\c\|_2\leq \rho & \Rightarrow \x^*=\zero \nonumber \\
\|\c\|_2> \rho & \Rightarrow \x^*=\lambda^*\inv{\diag(\one+\lambda^*\q)}\c \; .
\end{align}
\end{theorem}

\begin{algorithm}[t]
\begin{footnotesize}
\caption{\textbf{Block coordinate descent algorithm for $\ell_{1,p}$ multi-task learning of Gaussian graphical models.}
A Matlab implementation of the algorithm is available at \url{http://people.csail.mit.edu/jhonorio/ggms.zip} }
\label{Algorithm}
\begin{algorithmic}
    \STATE {\bfseries Input:} $\rho >0$, for each $k$, $\BSigmah\si{k}\succeq \zero$, $T\si{k}>0$
    \STATE Initialize for each $k$, $\BOmega\si{k}=\inv{\diag(\BSigmah\si{k})}$
    \FOR{ each iteration $1,\dots,L$ and each variable $1,\dots,N$}
        \STATE Split for each $k$, $\BOmega\si{k}$ into $\W\si{k},\y\si{k},z\si{k}$ and $\BSigmah\si{k}$ into $\S\si{k},\u\si{k},v\si{k}$ as described in Eq.~\eqref{DerivVars}
        \STATE Update for each $k$, $\inv{\W\si{k}}$ by using the Sherman-Woodbury-Morrison formula (Note that when iterating from one variable to the next one, only one row/column change on matrix $\W\si{k}$)
        \FOR{ each entry $n = 1,\dots,N-1$ and $\x=\t{(y_n\si{1},\dots,y_n\si{K})}$}
            \STATE For $p=\infty$, solve the $\ell_\infty$ regularized separable quadratic problem by Theorem \ref{thm:quadknapsack}, either by sorting the breakpoints or using medians of breakpoint subsets. For $p=2$, solve the $\ell_2$ regularized separable quadratic problem by Theorem \ref{thm:quadtrustreg}, by using the Newton-Raphson method.
        \ENDFOR
        \STATE Update for each $k$, $z\si{k} \leftarrow \frac{1}{v\si{k}}+\t{\y\si{k}}\inv{\W\si{k}}\y\si{k}$ (Note that after this step, $\BOmega\si{k}\succ\zero$)
    \ENDFOR
    \STATE {\bfseries Output:} for each $k$, $\BOmega\si{k}\succ \zero$
\end{algorithmic}
\end{footnotesize}
\end{algorithm}

Algorithm \ref{Algorithm} shows the block coordinate descent method in detail.
The algorithm has a time complexity of $\O(LN^3K)$ for $L$ iterations, $N$ variables and $K$ tasks.
In our experiments, we used $L=10$ iterations and observed that the algorithm usually reaches a plateau.
Algorithm \ref{Algorithm} is provably convergent.
Additionally, as a preprocessing step, we can reduce the size of the original problem by removing nodes that are not endpoints of any edge in the optimal solution.
(See Appendix \ref{sec:algorithm} for details.)

\section{Experimental Results} \label{sec:results}

In this section, we present our experiments on synthetic as well as real-world data sets.
For comparison purposes, we used the graphical lasso \citep{Friedman07} as the \emph{single-task} method.
We used the graphical lasso under two scenarios: by either pooling the data from all the tasks together and learning a single model, or by learning models independently per task.
We also compared our results with the methods of \citet{Chiquet11} (\url{http://cran.r-project.org/web/packages/simone/}), \citet{Guo11} and \citet{Varoquaux10}.
We did not include the method of \citet{Mohan14}, since their particular implementation (\url{http://faculty.washington.edu/mfazel/}) considers only two tasks.
Fortunately, as we discussed before, the regularizer in \citep{Mohan14} subsumes the problem that we analyze here, since it includes an additional $\ell_1$-norm regularization.
Thus, it does not differ much from the methods we experimentally compared here.

\subsection{Synthetic Experiments}

\begin{figure}
\begin{center}
\includegraphics[width=0.5\linewidth]{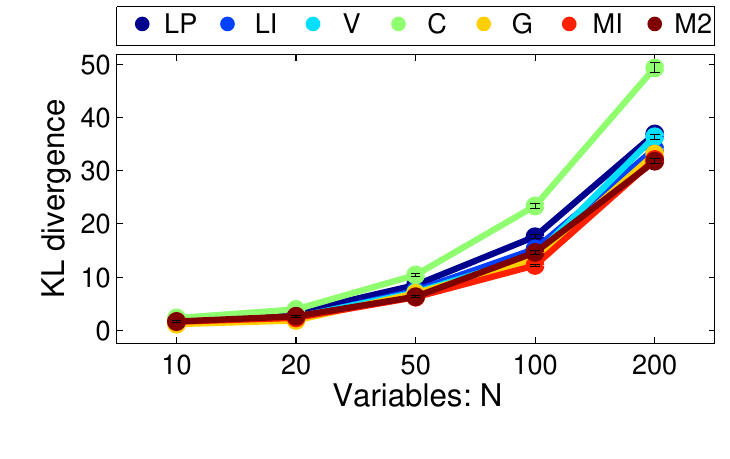}\includegraphics[width=0.5\linewidth]{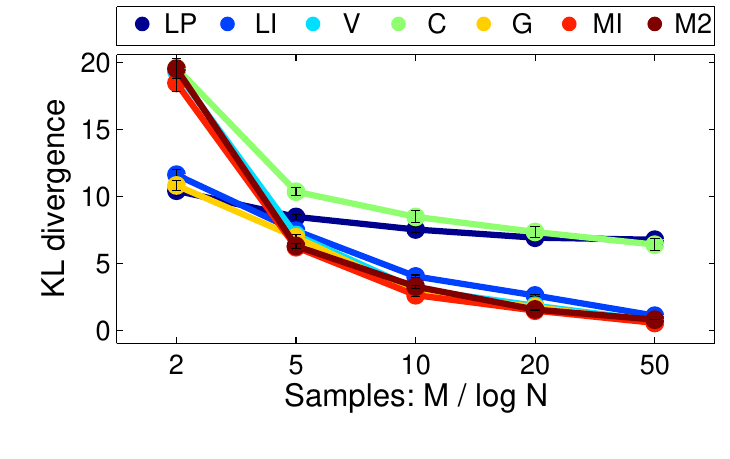} \\
\vspace{-0.25in}\makebox[0.5\linewidth]{\footnotesize \makebox[0.35in]{} (a) $M=5\log N$, $K=5$, overlap 1, density .05}\makebox[0.5\linewidth]{\footnotesize \makebox[0.35in]{} (b)  $N=50$, $K=5$, overlap 1, density .05} \\
\vspace{-0.1in}\makebox[1\linewidth]{} \\
\includegraphics[width=0.5\linewidth]{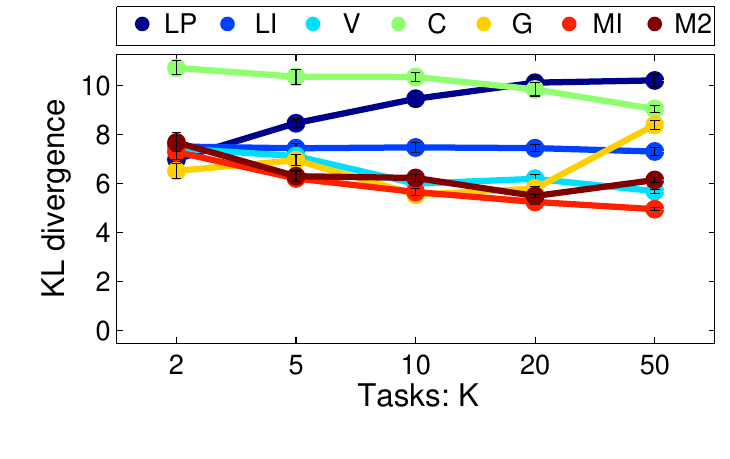}\includegraphics[width=0.5\linewidth]{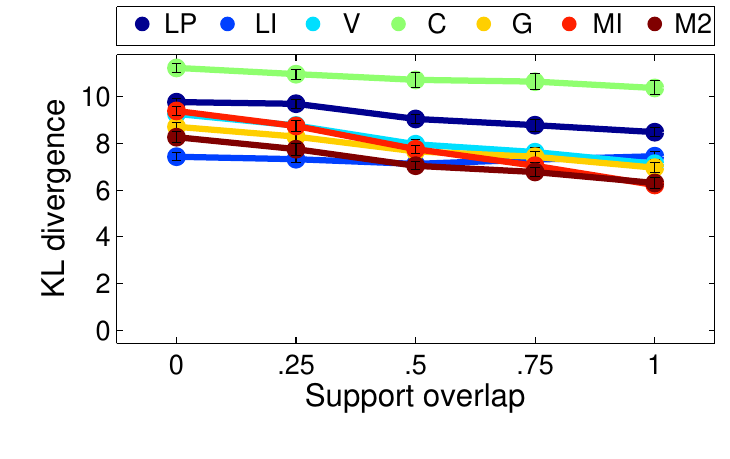} \\
\vspace{-0.25in}\makebox[0.5\linewidth]{\footnotesize \makebox[0.35in]{} (c) $N=50$, $M=5\log N$, overlap 1, density .05}\makebox[0.5\linewidth]{\footnotesize \makebox[0.35in]{} (d)  $N=50$, $M=5\log N$, $K=5$, density .05} \\
\vspace{-0.1in}\makebox[1\linewidth]{} \\
\includegraphics[width=0.5\linewidth]{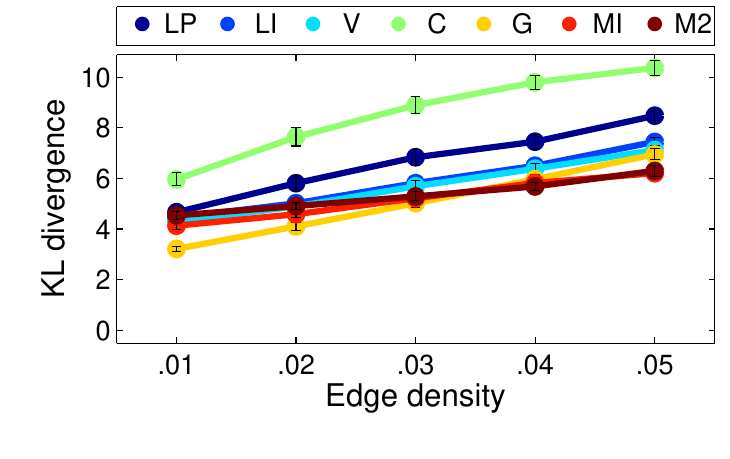} \\
\vspace{-0.25in}\makebox[0.5\linewidth]{\footnotesize \makebox[0.35in]{} (e) $N=50$, $M=5\log N$, $K=5$, overlap 1} \\
\end{center}
\vspace{-0.2in}
\captionx{
Kullback-Leibler (KL) divergence between the models and the ground truth, for different (a) number of variables, (b) number of samples, (c) number of tasks, (d) support overlap between tasks and (e) edge density.
We include the graphical lasso, by either pooling the data from all the tasks together and learning a single model (LP), or by learning models independently per task (LI); the methods of \citeauthor{Varoquaux10} (V), \citeauthor{Chiquet11} (C), \citeauthor{Guo11} (G), and the $\ell_{1,\infty}$ (MI) and $\ell_{1,2}$ (M2) multi-task methods.
(Error bars at $90\%$ significance level.
The regularization parameter was selected in a validation set.)
In most of the cases, our multi-task methods (MI, M2) have equal or better (lower) KL divergence than the comparison methods (LP,LI,V,C,G).
In some cases, our multi-task methods (MI,M2) are outperformed by the comparison methods, specifically for very small number of samples, for low support overlap between tasks and for low edge density.
}
\label{fig:kldsynthetic}
\end{figure}

\begin{figure}
\begin{center}
\includegraphics[width=0.5\linewidth]{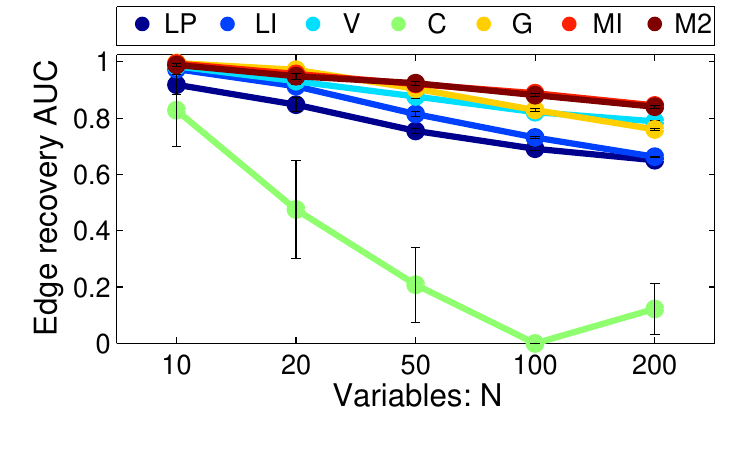}\includegraphics[width=0.5\linewidth]{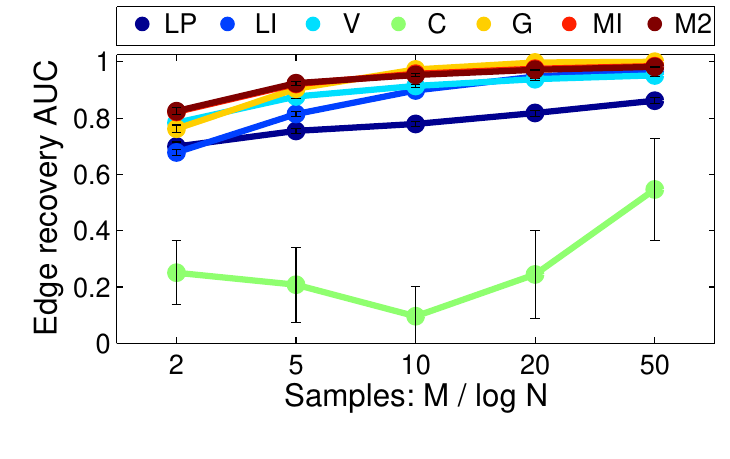} \\
\vspace{-0.25in}\makebox[0.5\linewidth]{\footnotesize \makebox[0.35in]{} (a) $M=5\log N$, $K=5$, overlap 1, density .05}\makebox[0.5\linewidth]{\footnotesize \makebox[0.35in]{} (b)  $N=50$, $K=5$, overlap 1, density .05} \\
\vspace{-0.1in}\makebox[1\linewidth]{} \\
\includegraphics[width=0.5\linewidth]{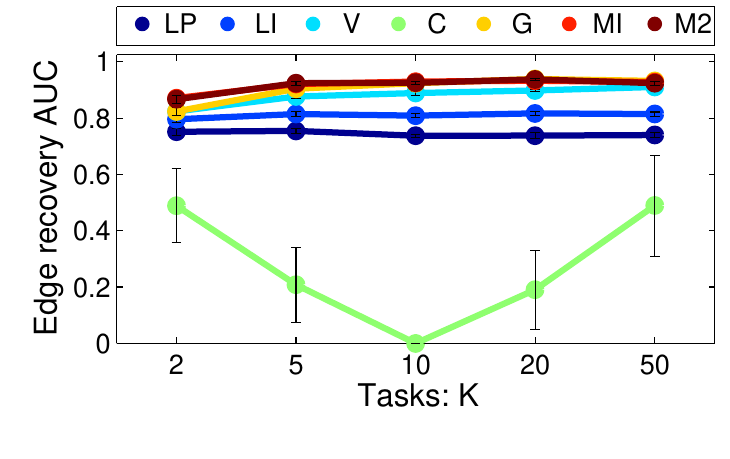}\includegraphics[width=0.5\linewidth]{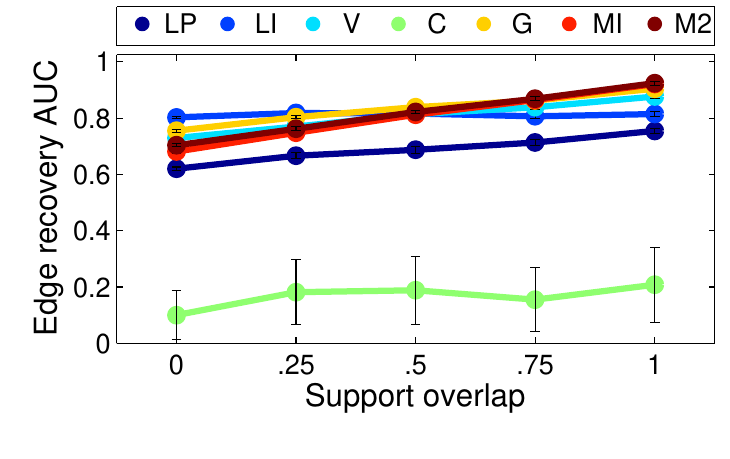} \\
\vspace{-0.25in}\makebox[0.5\linewidth]{\footnotesize \makebox[0.35in]{} (c) $N=50$, $M=5\log N$, overlap 1, density .05}\makebox[0.5\linewidth]{\footnotesize \makebox[0.35in]{} (d)  $N=50$, $M=5\log N$, $K=5$, density .05} \\
\vspace{-0.1in}\makebox[1\linewidth]{} \\
\includegraphics[width=0.5\linewidth]{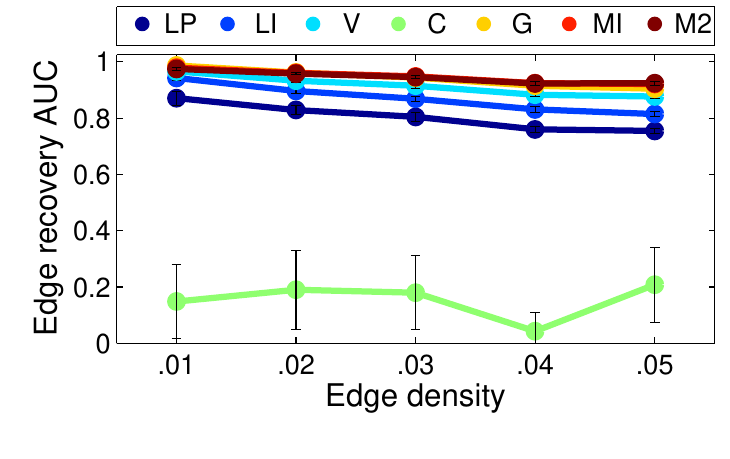} \\
\vspace{-0.25in}\makebox[0.5\linewidth]{\footnotesize \makebox[0.35in]{} (e) $N=50$, $M=5\log N$, $K=5$, overlap 1} \\
\end{center}
\vspace{-0.2in}
\captionx{
Area under the curve (AUC) of sensitivity vs. specificity of edge recovery for learnt models as the regularization level is varied, for different (a) number of variables, (b) number of samples, (c) number of tasks, (d) support overlap between tasks and (e) edge density.
We include the graphical lasso, by either pooling the data from all the tasks together and learning a single model (LP), or by learning models independently per task (LI); the methods of \citeauthor{Varoquaux10} (V), \citeauthor{Chiquet11} (C), \citeauthor{Guo11} (G), and the $\ell_{1,\infty}$ (MI) and $\ell_{1,2}$ (M2) multi-task methods.
(Error bars at $90\%$ significance level.)
In most of the cases, our multi-task methods (MI, M2) have equal or better (higher) AUC than the comparison methods (LP,LI,V,C,G).
In some cases, our multi-task methods (MI,M2) are outperformed by the comparison methods, specifically for low support overlap between tasks.
}
\label{fig:aucsynthetic}
\end{figure}

We begin with a synthetic example to test the ability of the method to recover the ground truth structure from data.
The model contains $N \in \{10,20,50,100,200\}$ variables and $K \in \{2,5,10,20,50\}$ tasks.
For each of $20$ repetitions, we perform the following steps.
For each task $k$, we generate a Gaussian graphical model $\BOmega\si{k}_g$ with a required edge density ($0.01,0.02,\dots,0.05$).
The weight of each edge in $\BOmega\si{k}_g$ is generated uniformly at random from $\{-1,+1\}$.
We ensure that the support overlap (i.e., the intersection of the edge sets across tasks) is equal to a prescribed proportion ($0,0.25,\dots,1$) of the support union.
We guarantee positive definiteness of $\BOmega\si{k}_g$ by verifying that its minimum eigenvalue is at least $0.1$.
We then generate $M \in \{2 \log N,5 \log N,10 \log N,20 \log N,50 \log N\}$ samples for training and $M$ samples for validation, from each Gaussian graphical model $\BOmega\si{k}_g$.
We use the validation set for selecting the optimal regularization parameter.
Since all the tasks have the same number of samples, we set $(\forall k){\rm\ }T\si{k} = 1$.

Figure \ref{fig:kldsynthetic} shows the Kullback-Leibler divergence between the models and the ground truth, for different number of variables, samples, tasks, support overlap between tasks, and edge density.
The regularization parameter was selected in a validation set.
Figure \ref{fig:aucsynthetic} shows the area under the curve of sensitivity vs. specificity of edge recovery for learnt models as the regularization level is varied, under the same scenarios as in Figure \ref{fig:kldsynthetic}.
Both our $\ell_{1,\infty}$ and $\ell_{1,2}$ multi-task methods produce better probability distributions (lower Kullback-Leibler divergence) than the comparison methods, except in the case of very small number of samples ($<5\log N$), low support overlap between tasks ($<0.5$), or for low edge density ($<0.03$).
Additionally, our multi-task methods recover the ground truth edges better (higher area under the curve) than the comparison methods, except in the case of low support overlap between tasks ($<0.5$).

\subsection{Real-World Data Sets}

\begin{table}
\begin{center}
\begin{footnotesize}
\begin{tabular}{l@{\hspace{0.07in}}r@{\hspace{0.07in}}r@{\hspace{0.07in}}l@{\hspace{0.07in}}r@{\hspace{0.07in}}r@{\hspace{0.07in}}l@{\hspace{0.07in}}r@{\hspace{0.07in}}r}
  \hline
  \textbf{Site} & \textbf{Subjects} & \textbf{Scans} & \textbf{Site} & \textbf{Subjects} & \textbf{Scans} & \textbf{Site} & \textbf{Subjects} & \textbf{Scans} \\
  \hline
  AnnArbor\_a & 23 & 295 & Cleveland1 & 17 & 125 & NewYorkA2 & 24 & 192 \\
  Baltimore & 46 & 120 & Cleveland2 & 14 & 125 & NewYorkB & 20 & 168 \\
  Bangor & 20 & 256 & Dallas & 24 & 114 & Newark & 19 & 135 \\
  Beijing1 & 40 & 225 & ICBM & 42 & 128 & Ontario & 11 & 100 \\
  Beijing2 & 42 & 225 & Leiden1 & 12 & 210 & Orangeburg & 20 & 162 \\
  Beijing3 & 41 & 225 & Leiden2 & 19 & 210 & Oulu1 & 57 & 243 \\
  Beijing4 & 30 & 225 & Leipzig & 37 & 192 & Oulu2 & 47 & 243 \\
  Beijing5 & 45 & 225 & NYU\_TRT1A & 13 & 192 & Oxford & 22 & 175 \\
  Berlin & 26 & 192 & NYU\_TRT1B & 12 & 192 & PaloAlto & 17 & 234 \\
  Cambridge1 & 48 & 117 & NYU\_TRT2A & 13 & 192 & Queensland & 18 & 189 \\
  Cambridge2 & 46 & 117 & NYU\_TRT2B & 12 & 192 & SaintLouis & 31 & 125 \\
  Cambridge3 & 49 & 117 & NYU\_TRT3A & 13 & 192 & Taipei\_a & 13 & 256 \\
  Cambridge4 & 55 & 117 & NYU\_TRT3B & 12 & 192 & Taipei\_b & 8 & 160 \\
  CambridgeWG & 35 & 144 & NewYorkA1 & 35 & 192 & & & \\
  \hline
\end{tabular}
\end{footnotesize}
\end{center}
\vspace{-0.15in}
\captionx{Number of subjects per collection site and number of scans per subject in the \emph{1000 functional connectomes} data set.}
\label{tab:sitesconnectomes}
\end{table}

\begin{table}
\begin{center}
\begin{footnotesize}
\begin{tabular}{l@{\hspace{0.07in}}l@{\hspace{0.07in}}r}
  \hline
  \textbf{Abbreviation} & \textbf{Cancer Type} & \textbf{Subjects} \\
  \hline
  BRCA & Breast invasive carcinoma & 590 \\
  COAD & Colon adenocarcinoma & 174 \\
  GBM & Glioblastoma multiforme & 595 \\
  LUSC & Lung squamous cell carcinoma & 155 \\
  OV & Ovarian serous cystadenocarcinoma & 590 \\
  \hline
\end{tabular}
\end{footnotesize}
\end{center}
\vspace{-0.15in}
\captionx{Number of subjects per cancer type in the \emph{cancer genome atlas} data set.}
\label{tab:typescancer}
\end{table}

We chose two real-world data sets for experimental validation: the \emph{1000 functional connectomes} data set and the \emph{cancer genome atlas} data set.

The \emph{1000 functional connectomes} data set contains resting-state fMRI of $1128$ subjects collected on $41$ sites around the world.
The data set is publicly available at \url{http://www.nitrc.org/projects/fcon_1000/}.
Resting-state fMRI is a procedure that captures brain function of a subject that is at wakeful rest (i.e., not focused on the outside world).
Registration of the data set to the same spatial reference template (Talairach space) and spatial smoothing was performed in SPM2 (\url{http://www.fil.ion.ucl.ac.uk/spm/}).
We extracted voxels from the gray matter only, and grouped them into $157$ regions by using standard labels (See Appendix \ref{sec:regionsconnectomes}), given by the Talairach Daemon (\url{http://www.talairach.org/}).
These regions span the entire brain: cerebellum, cerebrum and brainstem.
In order to capture laterality effects, we have regions for the left and right side of the brain.
Table \ref{tab:sitesconnectomes} shows the number of subjects per collection site as well as the number of scans per subject.

The \emph{cancer genome atlas} data set contains gene expression data of $2360$ subjects for various types of cancer.
The data set is publicly available at \url{http://tcga-data.nci.nih.gov/tcga/}.
We used $187$ genes commonly regulated in cancer (See Appendix \ref{sec:genescancer}) that were identified on independent data sets by \citet{Lu07}.
Table \ref{tab:typescancer} shows the number of subjects per collection site as well as the number of scans per subject.

\begin{figure}
\begin{center}
\includegraphics[width=0.5\linewidth]{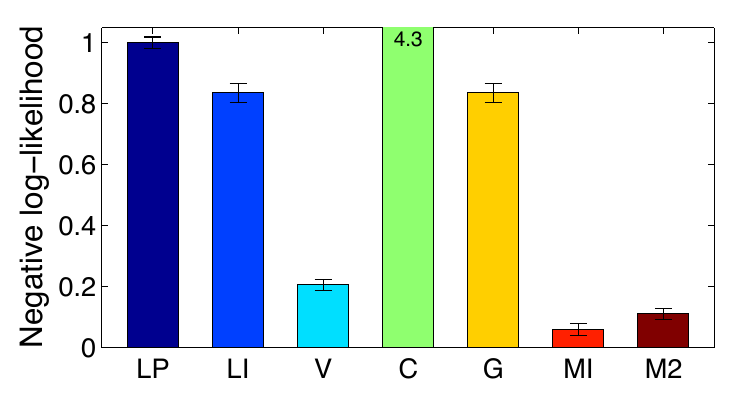}\includegraphics[width=0.5\linewidth]{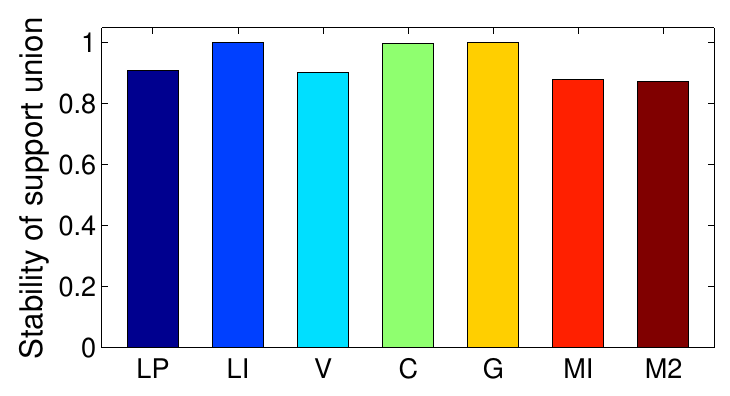} \\
\vspace{-0.125in}\makebox[0.5\linewidth]{\footnotesize \makebox[0.35in]{} (a)}\makebox[0.5\linewidth]{\footnotesize \makebox[0.35in]{} (b)} \\
\includegraphics[width=0.5\linewidth]{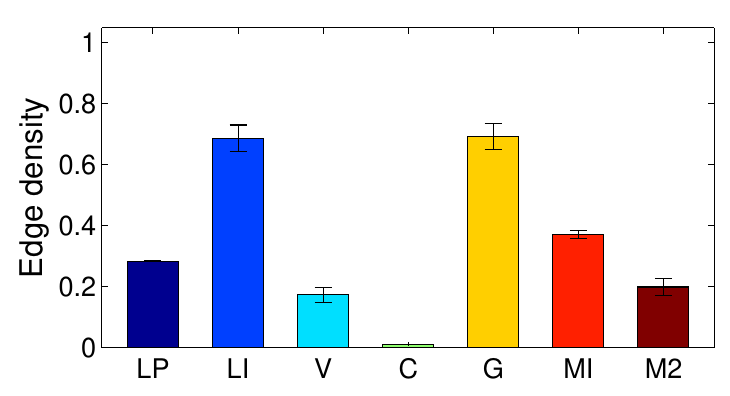}\includegraphics[width=0.5\linewidth]{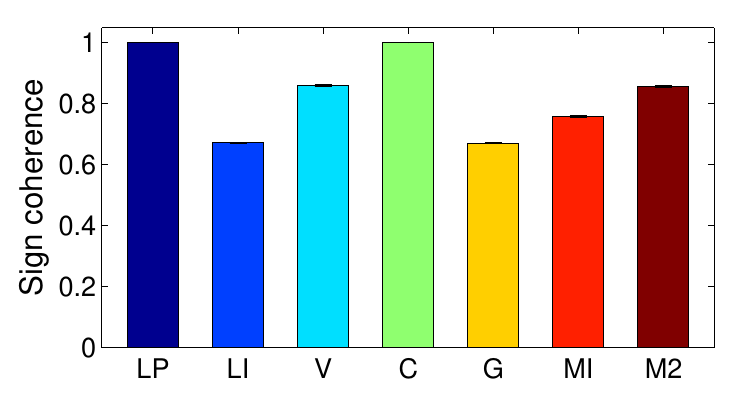} \\
\vspace{-0.125in}\makebox[0.5\linewidth]{\footnotesize \makebox[0.35in]{} (c)}\makebox[0.5\linewidth]{\footnotesize \makebox[0.35in]{} (d)} \\
\includegraphics[width=0.5\linewidth]{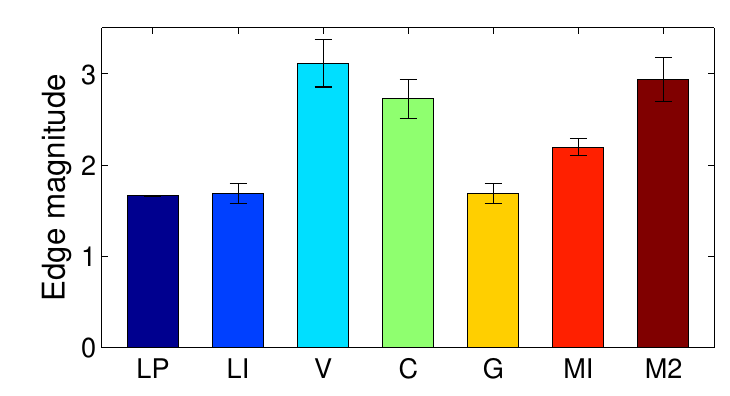}\includegraphics[width=0.5\linewidth]{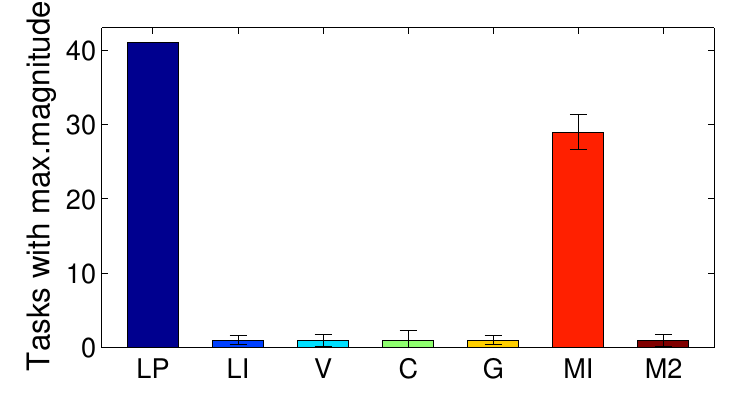} \\
\vspace{-0.125in}\makebox[0.5\linewidth]{\footnotesize \makebox[0.35in]{} (e)}\makebox[0.5\linewidth]{\footnotesize \makebox[0.35in]{} (f)} \\
\end{center}
\vspace{-0.2in}
\captionx{
(a) test negative log-likelihood, (b) stability of the support union across tasks, (c) edge density, (d) sign coherence across tasks, (e) average edge magnitude and (f) tasks where edges have maximum magnitude, for structures learnt for the \emph{1000 functional connectomes} data set.
We include the graphical lasso, by either pooling the data from all the tasks together and learning a single model (LP), or by learning models independently per task (LI); the methods of \citeauthor{Varoquaux10} (V), \citeauthor{Chiquet11} (C), \citeauthor{Guo11} (G), and the $\ell_{1,\infty}$ (MI) and $\ell_{1,2}$ (M2) multi-task methods.
(Error bars at $90\%$ significance level.
The regularization parameter was selected in a validation set.)
Our multi-task methods (MI, M2) have better (smaller) negative log-likelihood than the comparison methods (LP,LI,V,C,G).
Our multi-task methods (MI, M2) have relatively less stable support union than the comparison methods (LP,LI,V,C,G) since the latter produce different topologies per task and thus, a bigger support union.
M2 structures are sparser, more sign-coherent and have less tasks where edges have maximum magnitude than MI structures.
}
\label{fig:chartsconnectomes}
\end{figure}

\begin{figure}
\begin{center}
\includegraphics[width=0.5\linewidth]{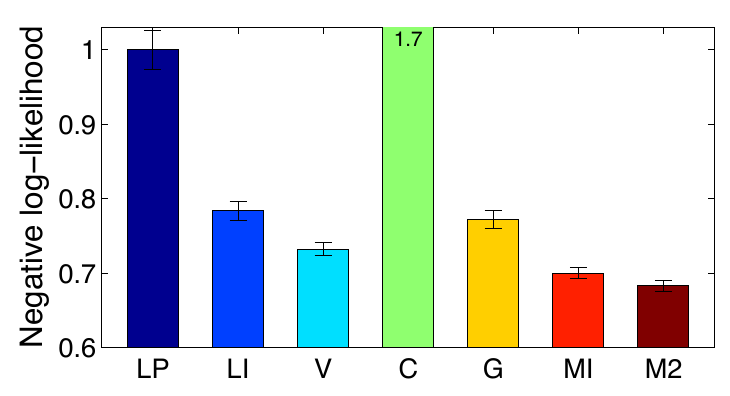}\includegraphics[width=0.5\linewidth]{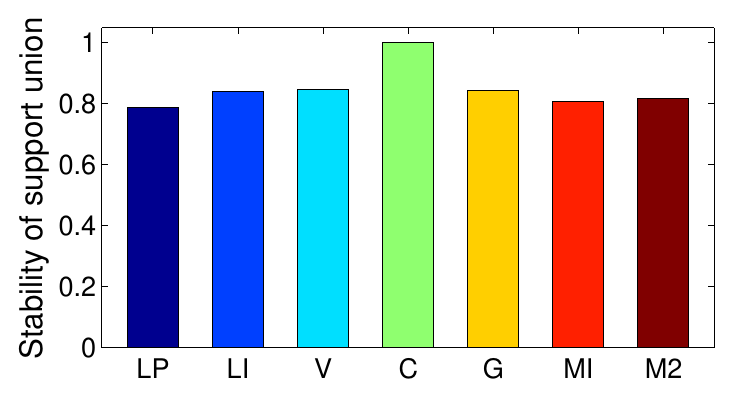} \\
\vspace{-0.125in}\makebox[0.5\linewidth]{\footnotesize \makebox[0.35in]{} (a)}\makebox[0.5\linewidth]{\footnotesize \makebox[0.35in]{} (b)} \\
\includegraphics[width=0.5\linewidth]{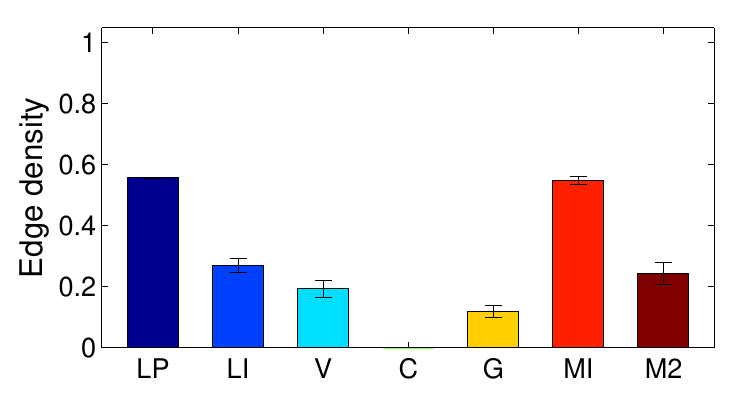}\includegraphics[width=0.5\linewidth]{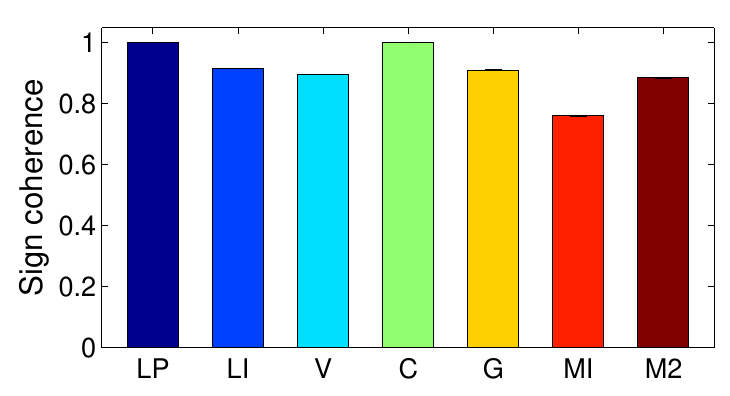} \\
\vspace{-0.125in}\makebox[0.5\linewidth]{\footnotesize \makebox[0.35in]{} (c)}\makebox[0.5\linewidth]{\footnotesize \makebox[0.35in]{} (d)} \\
\includegraphics[width=0.5\linewidth]{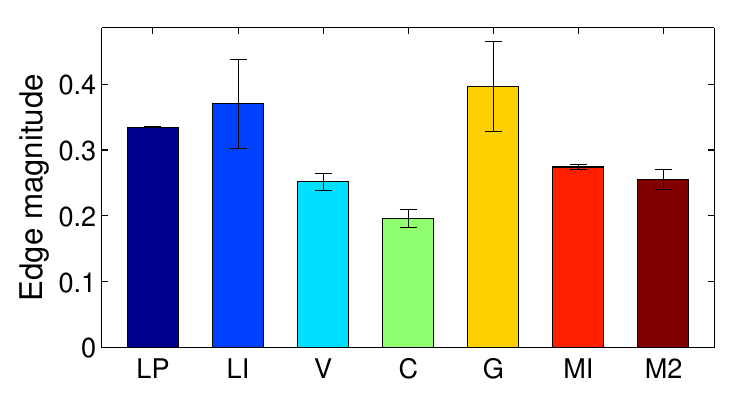}\includegraphics[width=0.5\linewidth]{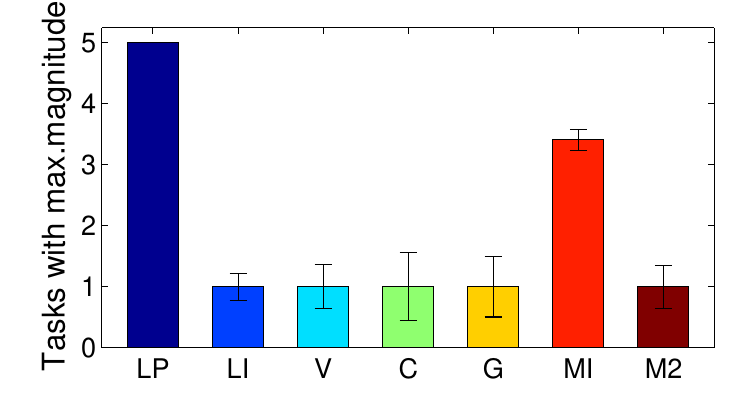} \\
\vspace{-0.125in}\makebox[0.5\linewidth]{\footnotesize \makebox[0.35in]{} (e)}\makebox[0.5\linewidth]{\footnotesize \makebox[0.35in]{} (f)} \\
\end{center}
\vspace{-0.2in}
\captionx{
(a) test negative log-likelihood, (b) stability of the support union across tasks, (c) edge density, (d) sign coherence across tasks, (e) average edge magnitude and (f) tasks where edges have maximum magnitude, for structures learnt for the \emph{cancer genome atlas} data set.
We include the graphical lasso, by either pooling the data from all the tasks together and learning a single model (LP), or by learning models independently per task (LI); the methods of \citeauthor{Varoquaux10} (V), \citeauthor{Chiquet11} (C), \citeauthor{Guo11} (G), and the $\ell_{1,\infty}$ (MI) and $\ell_{1,2}$ (M2) multi-task methods.
(Error bars at $90\%$ significance level.
The regularization parameter was selected in a validation set.)
Our multi-task methods (MI, M2) have better (smaller) negative log-likelihood than the comparison methods (LP,LI,V,C,G).
Our multi-task methods (MI, M2) have relatively less stable support union than the comparison methods (LP,LI,V,C,G) since the latter produce different topologies per task and thus, a bigger support union.
M2 structures are sparser, more sign-coherent and have less tasks where edges have maximum magnitude than MI structures.
}
\label{fig:chartscancer}
\end{figure}

\begin{figure}
\begin{center}
\raisebox{0.1in}{\footnotesize{(L)}}\includegraphics[width=0.25\linewidth,height=1.05in]{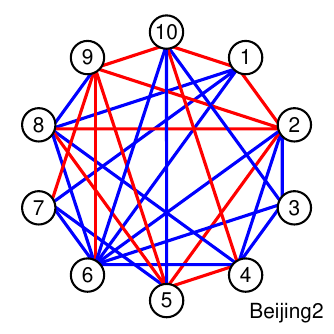}\includegraphics[width=0.25\linewidth,height=1.05in]{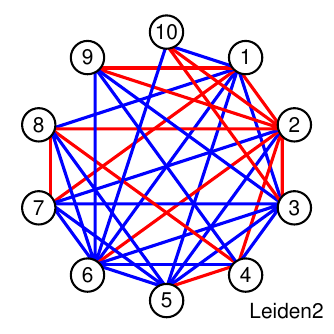}\includegraphics[width=0.25\linewidth,height=1.05in]{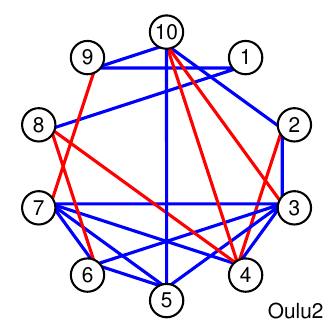} \\
\raisebox{0.1in}{\footnotesize{(V)}}\includegraphics[width=0.25\linewidth,height=1.05in]{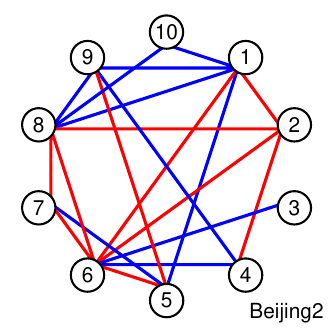}\includegraphics[width=0.25\linewidth,height=1.05in]{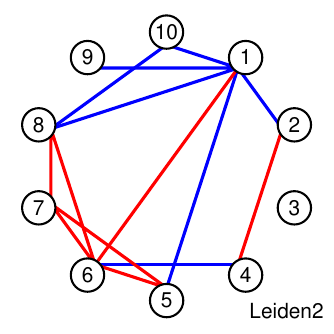}\includegraphics[width=0.25\linewidth,height=1.05in]{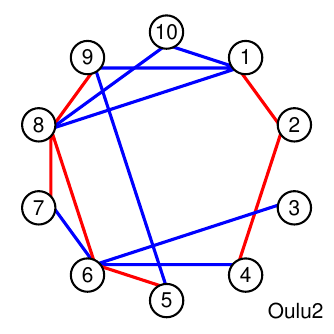} \\
\raisebox{0.1in}{\footnotesize{(C)}}\includegraphics[width=0.25\linewidth,height=1.05in]{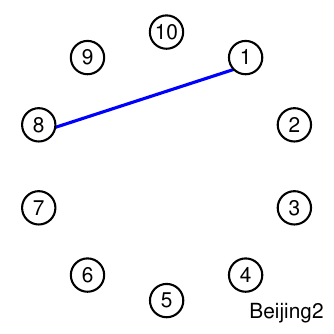}\includegraphics[width=0.25\linewidth,height=1.05in]{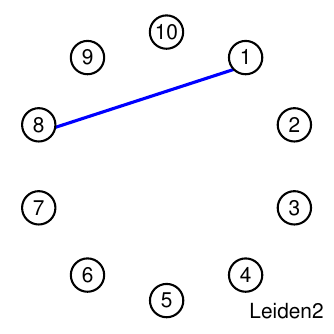}\includegraphics[width=0.25\linewidth,height=1.05in]{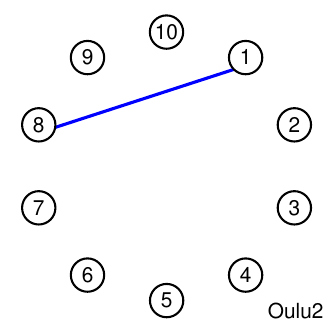} \\
\raisebox{0.1in}{\footnotesize{(G)}}\includegraphics[width=0.25\linewidth,height=1.05in]{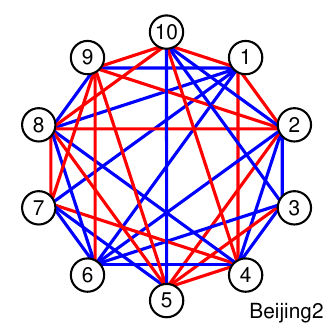}\includegraphics[width=0.25\linewidth,height=1.05in]{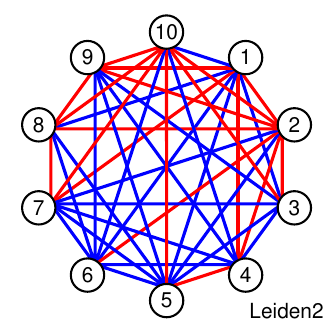}\includegraphics[width=0.25\linewidth,height=1.05in]{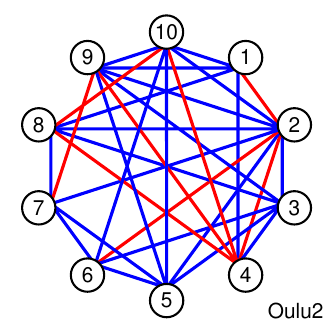} \\
\raisebox{0.1in}{\footnotesize{(MI)}}\includegraphics[width=0.25\linewidth,height=1.05in]{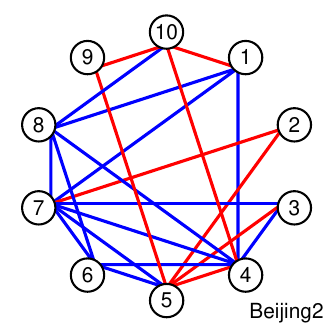}\includegraphics[width=0.25\linewidth,height=1.05in]{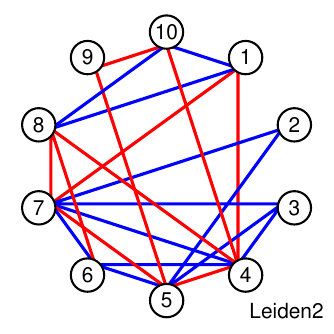}\includegraphics[width=0.25\linewidth,height=1.05in]{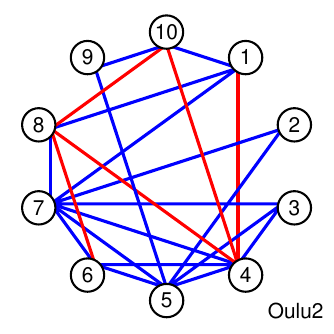} \\
\raisebox{0.1in}{\footnotesize{(M2)}}\includegraphics[width=0.25\linewidth,height=1.05in]{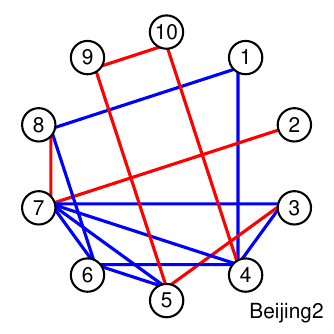}\includegraphics[width=0.25\linewidth,height=1.05in]{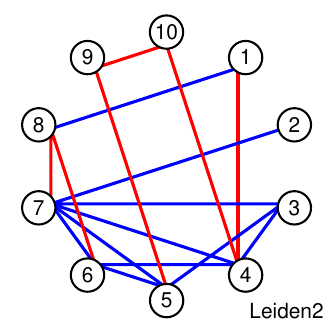}\includegraphics[width=0.25\linewidth,height=1.05in]{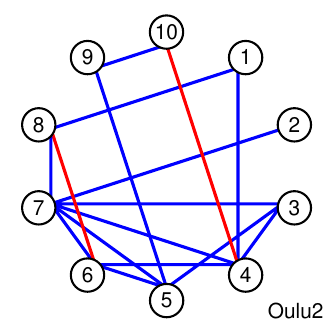} \\
\end{center}
\vspace{-0.2in}
\captionx{
Learnt structures from the \emph{1000 functional connectomes} data set for three randomly selected collection sites (Beijing2, Leiden2, Oulu2).
We show a subgraph of ten randomly selected brain regions (regions from 1 to 10: left Brodmann area 18, 22, 24, 32, mammillary body, putamen, ventral lateral nucleus, right Brodmann area 18, 31, culmen).
We include the graphical lasso (L), the methods of \citeauthor{Varoquaux10} (V), \citeauthor{Chiquet11} (C), \citeauthor{Guo11} (G), and the $\ell_{1,\infty}$ (MI) and $\ell_{1,2}$ (M2) multi-task methods.
(Positive interactions are shown in blue, negative interactions in red.
The regularization parameter was selected in a validation set.)
The methods C,MI,M2 produce a sparsity pattern that is consistent across collection sites, while the methods L,V,G fail to obtain a consistent sparsity pattern.
}
\label{fig:graphsconnectomes}
\end{figure}

\begin{figure}
\begin{center}
\raisebox{0.1in}{\footnotesize{(L)}}\includegraphics[width=0.25\linewidth,height=1.05in]{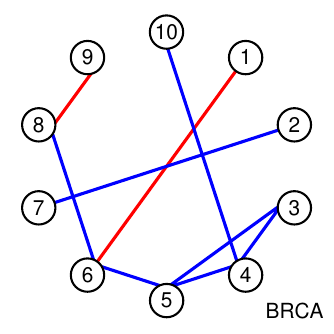}\includegraphics[width=0.25\linewidth,height=1.05in]{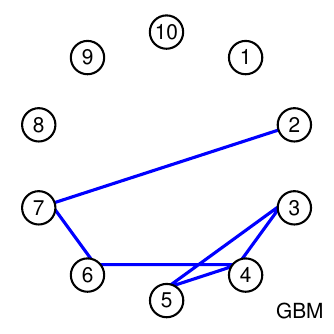}\includegraphics[width=0.25\linewidth,height=1.05in]{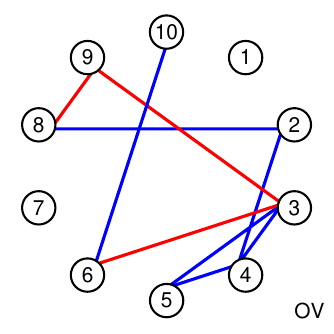} \\
\raisebox{0.1in}{\footnotesize{(V)}}\includegraphics[width=0.25\linewidth,height=1.05in]{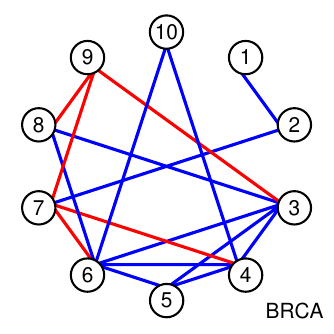}\includegraphics[width=0.25\linewidth,height=1.05in]{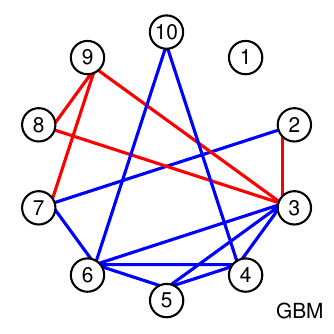}\includegraphics[width=0.25\linewidth,height=1.05in]{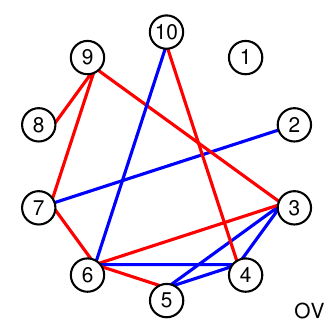} \\
\raisebox{0.1in}{\footnotesize{(C)}}\includegraphics[width=0.25\linewidth,height=1.05in]{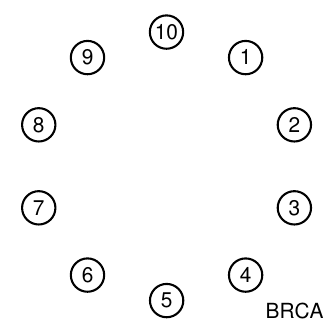}\includegraphics[width=0.25\linewidth,height=1.05in]{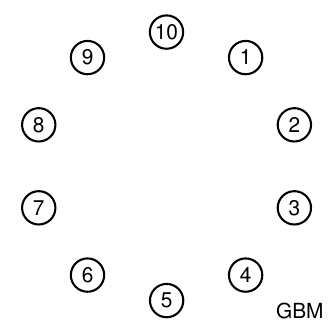}\includegraphics[width=0.25\linewidth,height=1.05in]{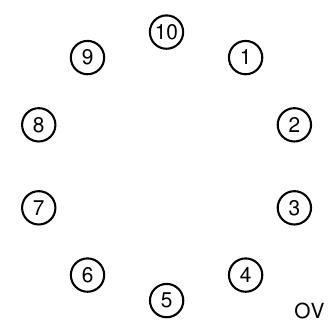} \\
\raisebox{0.1in}{\footnotesize{(G)}}\includegraphics[width=0.25\linewidth,height=1.05in]{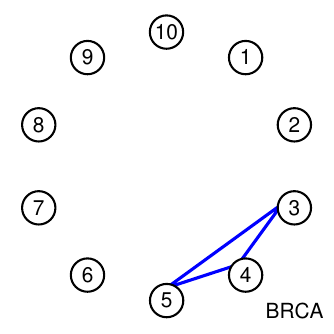}\includegraphics[width=0.25\linewidth,height=1.05in]{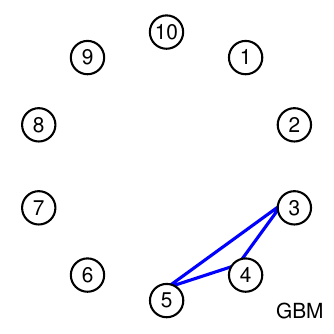}\includegraphics[width=0.25\linewidth,height=1.05in]{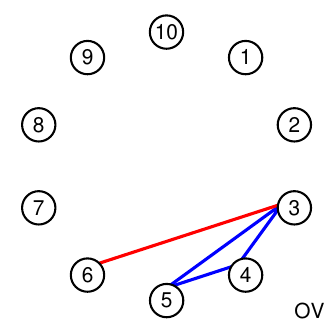} \\
\raisebox{0.1in}{\footnotesize{(MI)}}\includegraphics[width=0.25\linewidth,height=1.05in]{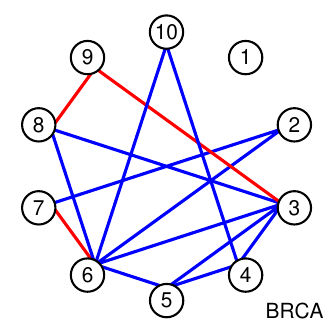}\includegraphics[width=0.25\linewidth,height=1.05in]{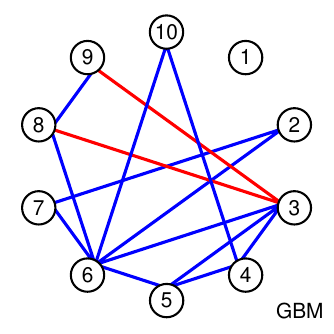}\includegraphics[width=0.25\linewidth,height=1.05in]{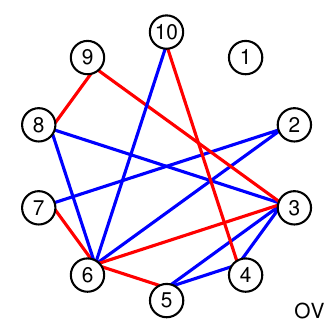} \\
\raisebox{0.1in}{\footnotesize{(M2)}}\includegraphics[width=0.25\linewidth,height=1.05in]{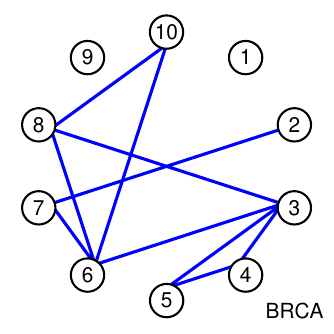}\includegraphics[width=0.25\linewidth,height=1.05in]{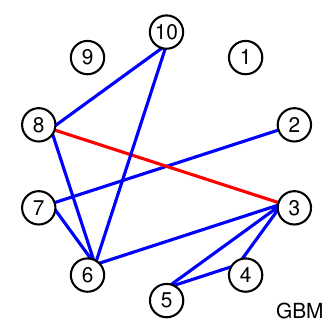}\includegraphics[width=0.25\linewidth,height=1.05in]{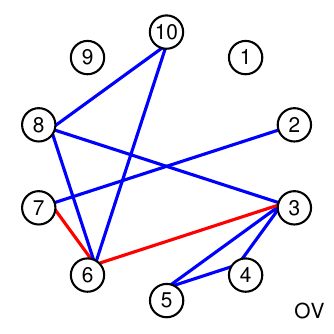} \\
\end{center}
\vspace{-0.2in}
\captionx{
Learnt structures from the \emph{cancer genome atlas} data set for three randomly selected cancer types (BRCA: breast invasive carcinoma, GBM: glioblastoma multiforme, OV: ovarian serous cystadenocarcinoma).
We show a subgraph of ten randomly selected genes (genes from 1 to 10: AGXT2, ANK2, CCNB2, CEP55, DKFZp762E1312, FAM107A, FEZ1, HLA-F, KIAA1217, PAQR8).
We include the graphical lasso (L), the methods of \citeauthor{Varoquaux10} (V), \citeauthor{Chiquet11} (C), \citeauthor{Guo11} (G), and the $\ell_{1,\infty}$ (MI) and $\ell_{1,2}$ (M2) multi-task methods.
(Positive interactions are shown in blue, negative interactions in red.
The regularization parameter was selected in a validation set.)
The methods C,MI,M2 produce a sparsity pattern that is consistent across cancer types, while the methods L,V,G fail to obtain a consistent sparsity pattern.
}
\label{fig:graphscancer}
\end{figure}

Our experimental setup is as follows.
For the \emph{1000 functional connectomes} data set, we learn one Gaussian graphical model for each of the $41$ collection sites, i.e., each collection site is a task.
For the \emph{cancer genome atlas} data set, we learn one Gaussian graphical model for each of the five cancer types, i.e., each cancer type is a task.
We used one third of the subjects for training, one third for validation and the remaining third for testing.
We use the validation set for selecting the optimal regularization parameter, and the testing set for reporting log-likelihoods.
We performed six repetitions by making each third of the subjects take turns as training, validation and testing sets.

In Figures \ref{fig:chartsconnectomes} and \ref{fig:chartscancer}, we report the negative log-likelihood on the testing set, the stability of the support union across tasks, the edge density, the sign coherence across tasks, the average edge magnitude as well as the number of tasks where edges have maximum magnitude.
The regularization parameter was selected in a validation set.
For visualization purposes, we report the negative log-likelihood in the following fashion.
After computing the negative log-likelihood, we subtracted the entropy measured on the testing set and then scaled the resulting value in $[0,1]$.
For computing our stability measure, we first computed the support union (i.e., the union of the edge sets across tasks) for each repetition.
Then we averaged those adjacency matrices across repetitions and obtained a score $s_{n_1n_2} \in [0,1]$ for each node pair $(n_1,n_2)$.
We then transformed each score by setting ${s_{n_1n_2} \leftarrow \max(s_{n_1n_2},1-s_{n_1n_2})}$ since a node pair is stable when an edge is consistently present or absent.
Finally, we averaged the transformed scores across all node pairs.
For computing the sign coherence across tasks, we first computed for each node pair $(n_1,n_2)$ the proportion of tasks with positive sign $s_{n_1n_2}^+ \in [0,1]$ as well as the proportion of tasks with negative sign $s_{n_1n_2}^- \in [0,1]$.
We then computed the score $s_{n_1n_2} \leftarrow \max{(s_{n_1n_2}^+,s_{n_1n_2}^-)}/(s_{n_1n_2}^+ + s_{n_1n_2}^-)$ since a node pair is sign-coherent when an edge is consistently positive or negative.
(Note that we exclude the case where the node pair is not an edge).
Finally, we averaged the scores across all node pairs.
The number of tasks where edges have maximum magnitude was first computed for each node pair $(n_1,n_2)$, and then averaged across all node pairs.
We can observe that the log-likelihood of both our $\ell_{1,\infty}$ and $\ell_{1,2}$ multi-task methods is better than the comparison methods.
Both our $\ell_{1,\infty}$ and $\ell_{1,2}$ multi-task methods are relatively less stable than the comparison methods, due to the fact that the latter produce different topologies per task and thus, a bigger support union.
Note that our $\ell_{1,2}$ multi-task method produces structures that are sparser than those produced by our $\ell_{1,\infty}$ counterpart.
Additionally, edges in $\ell_{1,2}$ structures are more coherent in signs that those of $\ell_{1,\infty}$ structures.
Furthermore, $\ell_{1,\infty}$ structures have approximately two thirds of the tasks have edges with maximum magnitude, while in $\ell_{1,2}$ structures only one task has an edge with maximum magnitude.
In fact, \citet{Schmidt11} noted that the $\ell_{1,\infty}$-norm regularizer encourages values to have the same magnitude.

Figures \ref{fig:graphsconnectomes} and \ref{fig:graphscancer} show a subgraph of learnt structures for three randomly selected tasks.
The regularization parameter was selected in a validation set.
We can observe that the sparsity pattern of the structures produced by our multi-task method is consistent across tasks.
The comparison methods fail to obtain a consistent sparsity pattern.
The exception is the method of \citet{Chiquet11}.
However, note in Figures \ref{fig:chartsconnectomes}(a) and \ref{fig:chartscancer}(a) that the produced models do not generalize well on testing data.

Next, we identify and discuss edges that show the important differences between the tasks under analysis.
We chose the edges are present in at least $95\%$ of the tasks, have sign coherence across tasks lower than $70\%$, and have a partial correlation\footnote{The partial correlation is computed from the precision matrix in the same way that the correlation is computed from the covariance matrix.} with a magnitude of at least $0.02$.

In the \emph{1000 functional connectomes} data set, we identified only one edge between the Right Brodmann area 23 and the Left Brodmann area 30, by using our $\ell_{1,2}$ multi-task method.
This edge is present in $39$ collection sites: $13$ sites had positive partial correlation (from $0.0086$ to $0.0339$), while $26$ sites had negative partial correlation (from $-0.0714$ to $-0.0096$).
The Brodmann areas 23 and 30 have been identified as core regions associated with the \emph{default mode network} which is active during wakeful rest \citep{Buchner08}.
Thus, the fact that we find this edge across $39$ collection sites is clinically relevant.
In comparison, our $\ell_{1,\infty}$ multi-task method found the edge in $40$ sites, the method of \citeauthor{Guo11} found the edge in $36$ sites, graphical lasso found the edge in $26$ sites, and the methods of \citeauthor{Chiquet11} and \citeauthor{Varoquaux10} did not find the edge in any collection site.
In neuroscience, it is well known that the default mode network is disrupted with diseases such as autism spectrum disorders, schizophrenia and Alzheimer's disease \citep{Buchner08}.
We speculate that the difference in the collection sites ($13$ with positive and $26$ with negative edge weights) could be explained by the presence of some subjects with an undetected disease.
For instance, Broadmann areas 23 and 30 showed increased activity for depressed unmedicated patients in a resting-state positron emission tomography (PET) study \citep{Monkul12}.

In the \emph{cancer genome atlas} data set, we identified five edges which are present in all the cancer types, by using our $\ell_{1,2}$ multi-task method.
In all of the cases, the edge weights were positive (or negative) in three cancer types and negative (or positive) in the remaining two cancer types.
The edges are between the ABCA8 and PRC1 genes (partial correlation from $-0.1342$ to $0.0509$), the ANP32E and CDC14B genes (partial correlation from $-0.0535$ to $0.0616$), the FLJ20489 and TMEM4 genes (partial correlation from $-0.056$ to $0.1203$), the FMNL3 and FMO4 genes (partial correlation from $-0.1214$ to $0.1546$) and the RPN2 and TMEPAI genes (partial correlation from $-0.049$ to $0.0708$).
Two of these genes (FMNL3 and PRC1) have been previously linked to other types of cancer not analyzed here.
The FMNL3 gene has been linked to pancreatic cancer \citep{Jones08}.
The edge between the FMNL3 and FMO4 genes has the highest value for lung squamous cell carcinoma.
Interestingly, lung cancer and pancreatic cancer are highly associated with tobacco use \citep{Engeland96}.
Thus, the aforementioned genes (FMNL3 and FMO4) could be related to genetic predisposition.
The PRC1 gene has been linked to melanoma \citep{Pleasance10}.
The edge between the ABCA8 and PRC1 genes has the highest value for breast invasive carcinoma.
Interestingly, women with breast cancer have a higher risk for melanoma and viceversa \citep{Goggins04}.

\section{Concluding Remarks} \label{sec:conclusions}

In this paper, we presented $\ell_{1,p}$ multi-task structure learning for Gaussian graphical models.
We analyzed the sufficient number of samples for correctly recovering the support union and edge signs, for both sub-Gaussian random variables and variables with finite high-order moment.
We also provide information-theoretic lower bounds.
Our result implies that for sub-Gaussian variates, the polynomial-time method of $\ell_{1,p}$ regularization achieves optimal rates, since the necessary and the sufficient number of samples are both $\O(\log K + \log N)$.
Finally, we showed that multi-task learning is statistically more efficient than single-task learning, in the sense that multi-task learning requires less samples in order to avoid failure.

For experiments, we used a block coordinate descent method which is provably convergent and yields sparse and positive definite estimates.
In synthetic experiments, we showed that the $\ell_{1,p}$ multi-task method outperforms others in recovering the topology of the ground truth model.
Additionally, the cross-validated log-likelihood of the $\ell_{1,p}$ multi-task method is higher than the comparison methods in two real-world data sets.

There are several ways of extending this research.
In this paper, we assume that the assignment of training samples to tasks is known.
A more challenging setting comes from the assumption that the assignment of training samples to tasks is unknown.
In this paper, we used the $\ell_{1,p}$-norm regularizer.
In the context of linear regression, it was shown that there are regimes in which the $\ell_1$ regularizer requires less samples than the $\ell_{1,p}$ regularizer, and viceversa \citep{Negahban11,Obozinski11}.
\emph{Dirty models} have been shown superior to the previous two approaches \citep{Jalali10}.
Thus, it is important to study the statistical efficiency of the methods such as \citep{Hara11,Hara13}.
Finally, in light of recent advances \citep{Loh13}, it is also important to study the consistency of non-convex regularizers, such as \citep{Guo11,Zhu15}.

\section*{Acknowledgments}

We thank Dardo Tomasi for the preprocessing of the \emph{1000 functional connectomes} data set.
This work was supported in part by NIDA Grants 1 R01 DA020949, 1 R01 DA023579 and NIBIB Grant 1 R01 EB007530.

\appendix

\section{Detailed Proofs} \label{sec:detailedproofs}

In this section, we state the proofs of all the theorems in our manuscript.

\subsection{Proof of Theorem \ref{thm:consistencysubgaussian}}

Here, we provide the detailed proof of Theorem \ref{thm:consistencysubgaussian}.
First, we derive some intermediate lemmas needed for the final proof.
\begin{lemma} \label{lem:strictdualnormp}
Let $p\in (1,\infty)$.
Let the $\ell_\dual{p}$-norm be the dual of the $\ell_p$-norm, i.e., $\frac{1}{p} + \frac{1}{\dual{p}} = 1$.
Let $\a,\b \in \R^K$ be two vectors such that $\|\a\|_p=\t{\a}\b$ and $\|\b\|_\dual{p} \leq 1$.
We have that, $\|\b\|_\dual{p} < 1 \Rightarrow \a = \zero$.
Additionally, $\|\b\|_\dual{p} = 1$ and $\a \neq \zero \Rightarrow (\forall k){\rm\ }b_k=\sgn(a_k)(|a_k|/\|\a\|_p)^{p/\dual{p}}$, and therefore $(\forall k){\rm\ }\sgn(a_k)=\sgn(b_k)$.
\end{lemma}
\begin{proof}
From the identity for dual norms, we have $\|\a\|_p=\max_{\|\c\|_\dual{p}\leq 1}\t{\a}\c$.
Let $\b$ be the optimal solution for the previous maximization.
That is, $\|\a\|_p=\max_{\|\c\|_\dual{p}\leq 1}\t{\a}\c=\t{\a}\b$ and $\|\b\|_\dual{p} \leq 1$.
Thus, we obtained our assumptions.

First, we prove the statement $\|\b\|_\dual{p} < 1 \Rightarrow \a = \zero$ by contradiction.
Assume $\|\b\|_\dual{p} < 1$ and $\a \neq \zero$.
The objective function is linear and non-constant since $\a \neq \zero$, and the constraint set is convex.
Thus, there is a unique optimal solution that must be on the boundary of the constraint set, i.e., $\|\b\|_\dual{p} = 1$.
Therefore, we have a contradiction.

Next, we prove the statement $\|\b\|_\dual{p} = 1$ and $\a \neq \zero \Rightarrow (\forall k){\rm\ }b_k=\sgn(a_k)(|a_k|/\|\a\|_p)^{p/\dual{p}}$.
Again, the objective function is linear and non-constant since $\a \neq \zero$, and the constraint set is convex.
Thus, there is a unique optimal solution that must be on the boundary of the constraint set, i.e., $\|\b\|_\dual{p} = 1$.
It is trivial to verify that for a vector $\b$ that fulfills $(\forall k){\rm\ }b_k=\sgn(a_k)(|a_k|/\|\a\|_p)^{p/\dual{p}}$, we have that $\|\b\|_\dual{p} = 1$ and $\|\a\|_p=\t{\a}\b$.
Therefore, the prescribed $\b$ is indeed the unique optimal solution.
Finally, for the prescribed $\b$, it is trivial to verify that $(\forall k){\rm\ }\sgn(a_k)=\sgn(b_k)$.
\qedhere
\end{proof}

\begin{lemma} \label{lem:strictdualnorminf}
Let $\a,\b \in \R^K$ be two vectors such that $\|\a\|_\infty=\t{\a}\b$ and $\|\b\|_1 \leq 1$.
We have that, $\|\b\|_1 < 1 \Rightarrow \a = \zero$.
Let $\K = \argmax_k{|a_k|}$ be the set of indices that attain the maximum value.
Additionally, $\|\b\|_1 = 1$ and $\a \neq \zero \Rightarrow (\forall k){\rm\ }b_k=\sgn(a_k)\iverson{k \in \K}t_k$ for any set of weights $\vt \geq 0$ such that $\sum_k{\iverson{k \in \K}t_k} = 1$, and therefore $(\forall k){\rm\ }a_kb_k \geq 0$.
\end{lemma}
\begin{proof}
From the identity for dual norms, we have $\|\a\|_\infty=\max_{\|\c\|_1\leq 1}\t{\a}\c$.
Let $\b$ be the optimal solution for the previous maximization.
That is, $\|\a\|_\infty=\max_{\|\c\|_1\leq 1}\t{\a}\c=\t{\a}\b$ and $\|\b\|_1\leq 1$.
Thus, we obtained our assumptions.

We prove the first statement $\|\b\|_1 < 1 \Rightarrow \a = \zero$ by contradiction.
Assume $\|\b\|_1 < 1$ and $\a \neq \zero$.
The objective function is linear and non-constant since $\a \neq \zero$, and the constraint set is convex.
Thus, there is a unique optimal solution that must be on the boundary of the constraint set, i.e., $\|\b\|_1 = 1$.
Therefore, we have a contradiction.

The second statement can be found in Lemma 1 in \citep{Negahban11}.
Note that by construction, the prescribed $\b$ fulfills:
\begin{align*}
\|\b\|_1 & = \textstyle{ \sum_k{|\sgn(a_k)\iverson{k \in \K}t_k|} } \\
 & = \textstyle{ \sum_k{\iverson{k \in \K}t_k} } \\
 & = 1 \; .
\end{align*}
Finally, note that:
\begin{align*}
(\forall k){\rm\ }a_kb_k & = a_k\sgn(a_k)\iverson{k \in \K}t_k \\
 & = |a_k|\iverson{k \in \K}t_k \\
 & \geq 0 \; ,
\end{align*}
\noindent which proves our claim.
\qedhere
\end{proof}

\begin{lemma} \label{lem:signrecovery}
Let $a,b \in \R$ and $\eps>0$.
Assume that $|a-b| \leq \eps$ and either $b=0$ or $|b|>2\eps$.
If $b=0$ then $|a| \leq \eps$.
If $b>0$ then $a>0$.
If $b<0$ then $a<0$.
\end{lemma}
\begin{proof}
If $b=0$, since $|a-b| \leq \eps$ we have $|a| \leq \eps$.

\noindent If $b>2\eps$, since $|a-b| \leq \eps$ we have $a \geq b - \eps > 2\eps - \eps = \eps > 0$.

\noindent If $b<-2\eps$, since $|a-b| \leq \eps$ we have $a \leq b + \eps < -2\eps + \eps = -\eps < 0$.
\qedhere
\end{proof}

\begin{lemma} \label{lem:consistency}
Let the $\ell_\dual{p}$-norm be the dual of the $\ell_p$-norm, i.e., $\frac{1}{p} + \frac{1}{\dual{p}} = 1$.
Let Assumption \ref{asm:incoherence} hold for the $\ell_\dual{p}$-norm and $\alpha_\dual{p} \in (0,1]$.
Assume that $(\forall k){\rm\ }\|\BSigmah\si{k} - \BSigmat\si{k}\|_\infty \leq \eps$.
Set the regularization parameter in Eq.~\eqref{eq:multitaskggm} to $\rho = C_2 \eps K^{1/\dual{p}}/\alpha_\dual{p}$ for some $C_2 \geq 8$.
If the following holds:
\begin{align} \label{eq:samplecomplexity}
1 \geq d_{\BOmegat} {\rm\ } \left( 6\eps \max{(C_{\BSigmat} C_{\BGammat}, C_{\BSigmat}^3 C_{\BGammat}^2)} \right) (1+\rho/\eps)^2 \; ,
\end{align}
\noindent then we have:
\begin{align*}
(\forall k){\rm\ }\|\BOmegat\si{k} - \BOmegah\si{k}\|_\infty \leq 2 \eps C_{\BGammat} (1+\rho/\eps) \; .
\end{align*}
Furthermore, the support union and the edge signs of the empirical minimizer $\BOmegah \equiv \{\omegah_{n_1n_2}\si{k}\}$ are equal to those of the true model $\BOmegat \equiv \{\omegat_{n_1n_2}\si{k}\}$.
That is:
\begin{align*}
\Sup_{\BOmegah} = \Sup_{\BOmegat}\text{\ \ \ and\ \ \ }(\forall k,(n_1,n_2) \in \Sup_{\BOmegat}){\rm\ } & \textstyle{ \left( \omegat_{n_1n_2}\si{k} = 0 \text{\ \ and\ \ } |\omegah_{n_1n_2}\si{k}| \leq 2 \eps C_{\BGammat} (1+\rho/\eps) \right) } \\
 & \text{or\ \ } (\omegat_{n_1n_2}\si{k} > 0 \text{\ \ and\ \ } \omegah_{n_1n_2}\si{k} > 0) \\
 & \text{or\ \ } (\omegat_{n_1n_2}\si{k} < 0 \text{\ \ and\ \ } \omegah_{n_1n_2}\si{k} < 0) \; ,
\end{align*}
\noindent provided that:
\begin{align*}
\textstyle{ (\forall k,(n_1,n_2) \in \Sup_{\BOmegat}){\rm\ }\omegat_{n_1n_2}\si{k} = 0 \text{\ \ \ or\ \ \ } |\omegat_{n_1n_2}\si{k}| > 4 \eps C_{\BGammat} (1+\rho/\eps) } \; .
\end{align*}
\end{lemma}
\begin{proof}
Recall that the support union was formally defined in Eq.~\eqref{eq:supportunion}.
For clarity, let the support union of the true model be $\Sup \equiv \Sup_{\BOmegat}$.
Let $\Supc$ be the complement of $\Sup$.
Additionally, we assume that every task contains the same number of samples, i.e., $(\forall k){\rm\ }T\si{k} = 1$.
Our proof makes use of some of the Karush-Kuhn-Tucker conditions, such as stationarity, complementary slackness, and dual feasibility.
Given the length of the proof, we split it into four parts.

\vspace{1em} \noindent \emph{Part I.}
First, we show that the empirical minimizer $\BOmegah$ correctly excludes all non-edges if it fulfills a \emph{strict dual feasibility} condition.
Note that the minimizer $\BOmegah\si{k}$ of the problem in Eq.~\eqref{eq:multitaskggm} fulfills the \emph{stationarity} condition.
Let $\Zh\si{k} \in \R^{N \times N}$, we have:
\begin{align} \label{eq:pdstationarity}
(\forall k){\rm\ }T\si{k}(\inv{\textstyle{\BOmegah\si{k}}} - \BSigmah\si{k}) - \rho \Zh\si{k} = \zero \; ,
\end{align}
\noindent where the $\ell_{\infty,\dual{p}}$-norm fulfills the \emph{dual feasibility} condition:
\begin{align} \label{eq:pddualfeasibility}
\|\Zh\|_{\infty,\dual{p}} \equiv \max_{n_1 \neq n_2}\|(\zh_{n_1n_2}\si{1},\dots,\zh_{n_1n_2}\si{K})\|_\dual{p} \leq 1 \; .
\end{align}
For shortness, let $\vzh_{n_1n_2} = (\zh_{n_1n_2}\si{1},\dots,_{n_1n_2}\si{K})$ and $\vomegah_{n_1n_2} = (\omegah_{n_1n_2}\si{1},\dots,\omegah_{n_1n_2}\si{K})$.
By the \emph{complementary slackness} condition, we have $\|\BOmegah\|_{1,p} = \sum_k{\dotprod{\Zh\si{k}}{\BOmegah\si{k}}}$ and thus:
\begin{align} \label{eq:pdinnerprod}
 \sum_{n_1 \neq n_2}{\|\vomegah_{n_1n_2}\|_p} & = \sum_{k,n_1 \neq n_2}{\zh_{n_1n_2}\si{k}\omegah_{n_1n_2}\si{k}} \nonumber \\
 \Leftarrow {\rm\ } (\forall n_1 \neq n_2){\rm\ }{\|\vomegah_{n_1n_2}\|_p} & = \sum_k{\zh_{n_1n_2}\si{k}\omegah_{n_1n_2}\si{k}} \; .
\end{align}
By Eq.~\eqref{eq:pddualfeasibility}, Eq.~\eqref{eq:pdinnerprod} and Lemmas \ref{lem:strictdualnormp} and \ref{lem:strictdualnorminf}, for $p>1$ we have:
\begin{align} \label{eq:compslacknesszero}
(\forall n_1 \neq n_2){\rm\ }\|\vzh_{n_1n_2}\|_\dual{p} < 1 \Rightarrow \vomegah_{n_1n_2} = \zero \; .
\end{align}
By Eq.~\eqref{eq:pddualfeasibility}, Eq.~\eqref{eq:pdinnerprod} and Lemma \ref{lem:strictdualnormp} for $p\in (1,\infty)$, we have:
\begin{align} \label{eq:compslacknessnozerop}
(\forall n_1 \neq n_2){\rm\ }\|\vzh_{n_1n_2}\|_\dual{p} \hns=\hns 1 \text{\ ,\ } \vomegah_{n_1n_2} \neq \zero & \Rightarrow (\forall k){\rm\ }\zh_{n_1n_2}\si{k}=\sgn(\omegah_{n_1n_2}\si{k})(|\omegah_{n_1n_2}\si{k}|/\|\vomegah_{n_1n_2}\|_p)^{p/\dual{p}} \nonumber \\
 & \Rightarrow (\forall k){\rm\ }\sgn(\zh_{n_1n_2}\si{k})=\sgn(\omegah_{n_1n_2}\si{k}) \; .
\end{align}
By Eq.~\eqref{eq:pddualfeasibility}, Eq.~\eqref{eq:pdinnerprod} and Lemma \ref{lem:strictdualnorminf} for $p=\infty$, let $\K = \argmax_k{|\omegah_{n_1n_2}\si{k}|}$ be the set of indices that attain the maximum value, and $\vt \geq 0$ be any set of weights such that $\sum_k{\iverson{k \in \K}t_k} = 1$.
We have:
\begin{align} \label{eq:compslacknessnozeroinf}
(\forall n_1 \neq n_2){\rm\ }\|\vzh_{n_1n_2}\|_\dual{p} \hns=\hns 1 \text{\ ,\ } \vomegah_{n_1n_2} \neq \zero & \Rightarrow \nonumber (\forall k){\rm\ }\zh_{n_1n_2}\si{k}=\sgn(\omegah_{n_1n_2}\si{k})\iverson{k \in \K}t_k \\
 & \Rightarrow (\forall k){\rm\ }\zh_{n_1n_2}\si{k}\omegah_{n_1n_2}\si{k} \geq 0 \; .
\end{align}
In particular, consider all edges $(n_1,n_2) \notin \Sup$.
From Eq.~\eqref{eq:compslacknesszero}, it follows that if the $\ell_{\infty,\dual{p}}$-norm fulfills a \emph{strict dual feasibility} condition:
\begin{align} \label{eq:strictdualfeasibility}
\|\Zh_{\Supc}\|_{\infty,\dual{p}} \equiv \max_{(n_1,n_2) \in \Supc}\|(\zh_{n_1n_2}\si{1},\dots,\zh_{n_1n_2}\si{K})\|_\dual{p} < 1 \; ,
\end{align}
\noindent then $\BOmegah_{\Supc} = \zero$, which implies that \emph{$\BOmegah$ correctly excludes all non-edges}, i.e., $\Sup \subseteq \Sup_{\BOmegah}$.

\vspace{1em} \noindent \emph{Part II.}
In what follows, we apply the \emph{primal-dual witness} method.
Our goal is to show that the \emph{strict dual feasibility} condition is fulfilled.

To be clear, the following procedure is not a practical algorithm for solving Eq.~\eqref{eq:multitaskggm}.
It is a proof technique for certifying the behavior of the empirical minimizer.
We use $\Supc$ in order to construct the primal-dual witness solution $(\BOmegac,\Zc)$ as follows:
\begin{enumerate}[a)]
\item \label{item:pdrestricted} We determine $\BOmegac$ in a manner that guarantees that $(\forall k){\rm\ }\BOmegac\si{k}\succ \zero$ and $\BOmegac_{\Supc}=\zero$, by solving the restricted problem:
\begin{align} \label{eq:pdrestricted}
\BOmegac = \argmax_{(\forall k){\rm\ }\BOmega\si{k}\succ \zero {\rm\ ,\ } \BOmega_{\Supc}=\zero}\left(\sum_k{T\si{k}\ell_{\BSigmah\si{k}}(\BOmega\si{k})} - \rho\|\BOmega\|_{1,p} \right) \; .
\end{align}
\item \label{item:pdcompslackness} We choose $\Zc_{\Sup}$ in order to fulfill the \emph{complementary slackness} condition in Eq.~\eqref{eq:compslacknessnozerop} for $p\in (1,\infty)$ or Eq.~\eqref{eq:compslacknessnozeroinf} for $p=\infty$.
\item \label{item:pdstationaritynonsupp} We set $\Zc_{\Supc}$ in order to guarantee that $(\BOmegac,\Zc)$ fulfills the \emph{stationarity} condition in Eq.~\eqref{eq:pdstationarity}.
That is:
\begin{align} \label{eq:pdstationaritynonsupp}
(\forall k){\rm\ }\Zc_{\Supc}\si{k} = \frac{T\si{k}}{\rho} ({[\inv{\textstyle{\BOmegac\si{k}}}]}_{\Supc} - \BSigmah_{\Supc}\si{k}) \; .
\end{align}
\item \label{item:pdstrictdualfeasibility} We verify the \emph{strict dual feasibility} condition.
That is, by using Eq.~\eqref{eq:strictdualfeasibility}, we verify that:
\begin{align*}
\|\Zc_{\Supc}\|_{\infty,\dual{p}} < 1 \; .
\end{align*}
\end{enumerate}
If the primal-dual witness construction succeeds, then it acts as a witness to the fact that the solution $\BOmegac$ to the restricted problem in Eq.~\eqref{eq:pdrestricted} is equal to the solution $\BOmegah$ to the original (unrestricted) problem in Eq.~\eqref{eq:multitaskggm}.
Note that all steps hold by construction, with the exception of step \eqref{item:pdstrictdualfeasibility}.
Thus, in what follows, we concentrate on proving that step \eqref{item:pdstrictdualfeasibility} holds under the assumptions made in this lemma.

First, note that by steps \eqref{item:pdrestricted} and \eqref{item:pdstationaritynonsupp}, the primal-dual witness solution $(\BOmegac,\Zc)$ fulfills the stationarity condition in Eq.~\eqref{eq:pdstationarity}.
That is:
\begin{align*}
(\forall k){\rm\ }\inv{\textstyle{\BOmegac\si{k}}} - \BSigmah\si{k} - \frac{\rho}{T\si{k}} \Zc\si{k} = \zero \; .
\end{align*}
Let $\BDelta\si{k} \equiv \BOmegat\si{k} - \BOmegac\si{k}$, $\A\si{k} \equiv \BSigmah\si{k}-\BSigmat\si{k}$ and $\B\si{k} \equiv \inv{\textstyle{\BOmegac\si{k}}} - \inv{\BOmegat\si{k}} - \inv{\BOmegat\si{k}}\hns\BDelta\si{k}\inv{\BOmegat\si{k}}$.
By Eq.~\eqref{eq:truesolution} we know that $(\forall k){\rm\ }\BSigmat\si{k} = \inv{\BOmegat\si{k}}$ and thus:
\begin{align} \label{eq:pdstationarityexp}
(\forall k){\rm\ }\inv{\BOmegat\si{k}}\hns\BDelta\si{k}\inv{\BOmegat\si{k}} - (\A\si{k} - \B\si{k}) - \frac{\rho}{T\si{k}} \Zc\si{k} = \zero \; .
\end{align}
Let $\vec(\C)$ be the vectorized form of matrix $\C$.
By property of the Kronecker product and the definition in Eq.~\eqref{eq:hessian}, we have:
\begin{align*}
(\forall k){\rm\ }\vec(\inv{\BOmegat\si{k}}\hns\BDelta\si{k}\inv{\BOmegat\si{k}}) & = (\kronprod{\inv{\BOmegat\si{k}}}{\inv{\BOmegat\si{k}}}) \vec(\BDelta\si{k}) \\
 & = \BGammat\si{k} \vec(\BDelta\si{k}) \; .
\end{align*}
Note that by step \eqref{item:pdrestricted} we have $\BOmegac_{\Supc} = \zero \Rightarrow \BDelta_{\Supc} = \zero$.
We then rewrite Eq.~\eqref{eq:pdstationarityexp} as follows:
\begin{align} \refstepcounter{equation}
(\forall k){\rm\ }\BGammat_{\Sup\Sup}\si{k}\BDelta_{\Sup}\si{k} - (\A_{\Sup}\si{k}-\B_{\Sup}\si{k}) - \frac{\rho}{T\si{k}} \Zc_{\Sup}\si{k} = \zero \; , \tag{\theequation.a}\label{eq:pdstationaritysup} \\
(\forall k){\rm\ }\BGammat_{\Supc\Sup}\si{k}\BDelta_{\Sup}\si{k} - (\A_{\Supc}\si{k}-\B_{\Supc}\si{k}) - \frac{\rho}{T\si{k}} \Zc_{\Supc}\si{k} = \zero \; . \tag{\theequation.b}\label{eq:pdstationaritysupc}
\end{align}
Since $\BGammat_{\Sup\Sup}\si{k}$ is invertible, from Eq.~\eqref{eq:pdstationaritysup} we have:
\begin{align*}
(\forall k){\rm\ }\BDelta_{\Sup}\si{k} = \inv{\BGammat_{\Sup\Sup}\si{k}} (\A_{\Sup}\si{k}-\B_{\Sup}\si{k}) + \frac{\rho}{T\si{k}} \inv{\BGammat_{\Sup\Sup}\si{k}} \Zc_{\Sup}\si{k} \; .
\end{align*}
By substituting $\BDelta_{\Sup}\si{k}$ in Eq.~\eqref{eq:pdstationaritysupc} and by defining $\BPsit\si{k} \equiv \BGammat_{\Supc\Sup}\si{k}\inv{\BGammat_{\Sup\Sup}\si{k}}$ as in Assumption \ref{asm:incoherence}, we obtain:
\begin{align} \label{eq:pdindep}
(\forall k){\rm\ }\Zc_{\Supc}\si{k} & = \frac{T\si{k}}{\rho} \BGammat_{\Supc\Sup}\si{k}\BDelta_{\Sup}\si{k} - \frac{T\si{k}}{\rho} (\A_{\Supc}\si{k}-\B_{\Supc}\si{k}) \nonumber \\
 & = \frac{T\si{k}}{\rho} \BPsit\si{k} (\A_{\Sup}\si{k}-\B_{\Sup}\si{k}) + \BPsit\si{k} \Zc_{\Sup}\si{k} - \frac{T\si{k}}{\rho} (\A_{\Supc}\si{k}-\B_{\Supc}\si{k}) \; .
\end{align}
Recall that $\BPsit\si{k} \in \R^{|\Supc| \times |\Sup|}$.
Given a set of $K$ vectors $(\forall k){\rm\ }\c\si{k} \in \R^{|\Sup|}$, define for shortness $\D = \myprod{\BPsit}{\C} \in \R^{|\Supc| \times K}$ by $K$ independent products $(\forall k){\rm\ }\d\si{k} = \BPsit\si{k} \c\si{k}$.
Furthermore, define the $\ell_{\infty,1,\dual{p}}$-norm as follows:
\begin{align} \label{eq:pdtensornorm}
\|\BPsit\|_{\infty,1,\dual{p}} \equiv \max_i \|\textstyle{ (\sum_j{|\psit_{ij}\si{1}|},\dots,\sum_j{|\psit_{ij}\si{K}|}) }\|_\dual{p} \leq 1-\alpha_\dual{p} \; ,
\end{align}
\noindent where the last inequality follows from Assumption \ref{asm:incoherence}.
By using the $\ell_{\infty,\dual{p}}$-norm defined in Eq.~\eqref{eq:strictdualfeasibility} and the $\ell_{\infty,1,\dual{p}}$-norm defined in Eq.~\eqref{eq:pdtensornorm}, it is easy to verify that:
\begin{align} \label{eq:pdnormineq}
\|\D\|_{\infty,\dual{p}} & = \max_i {\|(d_i\si{1},\dots,d_i\si{K})\|_\dual{p}} \nonumber \\
 & = \max_i \|\textstyle{ (\sum_j{\psit_{ij}\si{1}c_j\si{1}},\dots,\sum_j{\psit_{ij}\si{K}c_j\si{K}}) }\|_\dual{p} \nonumber \\
 & \leq \max_i \|\textstyle{ (\sum_j{|\psit_{ij}\si{1}|},\dots,\sum_j{|\psit_{ij}\si{K}|}) }\|_\dual{p} {\rm\ \ } \max_{jk}{|c_j\si{k}|} \nonumber \\
 & = \|\BPsit\|_{\infty,1,\dual{p}} \|\C\|_\infty \; .
\end{align}
Recall that we assumed that $(\forall k){\rm\ }T\si{k} = 1$.
Thus, the Eq.~\eqref{eq:pdindep} becomes:
\begin{align*}
\Zc_{\Supc} = \frac{1}{\rho} \myprod{\BPsit}{(\A_{\Sup}-\B_{\Sup})} + \myprod{\BPsit}{\Zc_{\Sup}} - \frac{1}{\rho} (\A_{\Supc}-\B_{\Supc}) \; .
\end{align*}
By assumption, we have that $(\forall k){\rm\ }\|\A\si{k}\|_\infty = \|\BSigmah\si{k} - \BSigmat\si{k}\|_\infty \leq \eps$.
For the moment, we will assume\footnote{In the next part of the proof, we show that this indeed holds.} that $(\forall k){\rm\ }\|\B\si{k}\|_\infty \leq \eps$.
Since $\rho \geq 8 \eps K^{1/\dual{p}}/\alpha_\dual{p}$, we have $\eps \leq \frac{\rho\alpha_\dual{p}}{8K^{1/\dual{p}}}$.
By Eq.~\eqref{eq:pddualfeasibility}, Eq.~\eqref{eq:pdtensornorm} and Eq.~\eqref{eq:pdnormineq}, we have:
\begin{align*}
\|\Zc_{\Supc}\|_{\infty,\dual{p}} & = \frac{1}{\rho} \|\myprod{\BPsit}{(\A_{\Sup}-\B_{\Sup})}\|_{\infty,\dual{p}} + \|\myprod{\BPsit}{\Zc_{\Sup}}\|_{\infty,\dual{p}} + \frac{1}{\rho} \|\A_{\Supc}-\B_{\Supc}\|_{\infty,\dual{p}} \\
 & \leq \frac{1}{\rho} \|\BPsit\|_{\infty,1,\dual{p}} \|\A_{\Sup}-\B_{\Sup}\|_\infty + \|\BPsit\|_{\infty,1,\dual{p}} \|\Zc_{\Sup}\|_\infty + \frac{1}{\rho} \|\A_{\Supc}-\B_{\Supc}\|_{\infty,\dual{p}} \\
 & \leq \frac{1-\alpha_\dual{p}}{\rho} \|\A_{\Sup}-\B_{\Sup}\|_\infty + (1-\alpha_\dual{p}) \|\Zc\|_\infty + \frac{1}{\rho} \|\A_{\Supc}-\B_{\Supc}\|_{\infty,\dual{p}} \\
 & \leq \frac{1-\alpha_\dual{p}}{\rho} (\|\A_{\Sup}\|_\infty+\|\B_{\Sup}\|_\infty) + (1-\alpha_\dual{p}) \|\Zc\|_{\infty,\dual{p}} + \frac{K^{1/\dual{p}}}{\rho} (\|\A_{\Supc}\|_\infty+\|\B_{\Supc}\|_\infty) \\
 & \leq \frac{1-\alpha_\dual{p}}{\rho} (\|\A\|_\infty+\|\B\|_\infty) + (1-\alpha_\dual{p}) \|\Zc\|_{\infty,\dual{p}} + \frac{K^{1/\dual{p}}}{\rho} (\|\A\|_\infty+\|\B\|_\infty) \\
 & \leq \frac{2 \eps (1-\alpha_\dual{p})}{\rho} + (1-\alpha_\dual{p}) + \frac{2 \eps K^{1/\dual{p}}}{\rho} \\
 & \leq \frac{\alpha_\dual{p}(1-\alpha_\dual{p})}{4 K^{1/\dual{p}}} + (1-\alpha_\dual{p}) + \frac{\alpha_\dual{p}}{4} \\
 & \leq \frac{\alpha_\dual{p}(1-\alpha_\dual{p})}{4} + (1-\alpha_\dual{p}) + \frac{\alpha_\dual{p}}{4} \\
 & = \textstyle{ 1 - \frac{1}{2}\alpha_\dual{p} - \frac{1}{4}\alpha_\dual{p}^2 } \\
 & < 1 \text{\ \ for\ \ } \alpha_\dual{p} \in (0,1] \; .
\end{align*}

\vspace{1em} \noindent \emph{Part III.}
Next, we show that the assumption $(\forall k){\rm\ }\|\B\si{k}\|_\infty \leq \eps$ in the previous part of the proof, indeed holds.
Recall that $\BDelta\si{k} \equiv \BOmegat\si{k} - \BOmegac\si{k}$.
By Neumann series, we know that $\inv{(\I-\C)} = \sum_{i=0}^\infty{\C^i}$.
Let $\J\si{k} \equiv \sum_{i=0}^\infty{(\inv{\BOmegat\si{k}}\hns\BDelta\si{k})^i}$, we have:
\begin{align} \label{eq:pdinvomegac}
(\forall k){\rm\ }\inv{\textstyle{\BOmegac\si{k}}} & = \inv{(\BOmegat\si{k} - \BDelta\si{k})} \nonumber \\
 & = \inv{(\BOmegat\si{k}(\I - \inv{\BOmegat\si{k}}\hns\BDelta\si{k}))} \nonumber \\
 & = \inv{(\I - \inv{\BOmegat\si{k}}\hns\BDelta\si{k})} \inv{\BOmegat\si{k}} \nonumber \\
 & = \J\si{k}\inv{\BOmegat\si{k}} \nonumber \\
 & = (\I + \inv{\BOmegat\si{k}}\hns\BDelta\si{k} + \J\si{k} (\inv{\BOmegat\si{k}}\hns\BDelta\si{k})^2) \inv{\BOmegat\si{k}} \nonumber \\
 & = \inv{\BOmegat\si{k}} + \inv{\BOmegat\si{k}}\hns\BDelta\si{k}\inv{\BOmegat\si{k}} + \J\si{k} (\inv{\BOmegat\si{k}}\hns\BDelta\si{k})^2 \inv{\BOmegat\si{k}} \; .
\end{align}
By replacing the above into the definition of $\B\si{k}$, we have:
\begin{align} \label{eq:pdmatrixb}
(\forall k){\rm\ }\B\si{k} & = \inv{\textstyle{\BOmegac\si{k}}} - \inv{\BOmegat\si{k}} - \inv{\BOmegat\si{k}}\hns\BDelta\si{k}\inv{\BOmegat\si{k}} \nonumber \\
 & = \J\si{k} (\inv{\BOmegat\si{k}}\hns\BDelta\si{k})^2 \inv{\BOmegat\si{k}} \; .
\end{align}
Define the norm $\|\C\|_* = \max_{n_1}{\sum_{n_2}|c_{n_1n_2}|}$.
Note that $(\forall k){\rm\ }\|\BSigmat\si{k}\|_* \leq C_{\BSigmat}$ where $C_{\BSigmat}$ is defined in Eq.~\eqref{eq:sigmaconst}.
Recall that by step \eqref{item:pdrestricted} we have $\BOmegac_{\Supc} = \zero \Rightarrow \BDelta_{\Supc} = \zero$.
Furthermore, by Eq.~\eqref{eq:degree}, each row/column in $\BDelta\si{k}$ contains at most $d_{\BOmegat}$ nonzero entries, thus ${(\forall k){\rm\ }\|\BDelta\si{k}\|_* \leq d_{\BOmegat}\|\BDelta\si{k}\|_\infty}$.
For the moment, we will assume that:
\begin{align} \label{eq:pdradius}
\|\BDelta\si{k}\|_\infty \leq 1/(3d_{\BOmegat}C_{\BSigmat}) \; .
\end{align}
Given the above and by sub-multiplicativity of the $\ell_*$-norm, we have:
\begin{align*}
(\forall k){\rm\ }\|\inv{\BOmegat\si{k}}\hns\BDelta\si{k}\|_* & = \|\BSigmat\si{k}\BDelta\si{k}\|_* \\
 & \leq \|\BSigmat\si{k}\|_* \|\BDelta\si{k}\|_* \\
 & \leq d_{\BOmegat} C_{\BSigmat}\|\BDelta\si{k}\|_\infty \\
 & \leq 1/3 \; ,
\end{align*}
\noindent and furthermore:
\begin{align*}
(\forall k){\rm\ }\|\J\si{k}\|_* & = \textstyle{ \|\sum_{i=0}^\infty{(\inv{\BOmegat\si{k}}\hns\BDelta\si{k})^i}\|_* } \\
 & \leq \textstyle{ \sum_{i=0}^\infty{\|\inv{\BOmegat\si{k}}\hns\BDelta\si{k}\|_*^i} } \\
 & \leq \textstyle{ \sum_{i=0}^\infty{(1/3)^i} } \\
 & \leq 3/2 \; .
\end{align*}
It is well known that $\|\C\D\|_\infty \leq \|\C\|_*\|\D\|_\infty$ as well as $\|\C\D\|_\infty \leq \|\C\|_\infty\|\t{\D}\|_*$.
By the above, sub-multiplicativity of the $\ell_*$-norm and given that $(\forall k){\rm\ }\|\BSigmat\si{k}\|_* \leq C_{\BSigmat}$ and ${(\forall k){\rm\ }\|\BDelta\si{k}\|_* \leq d_{\BOmegat}\|\BDelta\si{k}\|_\infty}$ as well as $\inv{\BOmegat\si{k}} = \BSigmat\si{k}$, we have from Eq.~\eqref{eq:pdmatrixb}:
\begin{align} \label{eq:pdlinfmatrixb}
(\forall k){\rm\ }\|\B\si{k}\|_\infty & = \|\J\si{k} (\inv{\BOmegat\si{k}}\hns\BDelta\si{k})^2 \inv{\BOmegat\si{k}}\|_\infty \nonumber \\
 & = \|\J\si{k}\BSigmat\si{k}\BDelta\si{k}\BSigmat\si{k}\BDelta\si{k}\BSigmat\si{k}\|_\infty \nonumber \\
 & \leq \|\J\si{k}\BSigmat\si{k}\BDelta\si{k}\BSigmat\si{k}\BDelta\si{k}\|_\infty \|\BSigmat\si{k}\|_* \nonumber \\
 & \leq \|\J\si{k}\BSigmat\si{k}\BDelta\si{k}\BSigmat\si{k}\|_\infty \|\BDelta\si{k}\|_* \|\BSigmat\si{k}\|_* \nonumber \\
 & \leq \|\J\si{k}\BSigmat\si{k}\BDelta\si{k}\|_\infty \|\BDelta\si{k}\|_* \|\BSigmat\si{k}\|_*^2 \nonumber \\
 & \leq \|\J\si{k}\BSigmat\si{k}\|_* \|\BDelta\si{k}\|_\infty \|\BDelta\si{k}\|_* \|\BSigmat\si{k}\|_*^2 \nonumber \\
 & \leq \|\J\si{k}\|_* \|\BDelta\si{k}\|_\infty \|\BDelta\si{k}\|_* \|\BSigmat\si{k}\|_*^3 \nonumber \\
 & \leq \textstyle{ \frac{3}{2} d_{\BOmegat}C_{\BSigmat}^3 \|\BDelta\si{k}\|_\infty^2 } \; .
\end{align}
Thus, Eq.~\eqref{eq:pdlinfmatrixb} provides an upper bound of $\|\B\si{k}\|_\infty$ for all $k$.
Recall that by step \eqref{item:pdrestricted} we have $\BOmegac_{\Supc} = \zero \Rightarrow \BDelta_{\Supc} = \zero$.
Define the function $\F : \R^{|\Sup| \times K} \to \R^{|\Sup| \times K}$ as follows $\F(\BDelta_{\Sup}) = (\F\si{1}(\BDelta_{\Sup}\si{1}),\dots,\F\si{K}(\BDelta_{\Sup}\si{K}))$, where:
\begin{align} \label{eq:pdbrouwerfunction}
(\forall k){\rm\ }\F\si{k}(\BDelta_{\Sup}\si{k}) = -\inv{\BGammat_{\Sup\Sup}\si{k}} \left( {[\inv{\textstyle{\BOmegac\si{k}}}]}_{\Sup} - \BSigmah_{\Sup}\si{k} - \frac{\rho}{T\si{k}} \Zc_{\Sup}\si{k} \right) + \BDelta_{\Sup}\si{k} \; .
\end{align}
Note that $(\BOmegac,\Zc)$ fulfills the stationarity condition in Eq.~\eqref{eq:pdstationarity} if and only if $\BDelta_{\Sup}$ is a fixed point of $\F$, that is $\F(\BDelta_{\Sup}) = \BDelta_{\Sup}$.
Define the ball:
\begin{align} \label{eq:pdball}
\Ball(r) = \{\BDelta_{\Sup} \mid \|\BDelta_{\Sup}\|_\infty \leq r\} \; .
\end{align}
Note that $\F$ is continuous and that $\Ball(r)$ is convex and compact.
By the Brouwer's fixed point theorem \citep[p. 161]{Ortega70}, if $\F$ maps $\Ball(r)$ into itself, then there exists some fixed point in $\B(r)$.
That is:
\begin{align*}
(\forall \BDelta_{\Sup} \in \Ball(r)){\rm\ }\F(\BDelta_{\Sup}) \in \Ball(r) {\rm\ \ \ }\Rightarrow{\rm\ \ \ } (\exists \BDelta_{\Sup} \in \Ball(r)){\rm\ }\F(\BDelta_{\Sup}) = \BDelta_{\Sup} \; .
\end{align*}
Note that by property of the Kronecker product and the definition in Eq.~\eqref{eq:hessian}, we have:
\begin{align*}
(\forall k){\rm\ }{[\inv{\BOmegat\si{k}}\hns\BDelta\si{k}\inv{\BOmegat\si{k}}]}_{\Sup} & = {[\kronprod{\inv{\BOmegat\si{k}}}{\inv{\BOmegat\si{k}}}]}_{\Sup\Sup} \BDelta_{\Sup}\si{k} \\
 & = \BGammat_{\Sup\Sup}\si{k} \BDelta_{\Sup}\si{k} \; .
\end{align*}
Recall that $\A\si{k} \equiv \BSigmah\si{k}-\BSigmat\si{k}$ and $\inv{\BOmegat\si{k}} = \BSigmat\si{k}$.
By Eq.~\eqref{eq:pdinvomegac}, Eq.~\eqref{eq:pdmatrixb} and the above observation, we can rewrite Eq.~\eqref{eq:pdbrouwerfunction} as follows:
\begin{align*}
(\forall k){\rm\ }\F\si{k}(\BDelta_{\Sup}\si{k}) & = -\inv{\BGammat_{\Sup\Sup}\si{k}} \left( {[\inv{\textstyle{\BOmegac\si{k}}}]}_{\Sup} - \BSigmah_{\Sup}\si{k} - \frac{\rho}{T\si{k}} \Zc_{\Sup}\si{k} \right) + \BDelta_{\Sup}\si{k} \\
 & \hspace{-1.05in} = -\inv{\BGammat_{\Sup\Sup}\si{k}} \left( {[\inv{\textstyle{\BOmegac\si{k}}}]}_{\Sup} - {[\inv{\BOmegat\si{k}}]}_{\Sup} - \A_{\Sup}\si{k} - \frac{\rho}{T\si{k}} \Zc_{\Sup}\si{k} \right) + \BDelta_{\Sup}\si{k} \\
 & \hspace{-1.05in} = -\inv{\BGammat_{\Sup\Sup}\si{k}} \left( {[\inv{\BOmegat\si{k}}\hns\BDelta\si{k}\inv{\BOmegat\si{k}}]}_{\Sup} + {[\J\si{k} (\inv{\BOmegat\si{k}}\hns\BDelta\si{k})^2 \inv{\BOmegat\si{k}}]}_{\Sup} - \A_{\Sup}\si{k} - \frac{\rho}{T\si{k}} \Zc_{\Sup}\si{k} \right) + \BDelta_{\Sup}\si{k} \\
 & \hspace{-1.05in} = -\inv{\BGammat_{\Sup\Sup}\si{k}} \left( \BGammat_{\Sup\Sup}\si{k} \BDelta_{\Sup}\si{k} + \B_{\Sup}\si{k} - \A_{\Sup}\si{k} - \frac{\rho}{T\si{k}} \Zc_{\Sup}\si{k} \right) + \BDelta_{\Sup}\si{k} \\
 & \hspace{-1.05in} = -\inv{\BGammat_{\Sup\Sup}\si{k}} \left( \B_{\Sup}\si{k} - \A_{\Sup}\si{k} - \frac{\rho}{T\si{k}} \Zc_{\Sup}\si{k} \right) \; .
\end{align*}
By assumption, we have that $(\forall k){\rm\ }\|\A\si{k}\|_\infty = \|\BSigmah\si{k} - \BSigmat\si{k}\|_\infty \leq \eps$.
Recall that we assumed that $(\forall k){\rm\ }T\si{k} = 1$.
By the above, Eq.~\eqref{eq:pdlinfmatrixb} and since $(\forall k){\rm\ }\|\inv{\BGammat_{\Sup\Sup}\si{k}}\|_* \leq C_{\BGammat}$ where $C_{\BGammat}$ is defined in Eq.~\eqref{eq:gammaconst}, we have:
\begin{align} \label{eq:pdlinff}
(\forall k){\rm\ }\|\F\si{k}(\BDelta_{\Sup}\si{k})\|_\infty & \leq \|\inv{\BGammat_{\Sup\Sup}\si{k}}\|_* \| \B_{\Sup}\si{k} - \A_{\Sup}\si{k} - \frac{\rho}{T\si{k}} \Zc_{\Sup}\si{k} \|_\infty \nonumber \\
 & \leq \|\inv{\BGammat_{\Sup\Sup}\si{k}}\|_* \left( \|\B_{\Sup}\si{k}\|_\infty + \|\A_{\Sup}\si{k}\|_\infty + \frac{\rho}{T\si{k}} \|\Zc_{\Sup}\si{k}\|_\infty \right) \nonumber \\
 & \leq \|\inv{\BGammat_{\Sup\Sup}\si{k}}\|_* \left( \|\B\si{k}\|_\infty + \|\A\si{k}\|_\infty + \frac{\rho}{T\si{k}} \|\Zc_{\Sup}\|_{\infty,\dual{p}} \right) \nonumber \\
 & \leq C_{\BGammat} \left( \textstyle{ \frac{3}{2} d_{\BOmegat}C_{\BSigmat}^3 \|\BDelta\si{k}\|_\infty^2 } + \eps + \rho \right) \nonumber \\
 & = C_{\BGammat} \left( \textstyle{ \frac{3}{2} d_{\BOmegat}C_{\BSigmat}^3 \|\BDelta\si{k}\|_\infty^2 } + \eps (1+\rho/\eps) \right) \; .
\end{align}
In order to finish the proof, we set $r$ in the ball defined in Eq.~\eqref{eq:pdball} so that Eq.~\eqref{eq:pdradius}, Eq.~\eqref{eq:pdlinfmatrixb} and Eq.~\eqref{eq:pdlinff} are fulfilled with the proper constants.
Recall that $\BOmegac_{\Supc} = \zero \Rightarrow \BDelta_{\Supc} = \zero$ and thus $\|\BDelta\si{k}\|_\infty = \|\BDelta_{\Sup}\si{k}\|_\infty$.
Assume that $(\forall k){\rm\ }\|\BDelta\si{k}\|_\infty \leq r$.
Note that this implies that Eq.~\eqref{eq:pdball} is fulfilled and then $\BDelta_{\Sup} \in \Ball(r)$.
Our goal is to fulfill the conditions in Eq.~\eqref{eq:pdradius}, Eq.~\eqref{eq:pdlinfmatrixb} and Eq.~\eqref{eq:pdlinff}, that is:
\begin{align} \refstepcounter{equation}
r & \leq 1/(3d_{\BOmegat}C_{\BSigmat}) \; , \tag{\theequation.a}\label{eq:pdfinalfirst} \\
\textstyle{\frac{3}{2}} d_{\BOmegat}C_{\BSigmat}^3 r^2 & \leq \eps \; , \tag{\theequation.b}\label{eq:pdfinalsecond} \\
C_{\BGammat} \left( \textstyle{ \frac{3}{2} d_{\BOmegat}C_{\BSigmat}^3 r^2 } + \eps (1+\rho/\eps) \right) & \leq r \; . \tag{\theequation.c}\label{eq:pdfinalthird}
\end{align}
As we will show, the above statements hold by using the following setting:
\begin{align} \label{eq:pdsettingr}
r = 2 \eps C_{\BGammat} (1+\rho/\eps) \; .
\end{align}
By Eq.~\eqref{eq:samplecomplexity} and Eq.~\eqref{eq:pdsettingr}, we have that Eq.~\eqref{eq:pdfinalfirst} holds since:
\begin{align} \label{eq:pdfinalfirstproof}
r & = 2 \eps C_{\BGammat} (1+\rho/\eps) \nonumber \\
 & \leq 2 \eps C_{\BGammat} (1+\rho/\eps)^2 \nonumber \\
 & \leq \textstyle{ \min\left( \frac{1}{3d_{\BOmegat}C_{\BSigmat}} {\rm\ },{\rm\ } \frac{1}{3d_{\BOmegat}C_{\BSigmat}^3C_{\BGammat}} \right) } \; .
\end{align}
By Eq.~\eqref{eq:samplecomplexity} and Eq.~\eqref{eq:pdsettingr}, we have that Eq.~\eqref{eq:pdfinalsecond} holds since:
\begin{align*}
\textstyle{\frac{3}{2}} d_{\BOmegat}C_{\BSigmat}^3 r^2 & = \textstyle{ \left( 6 d_{\BOmegat} \eps C_{\BSigmat}^3 C_{\BGammat}^2 (1+\rho/\eps)^2 \right) \eps } \\
 & \leq \eps \; .
\end{align*}
By Eq.~\eqref{eq:pdsettingr} and Eq.~\eqref{eq:pdfinalfirstproof}, we have that Eq.~\eqref{eq:pdfinalthird} holds since:
\begin{align*}
C_{\BGammat} \left( \textstyle{ \frac{3}{2} d_{\BOmegat}C_{\BSigmat}^3 r^2 } + \eps (1+\rho/\eps) \right) & \leq C_{\BGammat} \left( \textstyle{ \frac{3}{2} d_{\BOmegat}C_{\BSigmat}^3 \frac{1}{3d_{\BOmegat}C_{\BSigmat}^3C_{\BGammat}} r } + \eps (1+\rho/\eps) \right) \\
 & = \textstyle{\frac{1}{2}} r + \eps C_{\BGammat} (1+\rho/\eps) \\
 & = \textstyle{\frac{1}{2}} r + \textstyle{\frac{1}{2}} r \\
 & = r \; .
\end{align*}
Thus, we showed that $(\forall k){\rm\ }\|\B\si{k}\|_\infty \leq \eps$.

\vspace{1em} \noindent \emph{Part IV.}
Here, we complete our proof.
Since the primal-dual witness construction succeeded, the solution $\BOmegac$ to the restricted problem in Eq.~\eqref{eq:pdrestricted} is equal to the solution $\BOmegah$ to the original (unrestricted) problem in Eq.~\eqref{eq:multitaskggm}.
Recall that $\BOmegac_{\Supc} = \zero \Rightarrow \BDelta_{\Supc} = \zero$ and thus $\|\BDelta\si{k}\|_\infty = \|\BDelta_{\Sup}\si{k}\|_\infty$.
Since $\BDelta\si{k} \equiv \BOmegat\si{k} - \BOmegac\si{k}$ by Eq.~\eqref{eq:pdsettingr} we have:
\begin{align*}
(\forall k){\rm\ }\|\BOmegat\si{k} - \BOmegah\si{k}\|_\infty \leq 2 \eps C_{\BGammat} (1+\rho/\eps) \; .
\end{align*}
By our assumption that $(\forall k,(n_1,n_2) \in \Sup){\rm\ }\omegat_{n_1n_2}\si{k} = 0 \text{\ or\ } |\omegat_{n_1n_2}\si{k}| > 4 \eps C_{\BGammat} (1+\rho/\eps)$ and by Lemma \ref{lem:signrecovery}, it follows that the edge signs of $\BOmegat$ are equal to those of $\BOmegah$.
This together with the fact that $\BOmegah_{\Supc} = \zero$ implies that \emph{$\BOmegah$ correctly recovers the support union}, i.e., $\Sup_{\BOmegah} = \Sup$.
\qedhere
\end{proof}

\begin{lemma} \label{lem:concentrationsubgaussian}
Let $\BSigmat \equiv \{\sigmat_{n_1n_2}\si{k}\}$.
Assume that for each $n$ and $k$, the random variable $x_n\si{k} / \sqrt{\sigmat_{nn}\si{k}}$ is zero-mean and sub-Gaussian with parameter $C_1$.
Assume that we are given $M$ i.i.d. samples for each of the $K$ tasks.
For some $\tau > 2$, with probability at least $1-4/N^{\tau-2}$, we have:
\begin{align*}
(\forall k){\rm\ }\|\BSigmah\si{k} - \BSigmat\si{k}\|_\infty \leq \sqrt{\frac{\log K + \tau \log N}{M}} \left( 8\sqrt{2}(1+4C_1^2) \max_{nk}{\sigmat_{nn}\si{k}} \right) \; .
\end{align*}
\end{lemma}
\begin{proof}
Note that in order to bound $\|\BSigmah-\BSigmat\|_\infty$, we need to simultaneously bound each entry of $\BSigmah-\BSigma$.
We use Lemma 1 in \citep{Ravikumar11} since we assume that $x_n\si{k} / \sqrt{\sigmat_{nn}\si{k}}$ is zero-mean and sub-Gaussian with parameter $C_1$.
Thus, by using Lemma 1 in \citep{Ravikumar11} (for bounding each specific entry) and the union bound (for bounding simultaneously all entries for all tasks), we have for $\eps \in (0,8(1+4C_1^2) \max_{nk}{\sigmat_{nn}\si{k}})$:
\begin{align*}
\P[ (\exists k){\rm\ }\|\BSigmah\si{k}-\BSigmat\si{k}\|_\infty \geq \eps ] \leq 4KN^2 {\rm\ } \lexp{-\frac{M\eps^2}{128(1+4C_1^2)^2 (\max_{nk}{\sigmat_{nn}\si{k}})^2}} = \delta \; .
\end{align*}
By solving for $\eps$ in the above, we get $\eps = \sqrt{\frac{\log K + \log{(4N^2/\delta)}}{M}} \left( 8\sqrt{2}(1+4C_1^2) \max_{nk}{\sigmat_{nn}\si{k}} \right)$.
By setting $\delta = 4/N^{\tau-2}$, we prove our claim.
\qedhere
\end{proof}

Next, we provide the final proof.
\begin{proof}[\textbf{Proof of Theorem \ref{thm:consistencysubgaussian}}]
Lemma \ref{lem:consistency} assumed $(\forall k){\rm\ }\|\BSigmah\si{k} - \BSigmat\si{k}\|_\infty \leq \eps$.
By Lemma \ref{lem:concentrationsubgaussian}, we have $\eps = \sqrt{\frac{\log K + \tau \log N}{M}} \left( 8\sqrt{2}(1+4C_1^2) \max_{nk}{\sigmat_{nn}\si{k}} \right)$ with probability at least $1-4/N^{\tau-2}$.
By invoking both lemmas, we prove our claim.
\qedhere
\end{proof}

\subsection{Proof of Theorem \ref{thm:consistencygeneral}}

Here, we provide the detailed proof of Theorem \ref{thm:consistencygeneral}.
First, we derive an intermediate lemma needed for the final proof.
We also use Lemma \ref{lem:consistency} from the previous sub-section.
\begin{lemma} \label{lem:concentrationgeneral}
Let $\BSigmat \equiv \{\sigmat_{n_1n_2}\si{k}\}$.
Assume that for each $n$ and $k$, the random variable $x_n\si{k} / \sqrt{\sigmat_{nn}\si{k}}$ is zero-mean and has $4\m$-th moments upper bounded by $C_1$.
Assume that we are given $M$ i.i.d. samples for each of the $K$ tasks.
For some $\tau > 2$, with probability at least $1-1/N^{\tau-2}$, we have:
\begin{align*}
(\forall k){\rm\ }\|\BSigmah\si{k} - \BSigmat\si{k}\|_\infty \leq \sqrt{\frac{K^{1/\m} N^{\tau/\m}}{M}} \left( 2\m(\m(C_1+1))^{\frac{1}{2\m}} \max_{nk}{\sigmat_{nn}\si{k}} \right) \; .
\end{align*}
\end{lemma}
\begin{proof}
Note that in order to bound $\|\BSigmah-\BSigma\|_\infty$, we need to simultaneously bound each entry of $\BSigmah-\BSigma$.
We use Lemma 2 in \citep{Ravikumar11} since we assume that $x_n\si{k}$ is zero-mean, and that there is a positive integer $\m$ and a scalar $C_1$ such that $\E[(x_n\si{k} / \sqrt{\sigmat_{nn}\si{k}})^{4\m}] \leq C_1$.
Thus, by using Lemma 2 in \citep{Ravikumar11} (for bounding each specific entry) and the union bound (for bounding simultaneously all entries for all tasks), we have:
\begin{align*}
\P[ (\exists k){\rm\ }\|\BSigmah\si{k}-\BSigmat\si{k}\|_\infty \geq \eps ] \leq KN^2 {\rm\ } \frac{\m^{2\m+1}2^{2\m}(C_1+1) (\max_{nk}{\sigmat_{nn}\si{k}})^{2\m}}{M^\m\eps^{2\m}} = \delta \; .
\end{align*}
By solving for $\eps$ in the above, we obtain $\eps = \sqrt{\frac{K^{1/\m} N^{2/\m}}{M \delta^{1/\m}}} \left( 2\m(\m(C_1+1))^{\frac{1}{2\m}} \max_{nk}{\sigmat_{nn}\si{k}} \right)$.
By setting $\delta = 1/N^{\tau-2}$, we prove our claim.
\qedhere
\end{proof}

Next, we provide the final proof.
\begin{proof}[\textbf{Proof of Theorem \ref{thm:consistencygeneral}}]
Lemma \ref{lem:consistency} assumed $(\forall k){\rm\ }\|\BSigmah\si{k} - \BSigmat\si{k}\|_\infty \leq \eps$.
By Lemma \ref{lem:concentrationgeneral}, we have $\eps = \sqrt{\frac{K^{1/\m} N^{\tau/\m}}{M}} \left( 2\m(\m(C_1+1))^{\frac{1}{2\m}} \max_{nk}{\sigmat_{nn}\si{k}} \right)$ with probability at least $1-1/N^{\tau-2}$.
By invoking both lemmas, we prove our claim.
\qedhere
\end{proof}

\subsection{Proof of Theorem \ref{thm:tight}}

Here, we provide the detailed proof of Theorem \ref{thm:tight}.
First, we derive an intermediate lemma needed for the final proof.
\begin{lemma} \label{lem:bessel}
Assume $(x\si{1},y\si{1}),\dots,(x\si{M},y\si{M})$ are $M$ i.i.d. samples drawn from a bivariate Gaussian distribution with zero mean and identity covariance.
The random variable $z = \frac{1}{M} \sum_m x\si{m}y\si{m}$ follows a ``Bessel'' distribution with parameter $M$, with density function:
\begin{align*}
\frac{ M |M z/2|^{(M-1)/2} {\rm\ } \BesselK_{(M-1)/2}(|M z|) }{ \sqrt{\pi} {\rm\ } \Gamma(M/2) } \; ,
\end{align*}
\noindent where $\BesselK$ denotes the modified Bessel function of the second kind, and $\Gamma$ denotes the Gamma function.
Furthermore:
\begin{align*}
\P_z[|z| \leq \eps] & = 
\frac{ \left(\frac{\pi}{2}\right)^{(1-M {\rm\ mod\ } 2)/2} 2^{(M-1)/2} {\rm\ } \MeijerG_{1,3}^{2,1}\left( \begin{smallmatrix}1 \\ 1/2, & M/2, & 0\end{smallmatrix} \left| M^2 \eps^2/4 \right.\right) }{(M-2)!! {\rm\ } \pi} \; , \\
\P_z[|z| > \eps] & \leq 2\exp{-M\eps^2/8} \text{\ \ for\ \ }\eps \leq 3.96 \; , \\
\P_z[|z| > \eps] & \geq \exp{-\sqrt{2M}\eps - M\eps^2} \; ,
\end{align*}
\noindent where $\MeijerG$ denotes the Meijer G-function and $!!$ denotes the double factorial.
\end{lemma}
\begin{proof}
Note that $x$ and $y$ are jointly independent and have density functions $f_x(x) = \exp{-x^2/2}/\sqrt{2\pi}$ and $f_y(y) = \exp{-y^2/2}/\sqrt{2\pi}$ respectively.
By the product distribution formula, the density function of $w = xy$ is given by $f_w(w) = \int_{-\infty}^{+\infty}f_x(x)f_y(w/x)/|x|dx = \BesselK_0(|w|)/\pi$.
The characteristic function of $f_w$ is given by $\varphi_w(t) = \int_{-\infty}^{+\infty}\exp{i t w}f_w(w)dw = 1/\sqrt{1+t^2}$.
The density function of the sum $s = \sum_m w\si{m}$ is given by $f_s(s) = \int_{-\infty}^{+\infty}\exp{-i t s}\varphi_w(t)^Mdt = |s/2|^{(M-1)/2} {\rm\ } \BesselK_{(M-1)/2}(|s|)/(\sqrt{\pi} {\rm\ } \Gamma(M/2))$.
Note that $z$ as a function of $s$ can be defined as $z(s) = s/M$ and therefore $s(z) = M z$.
By change of variables, the density function $f_z(z) = |\partial s(z)/\partial z| f_s(s(z))$ is the one stated in our claim.

The exact probability $\P_z[|z| \leq \eps] = \int_{-\eps}^{+\eps}f_z(z)dz$.
The lower bound for $\P_z[|z| \leq \eps]$ can be obtained by using Taylor series.
The upper bound for $\P_z[|z| \leq \eps]$ comes from the fact that the moment generating function of $w = xy$ fulfills $\E_w[\exp{uw}] = 1/\sqrt{1-u^2} \leq \exp{2u^2}$ for $|u| \leq 0.99$.
By Markov's inequality and since $w\si{1},\dots,w\si{M}$ are jointly independent, we have:
\begin{align*}
\P_z[z > \eps] & = \P_{w\si{1},\dots,w\si{M}}\left[ \frac{1}{M} \sum_m w\si{m} > \eps \right] \\
 & = \P_{w\si{1},\dots,w\si{M}}\left[ \exp{\frac{t}{M} \sum_m w\si{m}} > \exp{t\eps} \right] \\
 & \leq \E_{w\si{1},\dots,w\si{M}}[ \exp{\frac{t}{M} \sum_m w\si{m}} ] \exp{-t\eps} \\
 & = \E_w[ \exp{\frac{t}{M}w} ]^M \exp{-t\eps} \\
 & \leq \exp{2\frac{t^2}{M^2}M-t\eps} \; .
\end{align*}
By optimally setting $t = M\eps/4$, we have $\P_z[z > \eps] \leq \exp{-M\eps^2/8}$.
By the union bound, we have $\P_z[|z| > \eps] \leq \P_z[z > \eps] + \P_z[-z > \eps]$, and we prove that the stated upper bound holds.
Note that we used the moment generating function upper bound for $u = t/M = \eps/4 \leq 0.99$ and thus $\eps \leq 3.96$.
\qedhere
\end{proof}

Next, we provide the final proof.
\begin{proof}[\textbf{Proof of Theorem \ref{thm:tight}}]
The proof is a specialization of Lemma \ref{lem:consistency} for a specific class of graphs.
This will allow us obtain tighter results than the ones presented in Theorems \ref{thm:consistencysubgaussian} and \ref{thm:consistencygeneral}.
For clarity, let the support union of the true model be $\Sup \equiv \Supt$.
Let $\Supc$ be the complement of $\Sup$.
Additionally, we assume that every task contains the same number of samples, i.e., $(\forall k){\rm\ }T\si{k} = 1$.
Given the length of the proof, we split it into four parts.

\vspace{1em} \noindent \emph{Part I.}
Here, we present our restricted model.
Assume that $N$ is an even number and that we partition the set of nodes $\V = \{1,\dots,N\}$ into $N/2$ pairs $\T_1,\dots,\T_{N/2}$.
That is, $\V = \cup_i{\T_i}$ and $|\T_i|=2$ for all $i$.
Let the true model $\BOmegat$ have support:
\begin{align} \label{eq:tightsupportdef}
\Sup \equiv \{(n_1,n_2) \mid (\exists i){\rm\ }\T_i = \{n_1,n_2\} \} \; .
\end{align}
Furthermore, assume that the algorithm knows that the diagonal elements in the true model $\BOmegat$ are all equal to $1$.
With respect to our general analysis in Lemma \ref{lem:consistency}, several things simplify in this restricted model.
For each task $k$, the true model $\BOmegat$ is given by:
\begin{align*}
\BOmegat\si{k} \equiv \left[ \begin{array}{cccc}
\BOmegat_{\T_1\T_1}\si{k} & \zero & \dotsm & \zero \\
\zero & \BOmegat_{\T_2\T_2}\si{k} & \dotsm & \zero \\
\vdots & \vdots & \ddots & \vdots \\
\zero & \zero & \dotsm & \BOmegat_{\T_{N/2}\T_{N/2}}\si{k}
\end{array} \right] \; ,
\end{align*}
\noindent where for the node pair $\T_i = \{n_1,n_2\}$, we define:
\begin{align*}
\BOmegat_{\T_i\T_i}\si{k} \equiv \left[ \begin{array}{cc}
1 & \omega_{n_1n_2}\si{k} \\
\omega_{n_1n_2}\si{k} & 1 \\
\end{array} \right] \; .
\end{align*}

\vspace{1em} \noindent \emph{Part II.}
Here, we show that a specifically constructed function $f$ governs the recovery of the true model.
Let $x_n\si{m,k}$ be the value of variable $n$ for sample $m$ and task $k$.
Let $\DD_{n_1n_2} = \{x_{n_1}\si{m,k},x_{n_2}\si{m,k}\}$ denote the data set formed by using only two variables $n_1$ and $n_2$.
Define the function $f : \R^{2 \times M \times K} \to \R$ as:
\begin{align} \label{eq:tightf}
f(\DD_{n_1n_2}) = \left\|\left(\frac{1}{M} \sum_m{x_{n_1}\si{m,1}x_{n_2}\si{m,1}},\dots,\frac{1}{M} \sum_m{x_{n_1}\si{m,K}x_{n_2}\si{m,K}}\right)\right\|_\dual{p} \; .
\end{align}
Note that by step \eqref{item:pdrestricted} in the proof of Lemma \ref{lem:consistency}, we have $\BOmegac_{\Supc} = \zero \Rightarrow [{\inv{\textstyle{\BOmegac\si{k}}}]}_{\Supc} = \zero$ for all $k$.
Additionally, since we assume that $(\forall k){\rm\ }T\si{k} = 1$, the \emph{stationarity condition} in Eq.~\eqref{eq:pdstationaritynonsupp} is equivalent to:
\begin{align*}
(\forall k){\rm\ }\Zc_{\Supc}\si{k} & = \frac{1}{\rho} ({[\inv{\textstyle{\BOmegac\si{k}}}]}_{\Supc} - \BSigmah_{\Supc}\si{k}) \nonumber \\
 & = -\frac{1}{\rho} \BSigmah_{\Supc}\si{k} \; .
\end{align*}
Thus, by using Eq.~\eqref{eq:tightf}, we have:
\begin{align} \label{eq:tightnonsupport}
\|\Zc_{\Supc}\|_{\infty,\dual{p}} = \frac{1}{\rho} \max_{(n_1,n_2) \in \Supc}f(\DD_{n_1n_2}) \; .
\end{align}
As we know, the above is related to the correct exclusion of all non-edges.
Next, we focus on the correct inclusion of all edges.
Recall that in our specific graph, the support $\Sup$ is form by a partition of $N/2$ node pairs $\T_1,\dots,\T_{N/2}$ as in Eq.~\eqref{eq:tightsupportdef}.
Furthermore, since we assume that $(\forall k){\rm\ }T\si{k} = 1$, the restricted problem in Eq.~\eqref{eq:pdrestricted} reduces to solving for all node pairs $i$:
\begin{align*}
\BOmegac_{\T_i\T_i} = \argmax_{(\forall k){\rm\ }\BOmega_{\T_i\T_i}\si{k}\succ \zero}\left(\sum_k{\ell_{\BSigmah_{\T_i\T_i}\si{k}}(\BOmega_{\T_i\T_i}\si{k})} - \rho\|\BOmega_{\T_i\T_i}\|_{1,p} \right) \; .
\end{align*}
Since we assume that the algorithm knows that the diagonal elements in the true model $\BOmegat$ are all equal to $1$, the above equation is equivalent to solving for every node pair ${\T_i = \{n_1,n_2\}}$:
\begin{align*}
\vomegac_{n_1n_2} = \argmax_{\vomega_{n_1n_2}}\left(\sum_k\left(\log{(1-{\omega_{n_1n_2}\si{k}}^2)} - \sigmah_{n_1n_2}\si{k}\omega_{n_1n_2}\si{k}\right) - \rho\|\vomega_{n_1n_2}\|_p\right) \; .
\end{align*}
The above problem has the minimizer $\vomegac_{n_1n_2}=\zero$ if and only if $\zero$ belongs to the subdifferential set of the non-smooth objective function at $\vomega_{n_1n_2}=\zero$.
That is:
\begin{align*}
 & (\forall k){\rm\ } \left. 0 = -\frac{2 \omega_{n_1n_2}\si{k}}{1-{\omega_{n_1n_2}\si{k}}^2} - \sigmah_{n_1n_2}\si{k} - \rho \zc_{n_1n_2}\si{k} \right|_{\vomega_{n_1n_2}=\zero} \\
\Rightarrow{\rm\ } & \vzc_{n_1n_2} = -\frac{1}{\rho} \vsigmah_{n_1n_2} \; ,
\end{align*}
\noindent where $\vzc_{n_1n_2} \in \R^K$ is the dual variable satisfying the \emph{dual feasibility} condition $\|\vzc_{n_1n_2}\|_\dual{p} \leq 1$.
Since the diagonal elements of $\BOmegac_{\T_i\T_i}$ are all equal to $1$, we impose $\|\vomegac_{n_1n_2}\|_\dual{p} < 1$ to guarantee positive definiteness.
Thus, we can conclude that the optimal solution $\vomegac_{n_1n_2}$ fulfills:
\begin{align*}
\vomegac_{n_1n_2} = \zero {\rm\ } \Leftrightarrow {\rm\ } \rho>1 \text{\ \ or\ \ } (\rho \leq 1 \text{\ \ and\ \ } \|\vsigmah_{n_1n_2}\|_\dual{p} \leq \rho) \; ,
\end{align*}
\noindent or equivalently by using Eq.~\eqref{eq:tightf}, for all $(n_1n_2) \in \Sup$:
\begin{align} \label{eq:tightsupport}
\vomegac_{n_1n_2} \neq \zero {\rm\ } \Leftrightarrow {\rm\ } \rho \leq 1 \text{\ \ and\ \ } f(\DD_{n_1n_2}) > \rho \; .
\end{align}
Thus, we obtained a simple expression for the correct inclusion of all edges.

\vspace{1em} \noindent \emph{Part III.}
Here, we analyze recovery success.
The condition for correctly excluding all non-edges is that $\|\Zc\|_{\infty,\dual{p}} < 1$.
By Eq.~\eqref{eq:tightnonsupport}, this is equivalent to $\max_{(n_1,n_2) \in \Supc}f(\DD_{n_1n_2}) < \rho$.
The condition for correctly including all edges is given by Eq.~\eqref{eq:tightsupport}.
Putting both conditions together, our goal is to lower-bound the probability that:
\begin{align*}
\rho \leq 1 \text{\ \ \ and\ \ \ } \max_{(n_1,n_2) \in \Supc}f(\DD_{n_1n_2}) < \rho \text{\ \ \ and\ \ \ } (\forall (n_1,n_2) \in \Sup){\rm\ }f(\DD_{n_1n_2}) > \rho \; .
\end{align*}
Recall that $|\Supc| \leq N^2 - N/2$.
Assume that $\rho/K^{1/\dual{p}} \leq 3.96$.
By the union bound and Lemma \ref{lem:bessel}, we have:
\begin{align} \label{eq:tightsuccessnonsup}
\P\left[ \max_{(n_1,n_2) \in \Supc}f(\DD_{n_1n_2}) > \rho \right] & = \P[ (\exists (n_1,n_2) \in \Supc) {\rm\ } f(\DD_{n_1n_2}) > \rho ] \nonumber \\
 & \leq \P\left[ (\exists (n_1,n_2) \in \Supc,k) {\rm\ } \left|\frac{1}{M} \sum_m{x_{n_1}\si{m,k}x_{n_2}\si{m,k}}\right| > \frac{\rho}{K^{1/\dual{p}}} \right] \nonumber \\
 & \leq \left(N^2-\frac{N}{2}\right)K {\rm\ } \P\left[ \left|\frac{1}{M} \sum_m{x_{n_1}\si{m,k}x_{n_2}\si{m,k}}\right| > \frac{\rho}{K^{1/\dual{p}}} \right] \nonumber \\
 & \leq (2N^2K-NK) \exp{-M\rho^2K^{-2/\dual{p}}/8} \; .
\end{align}
Next, we use Lemma 1 in \citep{Ravikumar11} since we assume that $x_n\si{k} / \sqrt{\sigmat_{nn}\si{k}}$ is zero-mean and sub-Gaussian with parameter $1$.
Recall that in our specific graph, the support $\Sup$ is form by a partition of $N/2$ node pairs $\T_1,\dots,\T_{N/2}$ as in Eq.~\eqref{eq:tightsupportdef}.
Assume that $2\rho/K^{1/\dual{p}} \leq \min_{(n_1,n_2) \in \Sup,k}{|\sigmat_{n_1n_2}\si{k}|}$.
Assume that $\rho/K^{1/\dual{p}} \in (0,40 \max_{nk}{\sigmat_{nn}\si{k}})$.
By the union bound and Lemma 1 in \citep{Ravikumar11}, we have:
\begin{align} \label{eq:tightsuccesssup}
\P\left[ (\exists (n_1,n_2) \in \Sup) {\rm\ } f(\DD_{n_1n_2}) < \rho \right] & \leq \P\left[ (\exists (n_1,n_2) \in \Sup,k) {\rm\ } \left|\frac{1}{M} \sum_m{x_{n_1}\si{m,k}x_{n_2}\si{m,k}}\right| < \frac{\rho}{K^{1/\dual{p}}} \right] \nonumber \\
 & \leq \frac{NK}{2} {\rm\ } \P\left[ \left|\frac{1}{M} \sum_m{x_{n_1}\si{m,k}x_{n_2}\si{m,k}}\right| < \frac{\rho}{K^{1/\dual{p}}} \right] \nonumber \\
 & \leq \frac{NK}{2} {\rm\ } \P\left[ \sigmat_{n_1n_2}\si{k} - \frac{1}{M} \sum_m{x_{n_1}\si{m,k}x_{n_2}\si{m,k}} > \sigmat_{n_1n_2}\si{k} - \frac{\rho}{K^{1/\dual{p}}} \right] \nonumber \\
 & \leq \frac{NK}{2} {\rm\ } \P\left[ \sigmat_{n_1n_2}\si{k} - \frac{1}{M} \sum_m{x_{n_1}\si{m,k}x_{n_2}\si{m,k}} > \frac{\rho}{K^{1/\dual{p}}} \right] \nonumber \\
 & \leq NK {\rm\ } \lexp{-\frac{M\rho^2K^{-2/\dual{p}}}{128 {\rm\ } 5^2 (\max_{nk}{\sigmat_{nn}\si{k}})^2}} \; .
\end{align}
In the above, without loss of generality, we assumed that $\sigmat_{nn}\si{k}>0$, the proof is carried in the same way in the case that $\sigmat_{nn}\si{k}<0$.
For convenience, in order to obtain a simple expression, assume that $\max_{nk}{\sigmat_{nn}\si{k}} \leq 1/20$.
Thus, from Eq.~\eqref{eq:tightsuccessnonsup} and Eq.~\eqref{eq:tightsuccesssup}, we have:
\begin{align*}
\P\left[ \max_{(n_1,n_2) \in \Supc}f(\DD_{n_1n_2}) > \rho \text{\ or\ } (\exists (n_1,n_2) \in \Sup) {\rm\ } f(\DD_{n_1n_2}) < \rho \right] \leq 2N^2K \exp{-M\rho^2K^{-2/\dual{p}}/8} = \delta \; .
\end{align*}
By solving for $M$ in the above, we get that recovery success is guaranteed with probability at least $1-\delta$, provided that $M \geq \frac{8 K^{2/\dual{p}}}{\rho^2} \left(\log K + 2\log N + \log{\frac{2}{\delta}}\right)$.

\vspace{1em} \noindent \emph{Part IV.}
Here, we analyze recovery failure by a proof by contradiction.
Our goal is to show that if a solution $\BOmegac$ exists with $\BOmegac_{\Supc} = \zero$, then $\|\Zc\|_{\infty,\dual{p}} > 1$.
This contradicts the dual feasibility condition in Eq.~\eqref{eq:pddualfeasibility}, and thus $\BOmegac$ is not an optimal solution from the Karush-Kuhn-Tucker conditions.
By Eq.~\eqref{eq:tightnonsupport}, our goal is to lower-bound the probability that:
\begin{align*}
\max_{(n_1,n_2) \in \Supc}f(\DD_{n_1n_2}) > \rho \; .
\end{align*}
Next, we partition the set of nodes $\V = \{1,\dots,N\}$ into $N/2$ pairs $\U_1,\dots,\U_{N/2}$ as follows.
For $i < N/2$, the pair $\U_i$ contains one node from $\T_i$ and one node from $\T_{i+1}$.
The pair $\U_{N/2}$ contains one node from $\T_{N/2}$ and one node from $\T_1$.
Note that by construction, if $T_i = \{n_1,n_2\}$ then $x_{n_1}\si{k}$ and $x_{n_2}\si{k}$ are independent for all $k$.
Define:
\begin{align*}
\U \equiv \{(n_1,n_2) \mid (\exists i){\rm\ }\U_i = \{n_1,n_2\} \} \; .
\end{align*}
By Lemma \ref{lem:bessel} and since $\U \in \Supc$, we have:
\begin{align*}
\P\left[ \max_{(n_1,n_2) \in \Supc}f(\DD_{n_1n_2}) > \rho \right] & = \P[ (\exists (n_1,n_2) \in \Supc) {\rm\ } f(\DD_{n_1n_2}) > \rho ] \\
 & \geq \P[ (\exists (n_1,n_2) \in \U) {\rm\ } f(\DD_{n_1n_2}) > \rho ] \\
 & \geq \P\left[ (\exists (n_1,n_2) \in \U,k) {\rm\ } \left|\frac{1}{M} \sum_m{x_{n_1}\si{m,k}x_{n_2}\si{m,k}}\right| > \rho \right] \\
 & = 1 - \P\left[ (\forall (n_1,n_2) \in \U,k) {\rm\ } \left|\frac{1}{M} \sum_m{x_{n_1}\si{m,k}x_{n_2}\si{m,k}}\right| \leq \rho \right] \\
 & = 1 - \P\left[ \left|\frac{1}{M} \sum_m{x_{n_1}\si{m,k}x_{n_2}\si{m,k}}\right| \leq \rho \right]^{NK/2} \\
 & = 1 - \left(1 - \P\left[ \left|\frac{1}{M} \sum_m{x_{n_1}\si{m,k}x_{n_2}\si{m,k}}\right| > \rho \right]\right)^{NK/2} \\
& \geq 1 - (1 - \exp{-\sqrt{2M}\rho - M\rho^2})^{NK/2} \\
 & = 1-\delta \; .
\end{align*}
By solving for $M$ in the above, we get that recovery failure is guaranteed with probability at least $1-\delta$, provided that $M \leq \frac{1}{\rho^2} \left(1+\log{\frac{1}{1-\delta^{2/(NK)}}} + \sqrt{1+2\log{\frac{1}{1-\delta^{2/(NK)}}}}\right)$.

To complete the proof, we set $\rho = \eps K^{1/\dual{p}}$.
Note that we assume that $\eps = \rho/K^{1/\dual{p}} \leq \min{(3.96, 40 \max_{nk}{\sigmat_{nn}\si{k}}, \min_{(n_1,n_2) \in \Sup,k}{|\sigmat_{n_1n_2}\si{k}|}/2)}$.
By positive definiteness, we have that $\min_{(n_1,n_2) \in \Sup,k}{|\sigmat_{n_1n_2}\si{k}|} \leq \max_{nk}{\sigmat_{nn}\si{k}}$.
Furthermore, we assumed that $\max_{nk}{\sigmat_{nn}\si{k}} \leq 1/20$ and therefore $\eps \leq 1/40$.
\qedhere
\end{proof}

\subsection{Proof of Theorem \ref{thm:inftheorylower}}

\begin{proof}
Assume we construct a block-diagonal matrix $\BOmegat\si{u} \in \R^{KN \times KN}$ by arranging the precision matrices $\BOmegat\si{u,1},\dots,\BOmegat\si{u,K} \in \R^{N \times N}$ on its diagonal.
Similarly, construct a matrix $\BOmega(\DD) \in \R^{KN \times KN}$ by arranging the precision matrices $\BOmega\si{1}(\DD),\dots,\BOmega\si{K}(\DD) \in \R^{N \times N}$ on its diagonal.
Note that $\BOmegat\si{u} \succ \zero$ and $\BOmega(\DD) \succ \zero$ since their blocks are positive definite as well.
Furthermore, note that under this setting, each sample contains $KN$ variables, and thus $\DD$ contains exactly $M$ samples in total.
Our results follow from invoking Theorems 1 and 2 in \citep{Wang10} by using the $M$ samples and $KN$ variables.
\qedhere
\end{proof}

\subsection{Proof of Theorem \ref{thm:diagonalstep}}

\begin{proof}
If $\BOmega\si{k}$ is a symmetric matrix, according to the Haynsworth inertia formula, ${\BOmega\si{k}\succ \zero}$ if and only if its Schur complement $z\si{k}-\t{\y\si{k}}\inv{\W\si{k}}\y\si{k}>0$ and $\W\si{k}\succ \zero$.
By maximizing Eq.~\eqref{DerivMultiTaskGGM} with respect to $z\si{k}$, we get $z\si{k}-\t{\y\si{k}}\inv{\W\si{k}}\y\si{k}=\frac{1}{v\si{k}}$.
This equality defines the optimal value for $z\si{k}$ in Eq.~\eqref{eq:diagonalstep}.
Since $v\si{k}>0$, the Schur complement $z\si{k}-\t{\y\si{k}}\inv{\W\si{k}}\y\si{k}$ is strictly positive.

Finally, in an iterative optimization algorithm, it suffices to initialize $\BOmega\si{k}$ to a matrix that is known to be positive definite, e.g., a diagonal matrix with positive elements.
\qedhere
\end{proof}

\subsection{Proof of Theorem \ref{thm:lpquad}}

\begin{proof}
By replacing the optimal $z\si{k}$ given by Theorem \ref {thm:diagonalstep} into the objective function in Eq.~\eqref{DerivMultiTaskGGM}, we get:
\begin{align} \label{VectorMultiTaskGGM}
\min_{(\forall k){\rm\ }\y\si{k}\in \R^{N-1}}\left(\begin{array}{l}
  \sum_k{T\si{k}\left(\frac{1}{2}\t{\y\si{k}}v\si{k}\inv{\W\si{k}}\y\si{k} +\t{\u\si{k}}\y\si{k}\right)} \\
  +\rho\sum_n\|(y_n\si{1},\dots,y_n\si{K})\|_p
\end{array}\right) \; .
\end{align}
Since $\W\si{k}\succ \zero \Rightarrow \inv{\W\si{k}}\succ \zero$, hence Eq.~\eqref{VectorMultiTaskGGM} is strictly convex.

Without loss of generality, we use the last row/column in our presentation, since permutation of rows and columns is always possible.
Let:
\begin{align*}
\inv{\W\si{k}} = \left[ \begin{array}{cc}
\H_{11}\si{k} & \h_{12}\si{k} \\
\t{\h_{12}\si{k}} & h_{22}\si{k}
\end{array} \right]
\; , \;
\y\si{k} = \left[ \begin{array}{c}
\y_1\si{k} \\
x_k
\end{array} \right]
\; , \;
\u\si{k} = \left[ \begin{array}{c}
\u_1\si{k} \\
u_2\si{k}
\end{array} \right] \; ,
\end{align*}
\noindent where $\H_{11}\si{k}\in \R^{N-2 \times N-2}$, $\h_{12}\si{k},\y_1\si{k},\u_1\si{k}\in \R^{N-2}$.

In terms of the variable $\x$ and the constants $q_k = T\si{k}v\si{k}h_{22}\si{k}$, $c_k = -T\si{k}(v\si{k}\t{\h_{12}\si{k}}\y_1\si{k}+u_2\si{k})$, the $\ell_{1,p}$ regularized quadratic problem in Eq.~\eqref{VectorMultiTaskGGM} can be reformulated as in Eq.~\eqref{eq:lpquad}.
Moreover, since $(\forall k){\rm\ }T\si{k}>0, v\si{k}>0, h_{22}\si{k}>0 \Rightarrow \q>\zero$, and therefore Eq.~\eqref{eq:lpquad} is strictly convex.
\qedhere
\end{proof}

\subsection{Proof of Theorem \ref{thm:quadknapsack}}

Here, we provide the detailed proof of Theorem \ref{thm:quadknapsack}.
First, we derive some intermediate lemmas needed for the final proof.
\begin{lemma} \label{LinfQuadIsQuadL1}
For $\q>\zero$, $\rho >0$, $p=\infty$, the $\ell_p$ regularized separable quadratic problem in Eq.~\eqref{eq:lpquad} is equivalent to the separable quadratic problem with one $\ell_1$ constraint:
\begin{align} \label{QuadL1Cons}
\min_{\left\|\r\right\|_1\leq \rho}\left(\frac{1}{2}\t{(\r-\c)}\inv{\diag(\q)}(\r-\c)\right) \; .
\end{align}
Furthermore, their optimal solutions are related by $\x^*=\inv{\diag(\q)}(\c-\r^*)$.
\end{lemma}
\begin{proof}
By Lagrangian duality, the problem in Eq.~\eqref{QuadL1Cons} is the dual of the problem in Eq.~\eqref{eq:lpquad}.
Furthermore, strong duality holds in this case.
\qedhere
\end{proof}

\begin{remark} \label{QuadL1NoZeros}
In Eq.~\eqref{QuadL1Cons}, we can assume that $(\forall k){\rm\ }c_k\neq 0$.
If $(\exists k){\rm\ }c_k=0$, the partial optimal solution is $r_k^*=0$, and since this assignment does not affect the constraint, we can safely remove $r_k$ from the optimization problem.
\end{remark}

\begin{remark} \label{QuadL1MoreThanRho}
In what follows, we assume that $\|\c\|_1 >\rho$.
If $\|\c\|_1\leq \rho$, the unconstrained optimal solution of Eq.~\eqref{QuadL1Cons} is also its optimal solution, since $\r^*=\c$ is inside the feasible region given that $\|\r^*\|_1\leq \rho$.
\end{remark}

\begin{lemma} \label{QuadL1SameOrthant}
For $\q>\zero$, $(\forall k){\rm\ }c_k\neq 0$, $\|\c\|_1 >\rho$, the optimal solution $\r^*$ of the separable quadratic problem with one $\ell_1$ constraint in Eq.~\eqref{QuadL1Cons} belongs to the same orthant as the unconstrained optimal solution $\c$, i.e., $(\forall k){\rm\ }r_k^*c_k\geq 0$.
\end{lemma}
\begin{proof}
We prove this by contradiction.
Assume $(\exists k_1){\rm\ }r_{k_1}^*c_{k_1}< 0$.
Let $\r$ be a vector such that $r_{k_1}=0$ and $(\forall k_2 \neq k_1){\rm\ }r_{k_2}=r_{k_2}^*$.
The solution $\r$ is feasible, since $\|\r^*\|_1\leq \rho$ and $\|\r\|_1=\|\r^*\|_1-|r_{k_1}^*|\leq \rho$.
The difference in the objective function between $\r^*$ and $\r$ is:
\begin{align*}
\textstyle{ \frac{1}{2}\t{(\r^*-\c)}\inv{\diag(\q)}(\r^*-\c) - \frac{1}{2}\t{(\r-\c)}\inv{\diag(\q)}(\r-\c) } & = \textstyle{ \frac{1}{2q_{k_1}}({r_{k_1}^*}^2-2 c_{k_1}r_{k_1}^*) } \\
 & > \textstyle{ \frac{{r_{k_1}^*}^2}{2q_{k_1}} } \\
 & > 0 \; .
\end{align*}
Thus, the objective function for $\r$ is smaller than for $\r^*$ (the assumed optimal solution), which is a contradiction.
\qedhere
\end{proof}

\begin{lemma} \label{QuadL1ConsIsQuadKnapsack}
For $\q>\zero$, $(\forall k){\rm\ }c_k\neq 0$, $\|\c\|_1 >\rho$, the separable quadratic problem with one $\ell_1$ constraint in Eq.~\eqref{QuadL1Cons} is equivalent to the continuous quadratic knapsack problem:
\begin{align} \label{QuadKnapsack}
\min_{\g \geq \zero {\rm\ ,\ } \t{\one}\g=\rho}\left( \sum_k{\frac{1}{2q_k}(g_k - |c_k|)^2} \right) \; .
\end{align}
Furthermore, their optimal solutions are related by $(\forall k){\rm\ }r_k^*=\sgn(c_k)g_k^*$.
\end{lemma}
\begin{proof}
By invoking Lemma \ref{QuadL1SameOrthant}, we can replace $(\forall k){\rm\ }r_k=\sgn(c_k)g_k$, $g_k\geq 0$ in Eq.~\eqref{QuadL1Cons}.
Finally, we change the inequality constraint $\t{\one}\g\leq \rho$ to an equality constraint since $\|\c\|_1 >\rho$ and therefore, the optimal solution must be on the boundary of the constraint set.
\qedhere
\end{proof}
The continuous quadratic knapsack problem has been solved in several areas.
\citet{Helgason80} provides an $\O(K\log K)$ algorithm which initially sort the breakpoints.
\citet{Brucker84} and later \citet{Kiwiel07} provide deterministic linear-time algorithms by using medians of breakpoint subsets.
In the context of machine learning, \citet{Duchi08b} provides a randomized linear-time algorithm, while \citet{Liu09} provides an $\O(K\log K)$ algorithm.
We point out that \citet{Duchi08b,Liu09} assume that the weights of the quadratic term are all equal, i.e., $(\forall k){\rm\ }q_k=1$.
Here, we assume arbitrary positive weights, i.e., $(\forall k){\rm\ }q_k>0$.
\begin{lemma} \label{QuadKnapsackHalfway}
For $\q>\zero$, $(\forall k){\rm\ }c_k\neq 0$, $\|\c\|_1 >\rho$, the continuous quadratic knapsack problem in Eq.~\eqref{QuadKnapsack} has the solution:
\begin{align} \label{QuadKnapsackBPFunc}
g_k(\nu)=\max(0,|c_k|-\nu q_k) \; ,
\end{align}
\noindent for some $\nu$, and furthermore, the optimal solution fulfills the condition:
\begin{align*}
\g^*=\g(\nu) {\rm\ } \Leftrightarrow {\rm\ } \t{\one}\g(\nu)=\rho \; .
\end{align*}
\end{lemma}
\begin{proof}
The Lagrangian of Eq.~\eqref{QuadKnapsack} is:
\begin{align*}
\min_{\g \geq \zero}{\left(\sum_k{\frac{1}{2q_k}(g_k - |c_k|)^2} + \nu(\t{\one}\g-\rho)\right)} \; .
\end{align*}
Both results can be obtained by invoking the Karush-Kuhn-Tucker optimality conditions on the above problem.
\qedhere
\end{proof}

\begin{remark}
Note that $g_k(\nu)$ in Eq.~\eqref{QuadKnapsackBPFunc} is a decreasing piecewise linear function with breakpoint $\nu=\frac{|c_k|}{q_k}>0$.
By Lemma \ref{QuadKnapsackHalfway}, finding the optimal $\g^*$ is equivalent to finding $\nu$ in a piecewise linear function $\t{\one}\g(\nu)$ that produces $\rho$.
\end{remark}

\begin{lemma} \label{QuadKnapsackSolution}
For $\q>\zero$, $(\forall k){\rm\ }c_k\neq 0$, $\|\c\|_1 >\rho$, the continuous quadratic knapsack problem in Eq.~\eqref{QuadKnapsack} has the optimal solution $g_k^*=\max(0,|c_k|-\nu^* q_k)$ for:
\begin{align*}
\frac{|c_{\pi_{k^*}}|}{q_{\pi_{k^*}}} \geq \nu^*=\frac{\sum_{k=1}^{k^*}{|c_{\pi_k}|}-\rho}{\sum_{k=1}^{k^*}{q_{\pi_k}}} \geq \frac{|c_{\pi_{k^*+1}}|}{q_{\pi_{k^*+1}}} \; ,
\end{align*}
\noindent where the breakpoints are sorted in decreasing order by a permutation $\pi$ of the indices ${1,2,\dots,K}$, i.e.,
$\frac{|c_{\pi_1}|}{q_{\pi_1}} \geq
\frac{|c_{\pi_2}|}{q_{\pi_2}} \geq \dots \geq
\frac{|c_{\pi_K}|}{q_{\pi_K}} \geq
\frac{|c_{\pi_{K+1}}|}{q_{\pi_{K+1}}}\equiv 0$.
\end{lemma}
\begin{proof}
Given $k^*$, $\nu^*$ can be found straightforwardly by using the equation of the line.
In order to find $k^*$, we search for the range in which:
\begin{align*}
\t{\one}\g\left(\frac{|c_{\pi_{k^*}}|}{q_{\pi_{k^*}}}\right) \leq \rho \leq \t{\one}\g\left(\frac{|c_{\pi_{k^*+1}}|}{q_{\pi_{k^*+1}}}\right) \; ,
\end{align*}
\noindent which proves our claim.
\qedhere
\end{proof}

Next, we provide the final proof.
\begin{proof}[\textbf{Proof of Theorem \ref{thm:quadknapsack}}]
For $\|\c\|_1\leq \rho$, from Remark \ref{QuadL1MoreThanRho} we know that $\r^*=\c$.
By Lemma \ref{LinfQuadIsQuadL1}, the optimal solution of Eq.~\eqref{eq:lpquad} is $\x^*=\inv{\diag(\q)}(\c-\r^*)=\zero$, and we prove the first claim.

For $\|\c\|_1> \rho$, by Lemma \ref{LinfQuadIsQuadL1}, the optimal solution of Eq.~\eqref{eq:lpquad} $x_{\pi_k}^*=\frac{1}{q_{\pi_k}}(c_{\pi_k}-r_{\pi_k}^*)$.
By Lemma \ref{QuadL1ConsIsQuadKnapsack}, $x_{\pi_k}^*=\frac{1}{q_{\pi_k}}(c_{\pi_k}-\sgn(c_{\pi_k})g_{\pi_k}^*)$.
By Lemma \ref{QuadKnapsackSolution}, $x_{\pi_k}^*=\frac{c_{\pi_k}}{q_{\pi_k}}-\sgn(c_{\pi_k})\max(0,\frac{|c_{\pi_k}|}{q_{\pi_k}}-\nu^*)$.

If $k>k^* \Rightarrow \frac{|c_{\pi_k}|}{q_{\pi_k}}< \nu^* \Rightarrow x_{\pi_k}^*=\frac{c_{\pi_k}}{q_{\pi_k}}$, and we prove the second claim.

If $k\leq k^* \Rightarrow \frac{|c_{\pi_k}|}{q_{\pi_k}}\geq \nu^* \Rightarrow x_{\pi_k}^*=\sgn(c_{\pi_k})\nu^*$, and we prove the third claim.
\qedhere
\end{proof}

\subsection{Proof of Theorem \ref{thm:quadtrustreg}}

Here, we provide the detailed proof of Theorem \ref{thm:quadtrustreg}.
First, we derive some intermediate lemmas needed for the final proof.
\begin{lemma} \label{L2QuadIsQuadTrustReg}
For $\q>\zero$, $\rho >0$, $p=2$, the $\ell_p$ regularized separable quadratic problem in Eq.~\eqref{eq:lpquad} is equivalent to the separable quadratic trust-region problem:
\begin{align} \label{QuadTrustReg}
\min_{\|\r\|_2\leq \rho}\left(\frac{1}{2}\t{(\r-\c)}\inv{\diag(\q)}(\r-\c)\right) \; .
\end{align}
Furthermore, their optimal solutions are related by $\x^*=\inv{\diag(\q)}(\c-\r^*)$.
\end{lemma}
\begin{proof}
By Lagrangian duality, the problem in Eq.~\eqref{QuadTrustReg} is the dual of the problem in Eq.~\eqref{eq:lpquad}.
Furthermore, strong duality holds in this case.
\qedhere
\end{proof}

\begin{remark} \label{QuadTRNoZeros}
In Eq.~\eqref{QuadTrustReg}, we can assume that $(\forall k){\rm\ }c_k\neq 0$.
If $(\exists k){\rm\ }c_k=0$, the partial optimal solution is $r_k^*=0$, and since this assignment does not affect the constraint, we can safely remove $r_k$ from the optimization problem.
\end{remark}

\begin{remark} \label{QuadTRMoreThanRho}
In what follows, we assume that $\|\c\|_2 >\rho$.
If $\|\c\|_2\leq \rho$, the unconstrained optimal solution of Eq.~\eqref{QuadTrustReg} is also its optimal solution, since $\r^*=\c$ is inside the feasible region given that $\|\r^*\|_2\leq \rho$.
\end{remark}
The trust-region problem has been extensively studied by the mathematical optimization community \citep{Boyd06,Forsythe65,More83}.
Trust-region methods arise in the optimization of general convex functions.
In that context, the strategy behind trust-region methods is to perform a local second-order approximation to the original objective function.
The quadratic model for local optimization is ``trusted'' to be correct inside a circular region (i.e., the trust region).
Separability is usually not assumed, i.e., a symmetric matrix $\mathbf{Q}$ is used instead of $\diag(\q)$ in Eq.~\eqref{eq:lpquad}, and therefore the general algorithms are more involved than ours.
In the context of machine learning, \citet{Duchi09} provides a closed form solution for the separable version of the problem when the weights of the quadratic term are all equal, i.e., $(\forall k){\rm\ }q_k=1$.
In this paper, we assume arbitrary positive weights, i.e., $(\forall k){\rm\ }q_k>0$.
A closed form solution is not possible in this general case, but the efficient one-dimensional Newton-Raphson method can be applied.
\begin{lemma} \label{QuadTrustRegDualLemma}
For $\q>\zero$, $(\forall k){\rm\ }c_k\neq 0$, $\|\c\|_2 >\rho$, the separable quadratic trust-region problem in Eq.~\eqref{QuadTrustReg} is equivalent to the problem:
\begin{align} \label{QuadTrustRegDual}
\min_{\lambda\geq 0}{\left(\sum_n\frac{c_n^2}{q_n+\lambda q_n^2}+\rho^2\lambda\right)} \; .
\end{align}
Furthermore, their optimal solutions are related by $\r^*=\inv{\diag(\one+\lambda^*\q)}\c$.
\end{lemma}
\begin{proof}
By Lagrangian duality, the problem in Eq.~\eqref{QuadTrustRegDual} is the dual of the problem in Eq.~\eqref{QuadTrustReg}.
Furthermore, strong duality holds in this case.
\qedhere
\end{proof}

\begin{remark}
For the special case $\q=\one$ of \citet{Duchi09}, the problem in Eq.~\eqref{QuadTrustRegDual} becomes $\min_{\lambda\geq 0}{\left(\frac{\|\c\|_2^2}{1+\lambda}+\rho^2\lambda\right)}$.
By minimizing with respect to $\lambda$ and by noting that $\lambda\geq 0$, we obtain the optimal solution $\lambda^*=\max{\left(0,\frac{\|\c\|_2}{\rho}-1\right)}$.
\end{remark}

Next, we provide the final proof.
\begin{proof}[\textbf{Proof of Theorem \ref{thm:quadtrustreg}}]
For $\|\c\|_2\leq \rho$, from Remark \ref{QuadTRMoreThanRho} we know that $\r^*=\c$.
By Lemma \ref{L2QuadIsQuadTrustReg}, the optimal solution of Eq.~\eqref{eq:lpquad} is $\x^*=\inv{\diag(\q)}(\c-\r^*)=\zero$, and we prove the first claim.

For $\|\c\|_2> \rho$, by Lemma \ref{L2QuadIsQuadTrustReg}, the optimal solution of Eq.~\eqref{eq:lpquad} is $(\forall k){\rm\ }x_k^*=\frac{1}{q_k}(c_k-r_k^*)$.
By Lemma \ref{QuadTrustRegDualLemma}, $x_k^*=\frac{1}{q_k}(c_k-\frac{1}{1+\lambda^*q_k}c_k)=\frac{\lambda^*}{1+\lambda^*q_k}c_k$, and we prove the second claim.
\qedhere
\end{proof}

\section{Algorithmic Details} \label{sec:algorithm}

First, we discuss the computational complexity the block coordinate descent method.
Algorithm \ref{Algorithm} has a time complexity of $\O(LN^3K)$ for $L$ iterations, $N$ variables and $K$ tasks.
The polynomial dependence $\O(N^3)$ on the number of variables is expected since, in the general case, we cannot produce an algorithm faster than computing the inverse of the sample covariance in the case of an infinite sample.
For $p=\infty$, the linear-time dependence $\O(K)$ on the number of tasks can be accomplished by using a deterministic linear-time method by using medians of breakpoint subsets \citep{Kiwiel07}.
A very easy-to-implement $\O(K\log K)$ algorithm is obtained by initially sorting the breakpoints and searching the range for which Eq.~\eqref{eq:quadknapsackrange} holds.
For $p=2$, the linear-time dependence $\O(K)$ on the number of tasks can be accomplished by using the one-dimensional Newton-Raphson method for solving Eq.~\eqref{eq:quadtrustregdual}.
In our implementation, we initialize $\lambda=0$ and perform 10 iterations of the Newton-Raphson method.

Next, we show that the block coordinate descent method converges to the optimal solution.
\begin{lemma}
The solution sequence generated by the block coordinate descent method is bounded and every cluster point is a solution of the $\ell_{1,p}$ multi-task structure learning problem in Eq.~\eqref{eq:multitaskggm}.
\end{lemma}
\begin{proof}
The non-smooth regularizer $\|\BOmega\|_{1,p}$ is separable into a sum of $\O(N^2)$ individual functions of the form $\|(\omega_{n_1n_2}\si{1},\dots,\omega_{n_1n_2}\si{K})\|_p$.
These functions are defined over blocks of $K$ variables, i.e., ${\omega_{n_1n_2}\si{1},\dots,\omega_{n_1n_2}\si{K}}$.
The objective function in Eq.~\eqref{eq:multitaskggm} is continuous on a compact level set.
By virtue of Theorem 4.1 in \citep{Tseng01}, we prove our claim.
\qedhere
\end{proof}
Finally, we show that we can reduce the size of the original problem by removing nodes that are not endpoints of any edge in the optimal solution.
The rule for node removal depends only on the sample covariance matrix and thus, it can be applied as a preprocessing step.
\begin{lemma}
Let the $\ell_\dual{p}$-norm be the dual of the $\ell_p$-norm, i.e., $\frac{1}{p} + \frac{1}{\dual{p}} = 1$.
If the $\ell_{\infty,\dual{p}}$-norm fulfills $\max_n\|(T\si{1}u_n\si{1},\dots,T\si{K}u_n\si{K})\|_\dual{p} \leq \rho$, then the problem in Eq.~\eqref{DerivMultiTaskGGM} has the minimizer $(\forall k){\rm\ }{\y\si{k}}^*=\zero$.
\end{lemma}
\begin{proof}
By replacing the optimal $z\si{k}$ given by Theorem \ref {thm:diagonalstep} into the objective function in Eq.~\eqref{DerivMultiTaskGGM}, we get:
\begin{align*}
\min_{(\forall k){\rm\ }\y\si{k}\in \R^{N-1}}\left(\begin{array}{l}
  \sum_k{T\si{k}\left(\frac{1}{2}\t{\y\si{k}}v\si{k}\inv{\W\si{k}}\y\si{k} +\t{\u\si{k}}\y\si{k}\right)} \\
  +\rho\sum_n\|(y_n\si{1},\dots,y_n\si{K})\|_p
\end{array}\right) \; .
\end{align*}
The above problem has the minimizer $(\forall k){\rm\ }{\y\si{k}}^*=\zero$ if and only if $\zero$ belongs to the subdifferential set of the non-smooth objective function at $(\forall k){\rm\ }{\y\si{k}}=\zero$.
That is:
\begin{align*}
(\exists \A\in \R^{N-1 \times K}){\rm\ }(T\si{1}\u\si{1},\dots,T\si{K}\u\si{K})+\A=\zero \text{\ \ and\ \ } \max_n\|(a_{n1},\dots,a_{nK})\|_\dual{p} \leq \rho \; .
\end{align*}
This condition is true for $\max_n\|(T\si{1}u_n\si{1},\dots,T\si{K}u_n\si{K})\|_\dual{p} \leq \rho$.
\qedhere
\end{proof}

\section{Real-World Data Sets}

In this section, we provide further details regarding the real-world data sets used in our experimental validation in Section \ref{sec:results}.

\subsection{1000 Functional Connectomes: Brain Regions} \label{sec:regionsconnectomes}

We present the list of the $157$ standard Talairach regions, from the \emph{1000 functional connectomes} data set.
In order to abbreviate our presentation, we use parentheses, e.g., ``(left, right) amygdala'', to indicate that we used two regions: left amygdala and right amygdala.
The brain regions are the following:

\vspace{-0.15in}\begin{itemize}\setlength{\itemsep}{0pt}\setlength{\parskip}{0pt}
\item Cerebellum: cerebellar lingual
\item Cerebellum: (culmen, declive, pyramis, tuber, uvula) of vermis
\item Cerebellum: (left, right) (cerebellar tonsil, culmen, declive, dentate, fastigium, inferior semi-lunar lobule, nodule, pyramis, tuber, uvula)
\item Cerebrum: hypothalamus
\item Cerebrum: (left, right) (amygdala, claustrum, hippocampus, pulvinar, putamen)
\item Cerebrum: (left, right) (anterior, lateral dorsal, lateral posterior, medial dorsal, midline, ventral anterior, ventral lateral, ventral posterior lateral, ventral posterior medial) nucleus
\item Cerebrum: (left, right) Brodmann area (1,2,\dots,47)
\item Cerebrum: (left, right) caudate (body, head, tail)
\item Cerebrum: (left, right) (lateral, medial) globus pallidus
\item Brainstem: (left, right) (mammillary body, red nucleus, substantia nigra, subthalamic nucleus)
\end{itemize}

\subsection{The Cancer Genome Atlas: Genes} \label{sec:genescancer}

We present the list of the $187$ genes, from the \emph{cancer genome atlas} data set.
These genes are commonly regulated in cancer and were identified on independent data sets by \citet{Lu07}.
The genes are the following:

{\small
\vspace{0.05in}\noindent
ABCA8, ABHD6, ACLY, ADAM10, ADAM12, ADHFE1, AGXT2, ALDH6A1, 
ANK2, ANKS1B, ANP32E, AP1S1, APOL2, ARL4D, ARPC1B, AURKA, 
AYTL2, BAT2D1, BAX, BFAR, BID, BOLA2, BRP44L, C10orf116, 
C17orf27, C1orf58, C1orf96, C5orf4, C6orf60, C8orf76, CALU, CARD4, 
CASC5, CBX3, CCNB2, CCT5, CDC14B, CDCA7, CEP55, CHRDL1, 
CIDEA, CKLF, CLEC3B, CLU, CNIH4, DBR1, DDX39, DHRS4, 
DKFZp667G2110, DKFZp762E1312, DMD, DNMT1, DTL, DTX3L, E2F3, ECHDC2, 
ECHDC3, EFCBP1, EFHC2, EIF2AK1, EIF2C2, EIF2S2, Ells1, EPHX2, 
EPRS, ERBB4, FAM107A, FAM49B, FARP1, FBXO3, FBXO32, FEN1, 
FEZ1, FKBP10, FKBP11, FLJ11286, FLJ14668, FLJ20489, FLJ20701, FLJ21511, 
FMNL3, FMO4, FNDC3B, FOXP1, FTL, GEMIN6, GLT25D1, GNL2, 
GOLPH2, GPR172A, GSTM5, GULP1, HDGF, HIF3A, HLA-F, HLF, 
HNRPK, HNRPU, HPSE2, HSPE1, ILF3, IPO9, IQGAP3, K-ALPHA-1, 
KCNAB1, KDELC1, KDELR2, KDELR3, KIAA1217, KIAA1715, LDHD, LOC162073, 
LOC91689, LRRFIP2, LSM4, MAGI1, MORC2, MPPE1, MSRA, MTERFD1, 
NAP1L1, NCL, NDRG2, NME1, NONO, NOX4, NPM1, NR3C2, 
NRP2, NUSAP1, P53AIP1, PALM, PAQR8, PDIA6, PGK1, PINK1, 
PLEKHB2, PLIN, PLOD3, PPAP2B, PPIH, PPP2R1B, PRC1, PSMA4, 
PSMA7, PSMB2, PSMB4, PSMB8, PTP4A3, RBAK, RECK, RORA, 
RPN2, SCNM1, SEMA6D, SFXN1, SHANK2, SLAMF8, SLC24A3, SLC38A1, 
SNCA, SNRPB, SNX10, SORBS2, SPP1, STAT1, SYNGR1, TAP1, 
TAPBP, TCEAL2, TMEM4, TMEPAI, TNFSF13B, TNPO1, TRPM3, TTK, 
TTL, TUBAL3, UBA2, USP2, UTP18, WASF3, WHSC1, WISP1, 
XTP3TPA, ZBTB12, ZWILCH.
}

\bibliography{references}

\begin{thebibliography}{96}
\expandafter\ifx\csname natexlab\endcsname\relax\def\natexlab#1{#1}\fi

\bibitem[Banerjee et~al.(2008)Banerjee, {El~Ghaoui}, and
  d'Aspremont]{Banerjee08}
O.~Banerjee, L.~{El~Ghaoui}, and A.~d'Aspremont.
\newblock Model Selection Through Sparse Maximum Likelihood Estimation for
  Multivariate \uppercase{G}aussian or Binary Data.
\newblock {\em Journal of Machine Learning Research}, 9\penalty0
  (Mar):\penalty0 485--516, 2008.

\bibitem[Banerjee et~al.(2006)Banerjee, {El~Ghaoui}, d'Aspremont, and
  Natsoulis]{Banerjee06}
O.~Banerjee, L.~{El~Ghaoui}, A.~d'Aspremont, and G.~Natsoulis.
\newblock Convex Optimization Techniques for Fitting Sparse
  \uppercase{G}aussian Graphical Models.
\newblock {\em International Conference on Machine Learning}, pages 89--96,
  2006.

\bibitem[Boyd and Vandenberghe(2006)]{Boyd06}
S.~Boyd and L.~Vandenberghe.
\newblock {\em Convex Optimization}.
\newblock Cambridge University Press, 2006.

\bibitem[Brucker(1984)]{Brucker84}
P.~Brucker.
\newblock An $\uppercase{O}(n)$ Algorithm for Quadratic Knapsack Problems.
\newblock {\em Operations Research Letters}, 3\penalty0 (3):\penalty0 163--166,
  1984.

\bibitem[Buckner et~al.(2008)Buckner, Andrews-Hanna, and Schacter]{Buchner08}
R.~Buckner, J.~Andrews-Hanna, and D.~Schacter.
\newblock The Brain's Default Network: Anatomy, Function, and Relevance to
  Disease.
\newblock {\em Annals of the New York Academy of Sciences}, 1124:\penalty0
  1--38, 2008.

\bibitem[Cai et~al.(2011)Cai, Liu, and Luo]{Cai11}
T.~Cai, W.~Liu, and X.~Luo.
\newblock A Constrained $\ell_1$ Minimization Approach to Sparse Precision
  Matrix Estimation.
\newblock {\em Journal of the American Statistical Association}, 106\penalty0
  (494):\penalty0 594--607, 2011.

\bibitem[Chiquet et~al.(2011)Chiquet, Grandvalet, and Ambroise]{Chiquet11}
J.~Chiquet, Y.~Grandvalet, and C.~Ambroise.
\newblock Inferring Multiple Graphical Structures.
\newblock {\em Statistics and Computing}, 21\penalty0 (4):\penalty0 537--553,
  2011.

\bibitem[Danaher et~al.(2014)Danaher, Wang, and Witten]{Danaher14}
P.~Danaher, P.~Wang, and D.~Witten.
\newblock The Joint Graphical Lasso for Inverse Covariance Estimation Across
  Multiple Classes.
\newblock {\em Journal of the Royal Statistical Society: Series B}, 76\penalty0
  (2):\penalty0 373--397, 2014.

\bibitem[d'Aspremont et~al.(2008)d'Aspremont, Banerjee, and
  {El~Ghaoui}]{DAspremont08}
A.~d'Aspremont, O.~Banerjee, and L.~{El~Ghaoui}.
\newblock First-Order Methods for Sparse Covariance Selection.
\newblock {\em SIAM Journal on Matrix Analysis and Applications}, 30\penalty0
  (1):\penalty0 56--66, 2008.

\bibitem[Dempster(1972)]{Dempster72}
A.~Dempster.
\newblock Covariance Selection.
\newblock {\em Biometrics}, 28\penalty0 (1):\penalty0 157--175, 1972.

\bibitem[Dinh et~al.(2013)Dinh, Kyrillidis, and Cevher]{Dinh13}
Q.~Dinh, A.~Kyrillidis, and V.~Cevher.
\newblock A Proximal \uppercase{N}ewton Framework for Composite Minimization:
  Graph Learning without \uppercase{C}holesky Decompositions and Matrix
  Inversions.
\newblock {\em International Conference on Machine Learning}, pages 271--279,
  2013.

\bibitem[Duchi et~al.(2008{\natexlab{a}})Duchi, Gould, and Koller]{Duchi08}
J.~Duchi, S.~Gould, and D.~Koller.
\newblock Projected Subgradient Methods for Learning Sparse
  \uppercase{G}aussians.
\newblock {\em Uncertainty in Artificial Intelligence}, pages 145--152,
  2008{\natexlab{a}}.

\bibitem[Duchi et~al.(2008{\natexlab{b}})Duchi, Shalev-Shwartz, Singer, and
  Chandra]{Duchi08b}
J.~Duchi, S.~Shalev-Shwartz, Y.~Singer, and T.~Chandra.
\newblock Efficient Projections onto the ${\ell}_1$-Ball for Learning in High
  Dimensions.
\newblock {\em International Conference on Machine Learning}, pages 272--279,
  2008{\natexlab{b}}.

\bibitem[Duchi et~al.(2010)Duchi, Shalev-Shwartz, Singer, and Tewari]{Duchi10}
J.~Duchi, S.~Shalev-Shwartz, Y.~Singer, and A.~Tewari.
\newblock Composite Objective Mirror Descent.
\newblock {\em Conference on Learning Theory}, pages 14--26, 2010.

\bibitem[Duchi and Singer(2009)]{Duchi09}
J.~Duchi and Y.~Singer.
\newblock Efficient Learning using Forward-Backward Splitting.
\newblock {\em Neural Information Processing Systems}, 22:\penalty0 495--503,
  2009.

\bibitem[Engeland et~al.(1996)Engeland, Andersen, Haldorsen, and
  Tretli]{Engeland96}
A.~Engeland, A.~Andersen, T.~Haldorsen, and S.~Tretli.
\newblock Smoking Habits and Risk of Cancers Other than Lung Cancer: 28 years'
  Follow-Up of 26,000 Norwegian Men and Women.
\newblock {\em Cancer Causes and Control}, 7\penalty0 (5):\penalty0 497--506,
  1996.

\bibitem[Forsythe and Golub(1965)]{Forsythe65}
G.~Forsythe and G.~Golub.
\newblock On the Stationary Values of a Second-Degree Polynomial on the Unit
  Sphere.
\newblock {\em SIAM Journal of the Society for Industrial and Applied
  Mathematics}, 13\penalty0 (4):\penalty0 1050--1068, 1965.

\bibitem[Friedman et~al.(2007)Friedman, Hastie, and Tibshirani]{Friedman07}
J.~Friedman, T.~Hastie, and R.~Tibshirani.
\newblock Sparse Inverse Covariance Estimation with the Graphical Lasso.
\newblock {\em Biostatistics}, 9\penalty0 (3):\penalty0 432--441, 2007.

\bibitem[Goggins et~al.(2004)Goggins, Gao, and Tsao]{Goggins04}
W.~Goggins, W.~Gao, and H.~Tsao.
\newblock Association Between Female Breast Cancer and Cutaneous Melanoma.
\newblock {\em International Journal of Cancer}, 111\penalty0 (5):\penalty0
  792--794, 2004.

\bibitem[Guillot et~al.(2012)Guillot, Rajaratnam, Rolfs, Maleki, and
  Wong]{Guillot12}
D.~Guillot, B.~Rajaratnam, B.~Rolfs, A.~Maleki, and I.~Wong.
\newblock Iterative Thresholding Algorithm for Sparse Inverse Covariance
  Estimation.
\newblock {\em Neural Information Processing Systems}, 25:\penalty0 1574--1582,
  2012.

\bibitem[Guo et~al.(2011)Guo, Levina, Michailidis, and Zhu]{Guo11}
J.~Guo, E.~Levina, G.~Michailidis, and J.~Zhu.
\newblock Joint Estimation of Multiple Graphical Models.
\newblock {\em Biometrika}, 98\penalty0 (1):\penalty0 1--15, 2011.

\bibitem[Hara and Washio(2011)]{Hara11}
S.~Hara and T.~Washio.
\newblock Common Substructure Learning of Multiple Graphical
  \uppercase{G}aussian Models.
\newblock {\em European Conference on Machine Learning and Knowledge Discovery
  in Databases}, 6912:\penalty0 1--16, 2011.

\bibitem[Hara and Washio(2013)]{Hara13}
S.~Hara and T.~Washio.
\newblock Learning a Common Substructure of Multiple Graphical
  \uppercase{G}aussian Models.
\newblock {\em Neural Networks}, 38:\penalty0 23--38, 2013.

\bibitem[Helgason et~al.(1980)Helgason, Kennington, and Lall]{Helgason80}
K.~Helgason, J.~Kennington, and H.~Lall.
\newblock A Polynomially Bounded Algorithm for a Singly Constrained Quadratic
  Program.
\newblock {\em Mathematical Programming}, 18\penalty0 (1):\penalty0 338--343,
  1980.

\bibitem[Honorio et~al.(2009)Honorio, Ortiz, Samaras, Paragios, and
  Goldstein]{Honorio09}
J.~Honorio, L.~Ortiz, D.~Samaras, N.~Paragios, and R.~Goldstein.
\newblock Sparse and Locally Constant \uppercase{G}aussian Graphical Models.
\newblock {\em Neural Information Processing Systems}, 22:\penalty0 745--753,
  2009.

\bibitem[Honorio and Samaras(2010)]{Honorio10b}
J.~Honorio and D.~Samaras.
\newblock Multi-Task Learning of \uppercase{G}aussian Graphical Models.
\newblock {\em International Conference on Machine Learning}, pages 447--454,
  2010.

\bibitem[Honorio et~al.(2012)Honorio, Samaras, Rish, and Cecchi]{Honorio12}
J.~Honorio, D.~Samaras, I.~Rish, and G.~Cecchi.
\newblock Variable Selection for \uppercase{G}aussian Graphical Models.
\newblock {\em International Conference on Artificial Intelligence and
  Statistics}, 22:\penalty0 538--546, 2012.

\bibitem[Hsieh et~al.(2012)Hsieh, Dhillon, Ravikumar, and Banerjee]{Hsieh12}
C.~Hsieh, I.~Dhillon, P.~Ravikumar, and A.~Banerjee.
\newblock A Divide-and-Conquer Procedure for Sparse Inverse Covariance
  Estimation.
\newblock {\em Neural Information Processing Systems}, 25:\penalty0 2330--2338,
  2012.

\bibitem[Hsieh et~al.(2011)Hsieh, Sustik, Dhillon, and Ravikumar]{Hsieh11}
C.~Hsieh, M.~Sustik, I.~Dhillon, and P.~Ravikumar.
\newblock Sparse Inverse Covariance Matrix Estimation Using Quadratic
  Approximation.
\newblock {\em Neural Information Processing Systems}, 24:\penalty0 2330--2338,
  2011.

\bibitem[Hsieh et~al.(2014)Hsieh, Sustik, Dhillon, and Ravikumar]{Hsieh14}
C.~Hsieh, M.~Sustik, I.~Dhillon, and P.~Ravikumar.
\newblock \uppercase{QUIC}: Quadratic Approximation for Sparse Inverse
  Covariance Estimation.
\newblock {\em Journal of Machine Learning Research}, 15\penalty0
  (Oct):\penalty0 2911--2947, 2014.

\bibitem[Hsieh et~al.(2013)Hsieh, Sustik, Dhillon, Ravikumar, and
  Poldrack]{Hsieh13}
C.~Hsieh, M.~Sustik, I.~Dhillon, P.~Ravikumar, and R.~Poldrack.
\newblock \uppercase{BIG \& QUIC}: Sparse Inverse Covariance Estimation for a
  Million Variables.
\newblock {\em Neural Information Processing Systems}, 26:\penalty0 3165--3173,
  2013.

\bibitem[Jalali et~al.(2010)Jalali, Ravikumar, Sanghavi, and Ruan]{Jalali10}
A.~Jalali, P.~Ravikumar, S.~Sanghavi, and C.~Ruan.
\newblock A Dirty Model for Multi-Task Learning.
\newblock {\em Neural Information Processing Systems}, 23:\penalty0 964--972,
  2010.

\bibitem[Jebara(2004)]{Jebara04}
T.~Jebara.
\newblock Multi-Task Feature and Kernel Selection for \uppercase{SVM}s.
\newblock {\em International Conference on Machine Learning}, pages 55--62,
  2004.

\bibitem[Johnson et~al.(2012)Johnson, Jalali, and Ravikumar]{Johnson12}
C.~Johnson, A.~Jalali, and P.~Ravikumar.
\newblock High-Dimensional Sparse Inverse Covariance Estimation using Greedy
  Methods.
\newblock {\em International Conference on Artificial Intelligence and
  Statistics}, 22:\penalty0 574--582, 2012.

\bibitem[Jones et~al.(2008)Jones, Zhang, Parsons, Lin, Leary, Angenendt,
  Mankoo, Carter, Kamiyama, Jimeno, Hong, Fu, Lin, Calhoun, Kamiyama, Walter,
  Nikolskaya, Nikolsky, Hartigan, Smith, Hidalgo, Leach, Klein, Jaffee,
  Goggins, Maitra, Iacobuzio-Donahue, Eshleman, Kern, Hruban, Karchin,
  Papadopoulos, Parmigiani, Vogelstein, Velculescu, and Kinzler]{Jones08}
S.~Jones, X.~Zhang, D.~Parsons, J.~Lin, R.~Leary, P.~Angenendt, P.~Mankoo,
  H.~Carter, H.~Kamiyama, A.~Jimeno, S.~Hong, B.~Fu, M.~Lin, E.~Calhoun,
  M.~Kamiyama, K.~Walter, T.~Nikolskaya, Y.~Nikolsky, J.~Hartigan, D.~Smith,
  M.~Hidalgo, S.~Leach, A.~Klein, E.~Jaffee, M.~Goggins, A.~Maitra,
  C.~Iacobuzio-Donahue, J.~Eshleman, S.~Kern, R.~Hruban, R.~Karchin,
  N.~Papadopoulos, G.~Parmigiani, B.~Vogelstein, V.~Velculescu, and K.~Kinzler.
\newblock Core Signaling Pathways in Human Pancreatic Cancers Revealed by
  Global Genomic Analyses.
\newblock {\em Science}, 321\penalty0 (5897):\penalty0 1801--1806, 2008.

\bibitem[Kambadur and Lozano(2013)]{Kambadur13}
P.~Kambadur and A.~Lozano.
\newblock A Parallel, Block Greedy Method for Sparse Inverse Covariance
  Estimation for Ultra-high Dimensions.
\newblock {\em International Conference on Artificial Intelligence and
  Statistics}, 23:\penalty0 351–359, 2013.

\bibitem[Kiwiel(2007)]{Kiwiel07}
K.~Kiwiel.
\newblock On Linear-Time Algorithms for the Continuous Quadratic Knapsack
  Problem.
\newblock {\em Journal of Optimization Theory and Applications}, 134\penalty0
  (3):\penalty0 549--554, 2007.

\bibitem[Kolar et~al.(2013)Kolar, Liu, and Xing]{Kolar13}
M.~Kolar, H.~Liu, and E.~Xing.
\newblock \uppercase{M}arkov Network Estimation from Multi-Attribute Data.
\newblock {\em International Conference on Machine Learning}, pages 73--81,
  2013.

\bibitem[Kolar et~al.(2014)Kolar, Liu, and Xing]{Kolar14}
M.~Kolar, H.~Liu, and E.~Xing.
\newblock Graph Estimation From Multi-Attribute Data.
\newblock {\em Journal of Machine Learning Research}, 15\penalty0
  (May):\penalty0 1713--1750, 2014.

\bibitem[Kolar et~al.(2009)Kolar, Song, and Xing]{Kolar09}
M.~Kolar, L.~Song, and E.~Xing.
\newblock Sparsistent Learning of Varying-Coefficient Models with Structural
  Changes.
\newblock {\em Neural Information Processing Systems}, 22:\penalty0 1006--1014,
  2009.

\bibitem[Lauritzen(1996)]{Lauritzen96}
S.~Lauritzen.
\newblock {\em Graphical Models}.
\newblock Oxford Press, 1996.

\bibitem[Lee et~al.(2006)Lee, Ganapathi, and Koller]{Lee06}
S.~Lee, V.~Ganapathi, and D.~Koller.
\newblock Efficient Structure Learning of \uppercase{M}arkov Networks Using
  ${\ell }_1$-Regularization.
\newblock {\em Neural Information Processing Systems}, 19:\penalty0 817--824,
  2006.

\bibitem[Levina et~al.(2008)Levina, Rothman, and Zhu]{Levina08}
E.~Levina, A.~Rothman, and J.~Zhu.
\newblock Sparse Estimation of Large Covariance Matrices via a Nested Lasso
  Penalty.
\newblock {\em The Annals of Applied Statistics}, 2\penalty0 (1):\penalty0
  245--263, 2008.

\bibitem[Liu et~al.(2009{\natexlab{a}})Liu, Palatucci, and Zhang]{Liu09}
H.~Liu, M.~Palatucci, and J.~Zhang.
\newblock Blockwise Coordinate Descent Procedures for the Multi-task Lasso,
  with Applications to Neural Semantic Basis Discovery.
\newblock {\em International Conference on Machine Learning}, pages 649--656,
  2009{\natexlab{a}}.

\bibitem[Liu et~al.(2009{\natexlab{b}})Liu, Ji, and Ye]{LiuJun09}
J.~Liu, S.~Ji, and J.~Ye.
\newblock Multi-Task Feature Learning Via Efficient $\ell_{2,1}$-Norm
  Minimization.
\newblock {\em Uncertainty in Artificial Intelligence}, pages 339--348,
  2009{\natexlab{b}}.

\bibitem[Liu and Ihler(2011)]{LiuQiang11}
Q.~Liu and A.~Ihler.
\newblock Learning Scale Free Networks by Reweighted $\ell_1$ regularization.
\newblock {\em International Conference on Artificial Intelligence and
  Statistics}, 15:\penalty0 40--48, 2011.

\bibitem[Loh and Wainwright(2013)]{Loh13}
P.~Loh and M.~Wainwright.
\newblock Regularized \uppercase{M}-estimators with Nonconvexity: Statistical
  and Algorithmic Theory for Local Optima.
\newblock {\em Neural Information Processing Systems}, 26:\penalty0 476--484,
  2013.

\bibitem[Lu et~al.(2007)Lu, Yi, Liu, Wen, James, Wang, and You]{Lu07}
Y.~Lu, Y.~Yi, P.~Liu, W.~Wen, M.~James, D.~Wang, and M.~You.
\newblock Common Human Cancer Genes Discovered by Integrated Gene-Expression
  Analysis.
\newblock {\em Public Library of Science ONE}, 2\penalty0 (11):\penalty0 e1149,
  2007.

\bibitem[Lu(2009)]{Lu09}
Z.~Lu.
\newblock Smooth Optimization Approach for Sparse Covariance Selection.
\newblock {\em SIAM Journal on Optimization}, 19\penalty0 (4):\penalty0
  1807–1827, 2009.

\bibitem[Marlin and K.Murphy(2009)]{Marlin09}
B.~Marlin and K.Murphy.
\newblock Sparse \uppercase{G}aussian Graphical Models with Unknown Block
  Structure.
\newblock {\em International Conference on Machine Learning}, pages 705--712,
  2009.

\bibitem[Marlin et~al.(2009)Marlin, Schmidt, and Murphy]{Marlin09b}
B.~Marlin, M.~Schmidt, and K.~Murphy.
\newblock Group Sparse Priors for Covariance Estimation.
\newblock {\em Uncertainty in Artificial Intelligence}, pages 383--392, 2009.

\bibitem[Meier et~al.(2008)Meier, {van~de~Geer}, and B\"uhlmann]{Meier08}
L.~Meier, S.~{van~de~Geer}, and P.~B\"uhlmann.
\newblock The Group Lasso for Logistic Regression.
\newblock {\em Journal of the Royal Statistical Society: Series B}, 70\penalty0
  (1):\penalty0 53--71, 2008.

\bibitem[Meinshausen and B\"uhlmann(2006)]{Meinshausen06}
N.~Meinshausen and P.~B\"uhlmann.
\newblock High Dimensional Graphs and Variable Selection with the Lasso.
\newblock {\em The Annals of Statistics}, 34\penalty0 (3):\penalty0 1436--1462,
  2006.

\bibitem[Mohan et~al.(2012)Mohan, Chung, Han, Witten, Lee, and Fazel]{Mohan12}
K.~Mohan, M.~Chung, S.~Han, D.~Witten, S.~Lee, and M.~Fazel.
\newblock Structured Learning of \uppercase{G}aussian Graphical Models.
\newblock {\em Neural Information Processing Systems}, 25:\penalty0 620--628,
  2012.

\bibitem[Mohan et~al.(2014)Mohan, London, Fazel, Witten, and Lee]{Mohan14}
K.~Mohan, P.~London, M.~Fazel, D.~Witten, and S.~Lee.
\newblock Node-Based Learning of Multiple \uppercase{G}aussian Graphical
  Models.
\newblock {\em Journal of Machine Learning Research}, 15\penalty0
  (Feb):\penalty0 445--488, 2014.

\bibitem[Monkul et~al.(2012)Monkul, Silva, Narayana, Peluso, Zamarripa, Nery,
  Najt, Li, Lancaster, Fox, Lafer, and Soares]{Monkul12}
E.~Monkul, L.~Silva, S.~Narayana, M.~Peluso, F.~Zamarripa, F.~Nery, P.~Najt,
  J.~Li, J.~Lancaster, P.~Fox, B.~Lafer, and J.~Soares.
\newblock Abnormal Resting State Corticolimbic Blood Flow in Depressed
  Unmedicated Patients with Major Depression: A$^{15}$O-H$_2$O \uppercase{PET}
  Study.
\newblock {\em Human Brain Mapping}, 33\penalty0 (2):\penalty0 272--279, 2012.

\bibitem[Mor\'e and Sorensen(1983)]{More83}
J.~Mor\'e and D.~Sorensen.
\newblock Computing a Trust Region Step.
\newblock {\em SIAM Journal on Scientific and Statistical Computing},
  4\penalty0 (3):\penalty0 553--572, 1983.

\bibitem[Negahban and Wainwright(2011)]{Negahban11}
S.~Negahban and M.~Wainwright.
\newblock Simultaneous Support Recovery in High Dimensions: Benefits and Perils
  of Block $\ell_1/\ell_\infty$-Regularization.
\newblock {\em IEEE Transactions on Information Theory}, 57\penalty0
  (6):\penalty0 3841--3863, 2011.

\bibitem[Nemirovski et~al.(2009)Nemirovski, Juditsky, Lan, and
  Shapiro]{Nemirovski09}
A.~Nemirovski, A.~Juditsky, G.~Lan, and A.~Shapiro.
\newblock Robust Stochastic Approximation Approach to Stochastic Programming.
\newblock {\em SIAM Journal on Optimization}, 19\penalty0 (4):\penalty0
  1574--1609, 2009.

\bibitem[Niculescu-Mizil and Caruana(2007)]{Niculescu07}
A.~Niculescu-Mizil and R.~Caruana.
\newblock Inductive Transfer for \uppercase{B}ayesian Network Structure
  Learning.
\newblock {\em International Conference on Artificial Intelligence and
  Statistics}, 2:\penalty0 339--346, 2007.

\bibitem[Obozinski et~al.(2011)Obozinski, Wainwright, and Jordan]{Obozinski11}
G.~Obozinski, M.~Wainwright, and M.~Jordan.
\newblock Support Union Recovery in High-Dimensional Multivariate Regression.
\newblock {\em The Annals of Statistics}, 39\penalty0 (1):\penalty0 1--47,
  2011.

\bibitem[Olsen et~al.(2012)Olsen, Oztoprak, Nocedal, and Rennie]{Olsen12}
P.~Olsen, F.~Oztoprak, J.~Nocedal, and S.~Rennie.
\newblock Newton-Like Methods for Sparse Inverse Covariance Estimation.
\newblock {\em Neural Information Processing Systems}, 25:\penalty0 755--763,
  2012.

\bibitem[Ortega and Rheinboldt(1970)]{Ortega70}
J.~Ortega and W.~Rheinboldt.
\newblock {\em Iterative Solution of Nonlinear Equations in Several Variables}.
\newblock Academic Press, 1970.

\bibitem[Oyen and Lane(2012)]{Oyen12}
D.~Oyen and T.~Lane.
\newblock Leveraging Domain Knowledge in Multitask \uppercase{B}ayesian Network
  Structure Learning.
\newblock {\em AAAI Conference on Artificial Intelligence}, pages 1091--1097,
  2012.

\bibitem[Peterson et~al.(2015)Peterson, Stingo, and Vannucci]{Peterson15}
C.~Peterson, F.~Stingo, and M.~Vannucci.
\newblock \uppercase{B}ayesian Inference of Multiple \uppercase{G}aussian
  Graphical Models.
\newblock {\em Journal of the American Statistical Association}, 110\penalty0
  (509):\penalty0 159--174, 2015.

\bibitem[Pleasance et~al.(2010)Pleasance, Cheetham, Stephens, McBride,
  Humphray, Greenman, Varela, Lin, Ord{\'o}{\~n}ez, Bignell, Ye, Alipaz, Bauer,
  Beare, Butler, Carter, Chen, Cox, Edkins, Kokko-Gonzales, Gormley, Grocock,
  Haudenschild, Hims, James, Jia, Kingsbury, Leroy, Marshall, Menzies, Mudie,
  Ning, Royce, Schulz-Trieglaff, Spiridou, Stebbings, Szajkowski, Teague,
  Williamson, Chin, Ross, Campbell, Bentley, Futreal, and
  Stratton]{Pleasance10}
E.~Pleasance, R.~Cheetham, P.~Stephens, D.~McBride, S.~Humphray, C.~Greenman,
  I.~Varela, M.~Lin, G.~Ord{\'o}{\~n}ez, G.~Bignell, K.~Ye, J.~Alipaz,
  M.~Bauer, D.~Beare, A.~Butler, R.~Carter, L.~Chen, A.~Cox, S.~Edkins,
  P.~Kokko-Gonzales, N.~Gormley, R.~Grocock, C.~Haudenschild, M.~Hims,
  T.~James, M.~Jia, Z.~Kingsbury, C.~Leroy, J.~Marshall, A.~Menzies, L.~Mudie,
  Z.~Ning, T.~Royce, O.~Schulz-Trieglaff, A.~Spiridou, L.~Stebbings,
  L.~Szajkowski, J.~Teague, D.~Williamson, L.~Chin, M.~Ross, P.~Campbell,
  D.~Bentley, P.~Futreal, and M.~Stratton.
\newblock A Comprehensive Catalogue of Somatic Mutations from a Human Cancer
  Genome.
\newblock {\em Nature}, 463:\penalty0 191--196, 2010.

\bibitem[Qi et~al.(2008)Qi, Liu, Carin, and Dunson]{Qi08}
Y.~Qi, D.~Liu, L.~Carin, and D.~Dunson.
\newblock Multi-Task Compressive Sensing with \uppercase{D}irichlet Process
  Priors.
\newblock {\em International Conference on Machine Learning}, pages 768--775,
  2008.

\bibitem[Ravikumar et~al.(2011)Ravikumar, Wainwright, Raskutti, and
  Yu]{Ravikumar11}
P.~Ravikumar, M.~Wainwright, G.~Raskutti, and B.~Yu.
\newblock High-Dimensional Covariance Estimation by Minimizing
  $\ell_1$-Penalized Log-Determinant Divergence.
\newblock {\em Electronic Journal of Statistics}, 5:\penalty0 935--980, 2011.

\bibitem[Rothman et~al.(2008)Rothman, Bickel, Levina, and Zhu]{Rothman08}
A.~Rothman, P.~Bickel, E.~Levina, and J.~Zhu.
\newblock Sparse Permutation Invariant Covariance Estimation.
\newblock {\em Electronic Journal of Statistics}, 2:\penalty0 494--515, 2008.

\bibitem[Scheinberg et~al.(2010)Scheinberg, Ma, and Goldfarb]{Scheinberg10}
K.~Scheinberg, S.~Ma, and D.~Goldfarb.
\newblock Sparse Inverse Covariance Selection via Alternating Linearization
  Methods.
\newblock {\em Neural Information Processing Systems}, 23:\penalty0 2101--2109,
  2010.

\bibitem[Scheinberg and Rish(2010)]{Scheinberg10b}
K.~Scheinberg and I.~Rish.
\newblock Learning Sparse \uppercase{G}aussian \uppercase{M}arkov Networks
  using a Greedy Coordinate Ascent Approach.
\newblock {\em European Conference on Machine Learning and Knowledge Discovery
  in Databases}, 6323:\penalty0 196--212, 2010.

\bibitem[Schmidt et~al.(2011)Schmidt, {Le~Roux}, and Bach]{Schmidt11}
M.~Schmidt, N.~{Le~Roux}, and F.~Bach.
\newblock Convergence Rates of Inexact Proximal-Gradient Methods for Convex
  Optimization.
\newblock {\em Neural Information Processing Systems}, 24:\penalty0 1458--1466,
  2011.

\bibitem[Schmidt et~al.(2008)Schmidt, Murphy, Fung, and Rosales]{Schmidt08}
M.~Schmidt, K.~Murphy, G.~Fung, and R.~Rosales.
\newblock Structure Learning in Random Fields for Heart Motion Abnormality
  Detection.
\newblock {\em IEEE Conference on Computer Vision and Pattern Recognition},
  pages 1--8, 2008.

\bibitem[Schmidt et~al.(2007)Schmidt, Niculescu-Mizil, and Murphy]{Schmidt07}
M.~Schmidt, A.~Niculescu-Mizil, and K.~Murphy.
\newblock Learning Graphical Model Structure Using ${\ell }_1$-Regularization
  Paths.
\newblock {\em Association for the Advancement of Artificial Intelligence
  Conference}, pages 1278--1283, 2007.

\bibitem[Schmidt et~al.(2009)Schmidt, {van~den~Berg}, Friedlander, and
  Murphy]{Schmidt09b}
M.~Schmidt, E.~{van~den~Berg}, M.~Friedlander, and K.~Murphy.
\newblock Optimizing Costly Functions with Simple Constraints: A Limited-Memory
  Projected Quasi-\uppercase{N}ewton Algorithm.
\newblock {\em International Conference on Artificial Intelligence and
  Statistics}, 5:\penalty0 456--463, 2009.

\bibitem[Sun et~al.(2013)Sun, Zhu, and Xu]{Sun14}
S.~Sun, Y.~Zhu, and J.~Xu.
\newblock Adaptive Variable Clustering in \uppercase{G}aussian Graphical
  Models.
\newblock {\em International Conference on Artificial Intelligence and
  Statistics}, 33:\penalty0 931--939, 2013.

\bibitem[Treister and Turek(2014)]{Treister14}
E.~Treister and J.~Turek.
\newblock A Block-Coordinate Descent Approach for Large-Scale Sparse Inverse
  Covariance Estimation.
\newblock {\em Neural Information Processing Systems}, 27:\penalty0 927--935,
  2014.

\bibitem[Tropp(2006)]{Tropp06b}
J.~Tropp.
\newblock Algorithms for Simultaneous Sparse Approximation, Part
  \uppercase{II}: Convex Relaxation.
\newblock {\em Signal Processing}, 86\penalty0 (3):\penalty0 589--602, 2006.

\bibitem[Tseng(2001)]{Tseng01}
P.~Tseng.
\newblock Convergence of a Block Coordinate Descent Method for
  Nondifferentiable Minimization.
\newblock {\em Journal of Optimization Theory and Applications}, 109\penalty0
  (3):\penalty0 475--494, 2001.

\bibitem[Turlach et~al.(2005)Turlach, Venables, and Wright]{Turlach05}
B.~Turlach, W.~Venables, and S.~Wright.
\newblock Simultaneous Variable Selection.
\newblock {\em Technometrics}, 47\penalty0 (3):\penalty0 349--363, 2005.

\bibitem[Varoquaux et~al.(2010)Varoquaux, Gramfort, Poline, and
  Thirion]{Varoquaux10}
G.~Varoquaux, A.~Gramfort, J.~Poline, and B.~Thirion.
\newblock Brain Covariance Selection: Better Individual Functional Connectivity
  Models Using Population Prior.
\newblock {\em Neural Information Processing Systems}, 23:\penalty0 2334--2342,
  2010.

\bibitem[Wainwright(2009)]{Wainwright09b}
M.~Wainwright.
\newblock Sharp Thresholds for High-Dimensional and Noisy Sparsity Recovery
  Using Constrained Quadratic Programming (Lasso).
\newblock {\em IEEE Transactions on Information Theory}, 55\penalty0
  (5):\penalty0 2183--2202, 2009.

\bibitem[Wainwright et~al.(2006)Wainwright, Ravikumar, and
  Lafferty]{Wainwright06}
M.~Wainwright, P.~Ravikumar, and J.~Lafferty.
\newblock High-Dimensional Graphical Model Selection Using ${\ell
  }_1$-Regularized Logistic Regression.
\newblock {\em Neural Information Processing Systems}, 19:\penalty0 1465--1472,
  2006.

\bibitem[Wang et~al.(2013)Wang, Banerjee, Hsieh, Ravikumar, and
  Dhillon]{Wang13}
H.~Wang, A.~Banerjee, C.~Hsieh, P.~Ravikumar, and I.~Dhillon.
\newblock Large Scale Distributed Sparse Precision Estimation.
\newblock {\em Neural Information Processing Systems}, 26:\penalty0 584--592,
  2013.

\bibitem[Wang et~al.(2010)Wang, Wainwright, and Ramchandran]{Wang10}
W.~Wang, M.~Wainwright, and K.~Ramchandran.
\newblock Information-Theoretic Bounds on Model Selection for
  \uppercase{G}aussian \uppercase{M}arkov Random Fields.
\newblock {\em IEEE International Symposium on Information Theory}, pages 1373
  -- 1377, 2010.

\bibitem[Wilson et~al.(2007)Wilson, Fern, Ray, and Tadepalli]{Wilson07}
A.~Wilson, A.~Fern, S.~Ray, and P.~Tadepalli.
\newblock Multi-Task Reinforcement Learning: A Hierarchical
  \uppercase{B}ayesian Approach.
\newblock {\em International Conference on Machine Learning}, pages 1015--1022,
  2007.

\bibitem[Xiao(2010)]{Xiao10}
L.~Xiao.
\newblock Dual Averaging Methods for Regularized Stochastic Learning and Online
  Optimization.
\newblock {\em Journal of Machine Learning Research}, 11\penalty0
  (Oct):\penalty0 2543--2596, 2010.

\bibitem[Yang et~al.(2013)Yang, Sun, and Toh]{Yang13c}
J.~Yang, D.~Sun, and K.~Toh.
\newblock A Proximal Point Algorithm for Log-Determinant Optimization with
  Group Lasso Regularization.
\newblock {\em SIAM Journal on Optimization}, 23\penalty0 (2):\penalty0
  857--893, 2013.

\bibitem[Yu et~al.(2008)Yu, Vishwanathan, G{\"u}nter, and Schraudolph]{Yu08}
J.~Yu, S.~Vishwanathan, S.~G{\"u}nter, and N.~Schraudolph.
\newblock A Quasi-\uppercase{N}ewton Approach to Nonsmooth Convex Optimization.
\newblock {\em International Conference on Machine Learning}, pages 1216--1223,
  2008.

\bibitem[Yuan(2010)]{Yuan10}
M.~Yuan.
\newblock High Dimensional Inverse Covariance Matrix Estimation via Linear
  Programming.
\newblock {\em Journal of Machine Learning Research}, 11\penalty0
  (Aug):\penalty0 2261--2286, 2010.

\bibitem[Yuan and Lin(2006)]{Yuan06}
M.~Yuan and Y.~Lin.
\newblock Model Selection and Estimation in Regression with Grouped Variables.
\newblock {\em Journal of the Royal Statistical Society: Series B}, 68\penalty0
  (1):\penalty0 49--67, 2006.

\bibitem[Yuan and Lin(2007)]{Yuan07}
M.~Yuan and Y.~Lin.
\newblock Model Selection and Estimation in the \uppercase{G}aussian Graphical
  Model.
\newblock {\em Biometrika}, 94\penalty0 (1):\penalty0 19--35, 2007.

\bibitem[Yun et~al.(2011)Yun, Tseng, and Toh]{Yun11}
S.~Yun, P.~Tseng, and K.~Toh.
\newblock A Block Coordinate Gradient Descent Method for Regularized Convex
  Separable Optimization and Covariance Selection.
\newblock {\em Mathematical Programming}, 129\penalty0 (2):\penalty0 331--355,
  2011.

\bibitem[Zhang and Wang(2010)]{Zhang10}
B.~Zhang and Y.~Wang.
\newblock Learning Structural Changes of \uppercase{G}aussian Graphical Models
  in Controlled Experiments.
\newblock {\em Uncertainty in Artificial Intelligence}, pages 701--708, 2010.

\bibitem[Zhou et~al.(2011)Zhou, R\"utimann, Xu, and B\"uhlmann]{Zhou11}
S.~Zhou, P.~R\"utimann, M.~Xu, and P.~B\"uhlmann.
\newblock High-Dimensional Covariance Estimation Based On \uppercase{G}aussian
  Graphical Models.
\newblock {\em Journal of Machine Learning Research}, 12\penalty0
  (Oct):\penalty0 2975--3026, 2011.

\bibitem[Zhu and Foygel(2015)]{Zhu15}
Y.~Zhu and R.~Foygel.
\newblock The Log-Shift Penalty for Adaptive Estimation of Multiple
  \uppercase{G}aussian Graphical Models.
\newblock {\em International Conference on Artificial Intelligence and
  Statistics}, 38:\penalty0 1153--1161, 2015.

\end{thebibliography}

\end{document}